\documentclass{article}

\usepackage[final, nonatbib]{neurips_2021}

\usepackage[utf8]{inputenc} 
\usepackage[T1]{fontenc}    
\usepackage{hyperref}       
\usepackage{url}            
\usepackage{booktabs}       
\usepackage{amsfonts}       
\usepackage{nicefrac}       
\usepackage{microtype}      
\usepackage{xcolor}         
\usepackage{subcaption}
\usepackage{todonotes}
\usepackage{bbm}
\usepackage{algpseudocode, algorithm, algorithmicx}
\usepackage{amsmath,amssymb, amsthm}
\usepackage[inline]{enumitem}
\usepackage{comment}
\usepackage{appendix}

\usepackage{arydshln}

\usepackage{microtype,hyperref}
\usepackage{eucal}

\newcommand{\citep}{\cite}

\setlist[itemize]{label=--}
\setlist[enumerate]{label=\arabic*),labelindent=\parindent,leftmargin=*}

\newtheorem{thm}{Theorem}
\newtheorem{lem}{Lemma}

\newtheorem{cor}{Corollary}

\usepackage{thmtools}
\usepackage{thm-restate}
\usepackage{cleveref}

\newcommand{\parhead}[1]{ \noindent {\bfseries\boldmath\ignorespaces #1}}

\newtheoremstyle{slplain}
  {.4\baselineskip\@plus.1\baselineskip\@minus.1\baselineskip}
  {.3\baselineskip\@plus.1\baselineskip\@minus.1\baselineskip}
  {\itshape}
  {}
  {\bfseries}
  {.\xspace}
  { }
  {}
\theoremstyle{slplain} 

\title{AC/DC: Alternating Compressed/DeCompressed Training of Deep Neural Networks}

\author{Alexandra Peste \thanks{Correspondence to: Alexandra Peste \href{mailto:alexandra.peste@ist.ac.at}{<alexandra.peste@ist.ac.at>}} \\ \small{IST Austria}
\And Eugenia Iofinova \\ \small{IST Austria} 
\And Adrian Vladu \\ \small{CNRS \& IRIF}
\And Dan Alistarh \\ \small{IST Austria \& Neural Magic}}

\begin{document}

\maketitle

\begin{abstract}
    The increasing computational requirements of deep neural networks (DNNs) have led to significant interest in obtaining DNN models that are \emph{sparse}, yet \emph{accurate}. Recent work has investigated the even harder case of \emph{sparse training},
    where the DNN weights are, for as much as possible, already sparse to reduce computational costs during training.
    Existing sparse training methods are often empirical and can have lower accuracy relative to the dense baseline. In this paper, we present a general approach called Alternating Compressed/DeCompressed (AC/DC) training of DNNs, demonstrate convergence for a variant of the algorithm, and  show that AC/DC outperforms existing sparse training methods in accuracy at similar computational budgets; at high sparsity levels, AC/DC even outperforms existing methods that rely on accurate pre-trained dense models. An important property
    of AC/DC is that it allows \emph{co-training} of dense and sparse models, yielding accurate \emph{sparse--dense model pairs} at the end of the training process. This is useful in practice, where compressed variants may be desirable for deployment in resource-constrained settings without re-doing the entire training flow, and also provides us with insights into the accuracy gap between dense and compressed models. 
    The code is available at: \url{https://github.com/IST-DASLab/ACDC}.

\end{abstract}

\vspace{-1em}
\section{Introduction}

The tremendous progress made by deep neural networks in solving diverse tasks has driven significant research and industry interest in deploying efficient versions of these models. 
To this end, entire families of model compression methods have been developed, such as pruning~\citep{hoefler2021sparsity} and quantization~\citep{gholami2021survey}, which are now  accompanied by hardware and software support~\citep{vanholder2016efficient, chen2018tvm,david2020tensorflow, mishra2021accelerating, graphcore}. 

Neural network pruning, which is the focus of this paper, is the compression method with arguably the longest history~\citep{lecun1990optimal}. The basic goal of pruning is  to obtain neural networks for which many connections are removed by being set to zero, while maintaining the network's accuracy. A myriad pruning methods have been proposed---please see~\cite{hoefler2021sparsity} for an in-depth survey---and it is currently understood that many popular networks can be compressed by more than an order of magnitude, in terms of their number of connections, without significant accuracy loss. 

Many accurate pruning methods require a fully-accurate, dense variant of the model, from which weights are subsequently removed. A shortcoming of this approach is the fact that the memory and computational savings due to compression are only available for the \emph{inference}, post-training phase, and not during training itself. This distinction becomes important especially for large-scale modern models, which can have millions or even billions of parameters, and for which fully-dense training can have high computational and even non-trivial  environmental costs~\citep{strubell2020energy}. 

One approach to address this issue is \emph{sparse training}, which essentially aims to remove connections from the neural network as early as possible during training, while still matching, or at least approximating, the accuracy of the fully-dense model. For example, the RigL technique~\cite{evci2020rigging} randomly removes a large fraction of connections  early in training, and then proceeds to optimize over the sparse support, providing savings due to sparse back-propagation. Periodically, the method re-introduces some of the weights during the training process, based on a combination of heuristics, which requires taking full gradients. 
These works, as well as many recent sparse training approaches~\citep{bellec2017deep, mocanu2018scalable, jayakumar2020top}, which we cover in detail in the next section, have shown empirically that non-trivial computational savings, usually measured in theoretical FLOPs, can be obtained using sparse training, and that the optimization process can be fairly robust to sparsification of the support. 

At the same time, this line of work still leaves intriguing open questions. The first is \emph{theoretical}: to our knowledge, none of the methods optimizing over sparse support, and hence providing training speed-up, have been shown to have convergence guarantees. 
The second is \emph{practical}, and concerns a deeper understanding of the relationship between the densely-trained model, and the sparsely-trained one. Specifically, (1)~most existing sparse training methods still leave a non-negligible accuracy gap, relative to dense training, or even post-training sparsification; and (2)~most existing work on sparsity requires significant changes to the training flow, and focuses on maximizing global accuracy metrics; thus, we lack understanding when it comes to \emph{co-training} sparse and dense models, as well as with respect to correlations between sparse and dense models \emph{at the level of individual predictions}.

\parhead{Contributions.} In this paper, we take a step towards addressing these questions. 
We investigate a general hybrid approach for sparse training of neural networks, which we call \emph{Alternating Compressed / DeCompressed (AC/DC)} training. AC/DC performs \emph{co-training of sparse and dense models}, and can return both an accurate \emph{sparse} model, and a \emph{dense} model, which can recover the dense baseline accuracy via fine-tuning. We show that a variant of AC/DC ensures convergence for general non-convex but smooth objectives, under analytic assumptions. Extensive experimental results show that it provides state-of-the-art accuracy among sparse training techniques at comparable training budgets, and can even outperform \emph{post-training} sparsification approaches when applied at high sparsities. 
    
AC/DC builds on the classic \emph{iterative hard thresholding (IHT)} family of methods for sparse recovery~\citep{blumensath2008iterative}. As the name suggests, AC/DC works by alternating the standard \emph{dense} training phases with \emph{sparse} phases where optimization is performed exclusively over a fixed \emph{sparse support}, and a subset of the weights and their gradients are fixed at zero, leading to computational savings. (This is in contrast to \emph{error feedback} algorithms, e.g.~\citep{courbariaux2016binarized, lin2019dynamic} which require computing fully-dense gradients, even though the weights themselves may be sparse.)
The process uses the same hyper-parameters, including the number of epochs, as regular training, and the frequency and length of the phases can be safely set to standard values, e.g. 5--10 epochs. 
We ensure that training ends on a \emph{sparse} phase, and return the resulting \emph{sparse} model, as well as the last \emph{dense} model obtained at the end of a \emph{dense} phase. 
This dense model may be additionally fine-tuned for a short period, leading to a more accurate \emph{dense-finetuned} model, which we usually find to match the accuracy of the \emph{dense baseline}. 

We emphasize that algorithms alternating sparse and dense training phases for deep neural networks have been previously investigated~\citep{jin2016training, han2016dsd}, but with the different goal on using sparsity as a regularizer to obtain \emph{more accurate dense models}. Relative to these works, our goals are two-fold: we aim to produce highly-accurate, highly-sparse models, but also to maximize the fraction of training time for which optimization is performed over a sparse support, leading to computational savings. 
Further, we are the first to provide convergence guarantees for variants of this approach. 

We perform an extensive empirical investigation, showing that AC/DC provides consistently good results on a wide range of models and tasks (ResNet~\citep{he2016deep}  and MobileNets~\citep{howard2017mobilenets} on the ImageNet~\citep{imagenet} / CIFAR~\citep{cifar100} datasets, and Transformers~\citep{vaswani2017attention, dai2019transformer} on WikiText~\citep{wikitext103}), under standard values of the training hyper-parameters. 
Specifically, when executed on the same number of training epochs,
our method outperforms all previous \emph{sparse training} methods in terms of the accuracy of the resulting sparse model, often by significant margins.
This comes at the cost of slightly higher theoretical computational cost relative to prior sparse training methods, although AC/DC usually reduces  training FLOPs to 45--65\% of the dense baseline. 
AC/DC is also close to the accuracy of state-of-the-art post-training pruning methods~\citep{kusupati2020soft, singh2020woodfisher} at medium  sparsities (80\% and 90\%); surprisingly, it \emph{outperforms} them in terms of accuracy, at higher sparsities. In addition, AC/DC is flexible with respect to the structure of the ``sparse projection'' applied at each compressed step: we illustrate this by obtaining \emph{semi-structured} pruned models using the 2:4 sparsity pattern efficiently supported by new NVIDIA hardware~\citep{mishra2021accelerating}. 
Further, we show that the resulting sparse models can provide significant real-world speedups for DNN inference on CPUs~\citep{NM}.  
\vspace{-0.2em}

An interesting feature of AC/DC is that it allows for accurate dense/sparse co-training of models. 
Specifically, at medium sparsity levels (80\% and 90\%), the method allows the co-trained dense model to recover the dense baseline accuracy via a short fine-tuning period.
In addition, dense/sparse co-training provides us with a lens into the training dynamics, in particular relative to the sample-level accuracy of the two models, but also in terms of the dynamics of the sparsity masks. 
Specifically, we observe that co-trained sparse/dense pairs have higher sample-level agreement than sparse/dense pairs obtained via post-training pruning, and that weight masks still change later in training. 

Additionally, we probe the accuracy differences between sparse and dense models, by examining their ``memorization'' capacity~\citep{zhang2016understanding}. 
For this, we perform dense/sparse co-training in a setting where a small number of valid training samples have \emph{corrupted labels}, and examine how these samples are classified during dense and sparse phases, respectively. 
We observe that the sparse model is less able to ``memorize'' the corrupted labels, and instead often classifies the corrupted samples to their \emph{true} (correct) class. 
By contrast, during dense phases model can easily ``memorize'' the corrupted labels. (Please see Figure~\ref{fig:random-correct-no-da-main} for an illustration.) 
This suggests that one reason for the higher accuracy of dense models is their ability to ``memorize'' hard-to-classify samples.

\vspace{-0.5em}
\section{Related Work}
\vspace{-0.5em}

There has recently been tremendous research interest into pruning techniques for DNNs; we direct the reader to the recent surveys of~\cite{gale2019state} and~\cite{hoefler2021sparsity} for a more comprehensive overview. 
Roughly, most DNN pruning methods can be split as (1) \emph{post-training} pruning methods, which start from an accurate dense baseline, and remove weights, followed by fine-tuning;  
and (2) \emph{sparse training} methods, which perform weight removal during the training process itself. (Other categories such as \emph{data-free} pruning methods~\citep{lee2018snip, tanaka2020pruning} exist, but they are beyond our scope.) We focus on \emph{sparse training}, although we will also compare against state-of-the-art post-training methods.

Arguably, the most popular metric for weight removal is \emph{weight magnitude}~\citep{hagiwara1994, han2015learning, zhu2017prune}. Better-performing approaches exist, such as second-order metrics~\citep{lecun1990optimal, hassibi1993optimal, 2017-dong, singh2020woodfisher}, or Bayesian approaches~\citep{molchanov2017variational}, but they tend to have higher computational and implementation cost. 

The general goal of \emph{sparse training} methods is to perform both the forward (inference) pass \emph{and the backpropagation} pass over a sparse support, leading to computational gains during the training process as well.
One of the first approaches to maintain sparsity throughout training was Deep Rewiring~\citep{bellec2017deep}, where SGD steps applied to positive weights are augmented with random walks in parameter space, followed by inactivating negative weights. To maintain sparsity throughout training, randomly chosen inactive connections are re-introduced in the ``growth'' phase. 
Sparse Evolutionary Training (SET)~\citep{mocanu2018scalable} introduces a non-uniform sparsity distribution across layers, which scales with the number of input and output channels, and trains sparse networks by pruning weights with smallest magnitude and re-introducing some weights randomly. 
RigL~\citep{evci2020rigging} prunes weights at random after a warm-up period, and then periodically performs weight re-introduction using a combination of connectivity- and gradient-based statistics, which require periodically evaluating full gradients. RigL can lead to state-of-the-art accuracy results even compared to post-training methods; however, to achieve high accuracy it requires significant additional data passes (e.g. 5x) relative to the dense baseline. 
Top-KAST~\citep{jayakumar2020top} alleviated the drawback of periodically having to evaluate dense gradients by updating the sparsity masks using gradients \emph{of reduced sparsity} relative to the weight sparsity. 
The latter two methods set the state-of-the-art for sparse training: when executing for the same number of epochs as the dense baseline, they provide computational reductions the order of 2x, while the accuracy of the resulting sparse models is lower than that of leading post-training methods, executed at the same sparsity levels. 
To our knowledge, none of these methods have convergence guarantees. 

Another approach towards faster training is \emph{training sparse networks from scratch}. The masks are updated by continuously pruning and re-introducing weights. For example, \cite{lin2019dynamic} uses magnitude pruning after applying SGD on the dense network, whereas 
\cite{dettmers2019sparse} update the masks by re-introducing weights with the highest gradient momentum. STR~\citep{kusupati2020soft} learns a separate pruning threshold for each layer and allows sparsity both during forward and backward passes; however, the desired sparsity can not be explicitly imposed, and the network has low sparsity for a large portion of training. These methods can lead to only limited computational gains, since they either require dense gradients, or the sparsity level cannot be imposed.
By comparison, our method provides models of similar or better accuracy at the same sparsity, with  computational reductions. 
We also obtain dense models that match the baseline accuracy, with a fraction of the baseline FLOPs.

The idea of alternating sparse and dense training phases has been examined before in the context of neural networks, but with the goal of using temporary sparsification as a regularizer. 
Specifically, \emph{Dense-Sparse-Dense (DSD)}~\citep{han2016dsd} proposes to first \emph{train a dense model to full accuracy}; this model is then sparsified via magnitude; next, optimization is performed over the sparse support, followed by an additional optimization phase over the full dense support. Thus, this process is used as a regularization mechanism for the dense model, which results in relatively small, but consistent accuracy improvements relative to the original dense model. In \cite{jin2016training}, the authors propose a similar approach to DSD, but alternate sparse phases during the regular training process. The resulting process is similar to AC/DC, but, importantly, the goal of their procedure is to return a \emph{more accurate dense model}. (Please see their Algorithm 1.) For this, the authors use relatively low sparsity levels, and gradually increase sparsity during optimization; they observe accuracy improvements for the resulting dense models, at the cost of increasing the total number of epochs of training.  
By contrast, our focus is on obtaining accurate \emph{sparse} models, while reducing computational cost, and executing the dense training recipe. 
We execute at higher sparsity levels, and on larger-scale datasets and models. In addition, we also show that the method works for other sparsity patterns, e.g. the 2:4 semi-structured  pattern~\citep{mishra2021accelerating}. 

More broadly, the Lottery Ticket Hypothesis (LTH)~\citep{frankle2018lottery} states that sparse networks can be trained in isolation \emph{from scratch} to the same performance as a post-training pruning baseline, by starting from the ``right'' weight and sparsity mask initializations,  optimizing only over this sparse support.
However, initializations usually require the availability of the fully-trained dense model, falling under \emph{post-training} methods. 
There is still active research on replicating these intriguing findings to large-scale models and datasets~\citep{gale2019state, frankle2020linear}. 
Previous work \cite{gale2019state, zhu2017prune} have studied progressive sparsification during regular training, which may also achieve training time speed-up, after a sufficient sparsity level has been achieved. However, AC/DC generally achieves a better trade-off between validation accuracy and training time speed-up, compared to these methods. 

Parallel work by~\citep{mohtashami2021simultaneous} investigates a related approach, but focusing on low-rank decompositions for Transformer models. 
Both their analytical approach and their application domain are different to the ones of the current work.

\section{Alternating Compressed / DeCompressed (AC/DC) Training}
\label{sec:method}

\subsection{Background and Assumptions}
\label{subsec:theory}

Obtaining \emph{sparse} solutions to optimization problems is a problem of interest in several areas~\citep{candes2006near, blumensath2008iterative, foucart2011hard}, where the goal is to minimize a function $f:\mathbb{R}^N \rightarrow \mathbb{R}$ under sparsity constraints:
\begin{equation}
    \min_{\theta \in \mathbb{R}^N} f(\theta) \quad \text{s.t.} \quad \| \theta \|_0 \leq k\,. 
\label{eq:optim_l0}
\end{equation}
For the case of $\ell_2$ regression, $f(\theta)  = \| b - A\theta\|^2_2$, a solution has been provided by Blumensath and Davies~\cite{blumensath2008iterative}, known as the \emph{Iterative Hard Thresholding (IHT)} algorithm, and subsequent work \citep{foucart2011hard, foucart2012sparse, yuan2014gradient} provided theoretical guarantees for the linear operators used in compressed sensing. 
The idea consists of alternating gradient descent (GD) steps and applications of a thresholding operator to ensure the $\ell_0$ constraint is satisfied. 
More precisely, $T_k$ is defined as the ``top-k'' operator, which keeps the largest $k$ entries  of a vector $\theta$ in absolute value, and replaces the rest with $0$. The IHT update at step $t+1$ has the following form:
\begin{equation}
    \theta_{t + 1} = T_k(\theta_t - \eta \nabla f(\theta_t)).
\label{eq:iht-step}
\end{equation}

Most convergence results for IHT assume deterministic gradient descent steps. For DNNs, stochastic methods are preferred, so we describe and analyze a stochastic version of IHT.

\parhead{Stochastic IHT.} 
We  consider functions $f: \mathbb{R}^N \rightarrow \mathbb{R}$, for which we can compute stochastic gradients $g_{\theta}$, which are unbiased estimators of the true gradient $\nabla f(\theta)$. Define the \emph{stochastic} IHT update as:
\begin{equation}
    \theta_{t+1} = T_k(\theta_t - \eta g_{\theta_t}).
\label{eq:stochastic-iht}
\end{equation}

This formulation covers the practical case where the stochastic gradient $g_{\theta}$ corresponds to a \emph{mini-batch} stochastic gradient. Indeed, as in practice $f$ takes the form $f(\theta) = \frac{1}{m}\sum_{i=1}^m f(\theta; x_i)$,  where  $S=\{x_1, \ldots, x_m\}$ are data samples, the stochastic gradients obtained via backpropagation take the form $\frac{1}{\vert B\vert}\sum_{i \in B} \nabla f(\theta; x_i)$, where $B$ is a sampled mini-batch. 
We aim to prove strong convergence bounds for stochastic IHT, under common assumptions that arise in the context of training DNNs. 

\parhead{Analytical Assumptions.} Formally, our analysis uses the following assumptions on $f$.
\begin{enumerate}
    \item Unbiased gradients with variance $\sigma$:
       $\mathbb{E}[g_{\theta} | \theta] = \nabla f(\theta),\textnormal{ and } \, \mathbb{E}[ \| g_{\theta} - \nabla f(\theta) \|^2 ] \leq \sigma^2\,.$
    \item Existence of a $k^*$-sparse minimizer $\theta^*$:
        $\exists \theta^* \in \arg\min_{\theta} f(\theta), \textnormal{ s.t. } \|\theta^*\|_0 \leq k^*\,.$
    \item For $\beta > 0$, the $\beta$-smoothness condition when restricted to $t$ coordinates ($(t,\beta)$-smoothness):
    \begin{equation} 
    f(\theta + \delta) \leq f(\theta) + \nabla f(\theta)^\top \delta + \frac{\beta}{2} \|\delta\|^2, \, \textnormal{ for all } \theta, \delta \textnormal{ s.t. } \|\delta\|_0 \leq t\,.
    \end{equation}
    \item For $\alpha > 0$ and number of indices $r$, the \emph{$r$-concentrated Polyak-\L ojasiewicz ($(r, \alpha)$-CPL) condition}:
    \begin{equation}\label{eq:conc-pl}
        \|T_{r}(\nabla f(\theta))\| \geq \frac{\alpha}{2} (f(\theta) - f(\theta^*))\,, \textnormal{ for all } \theta.
    \end{equation}
\end{enumerate}

The first assumption is standard in stochastic optimization, while the existence of very sparse minimizers is a known property in over-parametrized DNNs~\citep{frankle2018lottery}, and is the very premise of our study. 
Smoothness is also a standard assumption, e.g.~\citep{lin2019dynamic}---we only require it along \emph{sparse} directions, which is a strictly weaker assumption. 
The more interesting requirement for our convergence proof is the $(r, \alpha)$-CPL condition in  Equation (\ref{eq:conc-pl}), which we now discuss in detail.

The standard Polyak-\L ojasiewicz (PL) condition~\citep{karimi2016linear}  is common in non-convex optimization, and versions of it are essential in the analysis of DNN training~\citep{liu2020toward, allen2019convergence}. 
Its standard form states that small gradient norm, i.e. approximate stationarity, implies closeness to optimum in function value. 
We require a slightly stronger version, in terms of the norm of the gradient contributed by its largest coordinates in absolute value. This restriction appears necessary for the success of IHT methods, as the sparsity enforced by the truncation step automatically reduces the progress ensured by a gradient step to an amount proportional to the norm of the top-$k$ gradient entries. This strengthening of the PL condition is supported both theoretically, by the mean-field view, which argues that gradients are sub-gaussian~\citep{shevchenko2020landscape}, and by empirical validations of this behaviour~\citep{alistarh2018convergence, shi2019understanding}.

We are now ready to state our main analytical result. 

\begin{restatable}{thm}{mainthm}
\label{thm:iht-pl}
Let $f:\mathbb{R}^N \rightarrow \mathbb{R}$ be a function with a $k^*$-sparse minimizer $\theta^*$. 
Let $\beta > \alpha > 0$ be parameters, let $k = C\cdot k^* \cdot (\beta/\alpha)^2$ for some appropriately chosen constant $C$, and suppose that $f$ is $(2k+3k^*,\beta)$-smooth and $(k^*, \alpha)$-CPL.
For initial parameters $\theta_{0}$ and precision $\epsilon>0$, 
given access to stochastic gradients with variance $\sigma$,
stochastic IHT~(\ref{eq:stochastic-iht}) converges
in $O\left(\frac{\beta}{\alpha}\cdot\ln\frac{f\left(\theta_{0}\right)-f\left(\theta^{*}\right)}{\epsilon}\right)$
iterations to a point $\theta$ with $\|\theta\|_0\leq k$,
such that
\vspace{-0.2cm}
\begin{equation*}
    \mathbb{E}\left[f\left(\theta\right)-f\left(\theta^{*}\right)\right]\leq\epsilon+\frac{16\sigma^{2}}{\alpha}.
\end{equation*}
\end{restatable}

Assuming a fixed objective function $f$ and tolerance $\epsilon$, we can obtain lower loss and faster running time by either 1) increasing the support $k$ demanded from our approximate minimizer $\theta$ relative to the optimal $k^*$, or by reducing the gradient variance. 
We provide a complete proof of this result in the Appendix. Our analysis approach also works in the absence of the CPL condition (Theorem~\ref{thm:iht-nonconv}), 
in which case we prove that a version of the algorithm can find sparse nearly-stationary points.
As a bonus, we also simplify existing analyses for IHT and extend them to the stochastic case (Theorem~\ref{thm:vanilla-iht-stoch}). 
Another interpretation of our results is in showing that, under our assumptions, \emph{error feedback}~\citep{lin2019dynamic} is not necessary for recovering good sparse minimizers; this has practical implications, as it allows us to perform fully-sparse back-propagation in sparse optimization phases. 
Next, we discuss our practical implementation, and its connection to these theoretical results.

\begin{figure*}[t]
    \centering
        \includegraphics[height=0.9in]{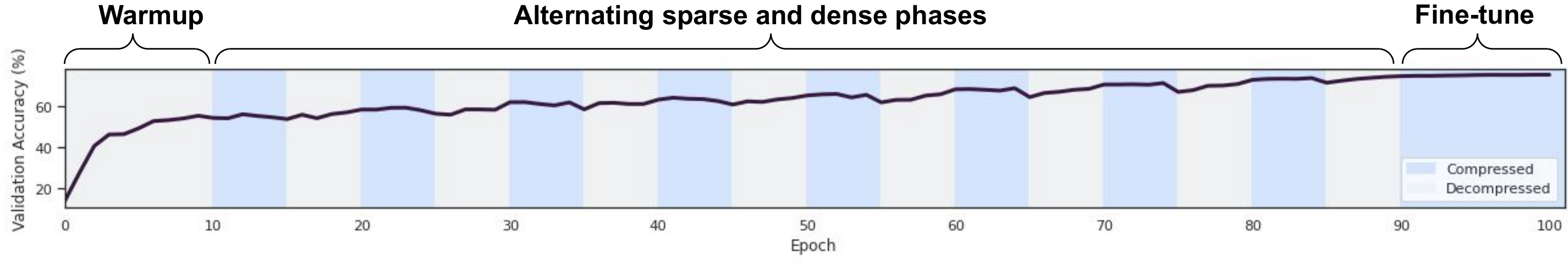}
        \caption{The AC/DC training process. After a short warmup we alternatively prune to maximum sparsity and restore the pruned weights. The plot shows the sparsity and validation accuracy throughout the process for a sample run on ResNet50/ImageNet at 90\% sparsity.}
        \label{fig:ac-dc}
\end{figure*}

\begin{algorithm}
\small{
\caption{Alternating Compressed/Decompressed (AC/DC) Training \label{alg:iht}}

\begin{algorithmic}[1]

\Require Weights $\theta \in \mathbb{R}^N$, data $S$, sparsity $k$, compression phases $\mathcal{C}$, decompression phases $\mathcal{D}$

\State Train the weights $\theta$ for $\Delta_w$ epochs \Comment{Warm-up phase}
\While{epoch $\leq$ max epochs}
    \If{entered a compression phase}  
    \State $\theta \leftarrow T_k(\theta, k)$ \Comment{apply compression (top-k) operator on weights}
    \State $m \leftarrow \mathbbm{1}[\theta_i \neq 0]$ \Comment{create masks}
    \EndIf
         
    \If{entered a decompression phase}
    \State $m \leftarrow \mathbbm{1}_N$ \Comment{reset all masks}
    \EndIf
    \State $\theta \leftarrow \theta \odot m$ \Comment{apply the masks (ensure sparsity for compression phases)}
    \State $\tilde{\theta} \leftarrow \{\theta_i | m_i \neq 0, 1\leq i \leq N \}$ \Comment{get the support for the gradients}
    \For{x mini-batch in $S$}
        \State $\theta \leftarrow \theta - \eta \nabla_{\tilde{\theta}} f(\theta; x)$ \Comment{optimize the active weights}
    \EndFor
    \State epoch $\leftarrow$ epoch $+1$
\EndWhile
\State \Return $\theta$
\end{algorithmic}
}
\end{algorithm}

\subsection{AC/DC: Applying IHT to Deep Neural Networks}

AC/DC starts from a standard DNN training flow, using standard optimizers such as SGD with momentum~\citep{qian1999momentum} or Adam~\citep{kingma2014adam}, and it preserves all standard training hyper-parameters. It will only periodically modify the \emph{support} for optimization. 
Please see Algorithm~\ref{alg:iht} for pseudocode. 

We partition the set of training epochs into \emph{compressed} epochs $\mathcal{C}$, and \emph{decompressed} epochs $\mathcal{D}$. We begin with a \emph{dense warm-up} period of $\Delta_w$ consecutive epochs, during which regular dense (decompressed) training is performed. 
We then start alternating \emph{compressed optimization} phases of length $\Delta_c$ epochs each, with \emph{decompressed (regular) optimization} phases of length $\Delta_d$ epochs each. The process completes on a compressed fine-tuning phase, returning an accurate sparse model. 
Alternatively, if our goal is to return a dense model matching the baseline accuracy, we take the best dense checkpoint obtained during alternation, and fine-tune it over the entire support. In practice, we noticed that allowing a longer final decompressed phase of length $\Delta_D > \Delta_d$ improves the performance of the dense model, by allowing it to better recover the baseline accuracy after fine-tuning. 
Please see Figure~\ref{fig:ac-dc} for an illustration of the schedule.

In our experiments, we focus on the case where the compression operation is unstructured or semi-structured pruning. In this case, at the beginning of each sparse optimization phase, we apply the top-k operator across all of the network weights to obtain a mask $M$ over the weights $\theta$. The top-k operator is applied globally across all of the network weights, and will represent the sparse support over which optimization will be performed for the rest of the current sparse phase. 
At the end of the sparse phase, the mask $M$ is reset to all-$1$s, so that the subsequent dense phase will optimize over the full dense support. Furthermore, once all weights are re-introduced, it is beneficial to reset to $0$ the gradient momentum term of the optimizer; this is particularly useful for the weights that were previously pruned, which would otherwise have stale versions of gradients.

\parhead{Discussion}.
Moving from IHT to a robust implementation in the context of DNNs required some adjustments. 
First, each \emph{decompressed  phase} can be directly mapped to a \emph{deterministic/stochastic IHT} step, where, instead of a single gradient step in between consecutive truncations of the support, we perform several stochastic steps. These additional steps improve the accuracy of the method in practice, and we can bound their influence in theory as well, although they do not necessarily provide better bounds. 
This leaves open the interpretation of the \emph{compressed phases}: for this, notice that the core of the proof for Theorem~\ref{thm:iht-pl} is in showing that a single IHT step significantly decreases the expected value of the objective; using a similar argument, we can prove that additional optimization steps over the sparse support can only improve convergence. 
Additionally, we show convergence for a variant of IHT closely following AC/DC (please see Corollary~\ref{cor:iht-acdc} in the Supplementary Material), but the bounds do not improve over Theorem~\ref{thm:iht-pl}. However, this additional result confirms that the good experimental results obtained with AC/DC are theoretically motivated.

\section{Experimental Validation}
\label{sec:experiments}

\parhead{Goals and Setup.} We tested AC/DC on image classification tasks (CIFAR-100~\citep{cifar100} and ImageNet~\citep{imagenet}) and on language modelling tasks~\citep{wikitext103} using the Transformer-XL model \citep{dai2019transformer}. The goal is to examine the \emph{validation accuracy} of the resulting sparse and dense models, versus the induced sparsity, as well as the number of FLOPs used for training and inference, relative to other sparse training methods. Additionally, we compare to state-of-the-art post-training pruning methods~\citep{singh2020woodfisher}. 
We also examine prediction differences between the sparse and dense models. 
We use PyTorch \citep{pytorch} for our implementation, Weights \& Biases \citep{wandb} for experimental tracking, and NVIDIA GPUs for training. 
All reported image classification experiments were performed in triplicate by varying the random seed; we report mean and standard deviation. 
Due to computational limitations, the language modelling experiments were conducted in a single run.

\parhead{ImageNet Experiments.}
On the ImageNet dataset~\citep{imagenet}, we test AC/DC on ResNet50 \citep{he2016deep} and MobileNetV1 \citep{howard2017mobilenets}. In all reported results, the models were trained for a fixed number of 100 epochs, using SGD with momentum. We use a cosine learning rate scheduler and training hyper-parameters following  \cite{kusupati2020soft}, but without label smoothing. The models were trained and evaluated using mixed precision (FP16). On a small subset of experiments, we noticed differences in accuracy of up to 0.2-0.3\% between AC/DC trained with full or mixed precision. However, the differences in evaluating the models with FP32 or FP16 are negligible (less than 0.05\%). Our dense ResNet50 baseline has $76.84\%$ validation accuracy. 
Unless otherwise specified, weights are pruned globally, based on their magnitude and in a single step. Similar to previous work, we did not prune biases, nor the Batch Normalization parameters. The sparsity level is computed with respect to all the parameters, except the biases and Batch Normalization parameters and this is consistent with previous work \cite{evci2020rigging, singh2020woodfisher}.

For all results, the AC/DC training schedule starts with a ``warm-up'' phase of dense training for 10 epochs, after which we alternate between compression and de-compression every 5 epochs, until the last dense and sparse phase. It is beneficial to allow these last two ``fine-tuning'' phases to run longer: 
the last decompression phase runs for 10 epochs, whereas the final 15 epochs are the compression fine-tuning phase. 
We reset SGD momentum at the beginning of every decompression phase. 
In total, we have an equal number of epochs of dense and sparse training; see Figure~(\ref{fig:valacc-rn50}) for an illustration. We use exactly the same setup for both ResNet50 and MobileNetV1 models, which resulted in high-quality sparse models. 
To recover a dense model with baseline accuracy using AC/DC, we finetune the best dense checkpoint obtained during training; practically, this replaces the last \emph{sparse} fine-tuning phase with a phase where the \emph{dense} model is fine-tuned instead.

\begin{minipage}[c]{0.5\textwidth}
\centering
\captionof{table}{\small{ResNet50/ImageNet, medium sparsity results.}}
\label{table:medium-sparse-rn50}
\vspace{-0.2cm}
\scalebox{0.63}{%
\begin{tabular}{@{}ccccc@{}}
\toprule
Method   & \begin{tabular}[c]{@{}c@{}}Sparsity\\ ($\%$)\end{tabular}  & \begin{tabular}[c]{@{}c@{}}Top-1\\ Acc. ($\%$)\end{tabular} &  \begin{tabular}[c]{@{}c@{}}GFLOPs\\ Inference\end{tabular} & \begin{tabular}[c]{@{}c@{}}EFLOPs\\ Train\end{tabular} \\ \midrule 
Dense    & $0$      & $76.84$ &  $8.2$ & $3.14$ \\ \midrule \midrule
\bf{AC/DC}  & $80$   & $76.3 \pm 0.1$  &  $0.29 \times$ & $0.65 \times$ \\
RigL$_{1\times}$  & $80$  & $74.6 \pm 0.06$   & $0.23 \times$ & $0.23 \times$  \\
RigL$_{1\times}$(ERK)  & $80$  & $75.1 \pm 0.05$   & $0.42 \times$ & $0.42 \times$  \\
Top-KAST & $80$ fwd, $50$ bwd & $75.03$  & $0.23\times$ & $0.32\times$ \\
\hline
STR  & $79.55$    & $76.19$  & $0.19 \times$ & -    \\
\bf{WoodFisher} & 80 & $\bf{76.76}$  & $0.25 \times$ & - \\
\midrule
\midrule
\textbf{AC/DC}   & $90$     & $75.03 \pm 0.1 $   & $0.18 \times$ & $0.58 \times$ \\
RigL$_{1\times}$   & $90$     & $72.0 \pm 0.05$ & $0.13 \times$  & $0.13 \times$  \\
RigL$_{1\times}$ (ERK)   & $90$     & $73.0 \pm 0.04$ & $0.24 \times$  & $0.25 \times$  \\
Top-KAST  & $90$ fwd, $80$ bwd   & $74.76$   & $0.13\times$ & $0.16\times$   \\
\hline
STR  & $90.23$ & $74.31$  & $0.08 \times$  & -  \\
\bf{WoodFisher} & $90$  & $\bf{75.21}$ & $0.15 \times$  & -   \\ 
\bottomrule 
\end{tabular}}
\end{minipage}
\hspace{0.15cm}
\begin{minipage}[c]{0.46\textwidth}
\captionof{table}{\small{ResNet50/ImageNet, high sparsity results.}}
\label{table:high-sparse-rn50}
\vspace{-0.2cm}
\scalebox{0.62}{%
\begin{tabular}{@{}ccccc@{}}
\toprule
Method   & \begin{tabular}[c]{@{}c@{}}Sparsity\\ ($\%$)\end{tabular} & \begin{tabular}[c]{@{}c@{}}Top-1\\ Acc. ($\%$)\end{tabular} & \begin{tabular}[c]{@{}c@{}}GFLOPs\\ Inference\end{tabular} & \begin{tabular}[c]{@{}c@{}}EFLOPs\\ Train\end{tabular} \\ \midrule 
Dense  & $0$  & $76.84$       & $8.2$ & $3.14$ \\ \midrule \midrule
\bf{AC/DC}   & $95$    & $\bf{73.14} \pm 0.2 $    & $0.11 \times$ & $0.53 \times$ \\
RigL$_{1\times}$   & $95$   & $67.5 \pm 0.1$  & $0.08 \times$  & $0.08 \times$  \\
RigL$_{1\times}$ (ERK)   & $95$   & $69.7 \pm 0.17$  & $0.12 \times$  & $0.13 \times$  \\
Top-KAST  & $95$ fwd, $50$ bwd   & $71.96$   & $0.08\times$  & $0.22\times$   \\
\hline
STR  & $94.8$ & $70.97$  & $0.04 \times$  & -  \\
WoodFisher & $95$ & $72.12$  &  $0.09 \times$  & -   \\ 
\midrule \midrule
\bf{AC/DC}    & $98$     & $\bf{68.44} \pm 0.09 $  & $0.06 \times$ & $0.46 \times$ \\
Top-KAST & $98$ fwd, $90$ bwd & 67.06 & $0.05\times$ & $0.08 \times$ \\
\hline
STR  & $97.78$ & $62.84$  & $0.02 \times$  & -  \\
WoodFisher & $98$ & $65.55$  & $0.05 \times$  & -   \\
\bottomrule
\end{tabular}}
\end{minipage}

\vspace{0.1cm}

\parhead{ResNet50 Results.} 
Tables~\ref{table:medium-sparse-rn50}\&~\ref{table:high-sparse-rn50} contain the validation accuracy results across medium and high global sparsity levels, as well as inference and training FLOPs. Overall, AC/DC achieves higher validation accuracy than any of the state-of-the-art sparse training methods, when using the same number of epochs. 
At the same time, due to dense training phases, AC/DC has
higher FLOP requirements relative to RigL or Top-KAST at the same sparsity. 
At medium sparsities (80\% and 90\%), AC/DC sparse models are slightly less accurate than the state-of-the-art post-training methods (e.g. WoodFisher), by small margins. 
The situation is reversed at higher sparsities, where AC/DC produces more accurate models: the gap to the second-best methods (WoodFisher / Top-KAST) is of more than 1\% at 95\% and 98\% sparsity.

\begin{figure*}[t]
    \centering
    \begin{subfigure}[t]{0.5\textwidth}
        \centering
        \includegraphics[height=1.2in]{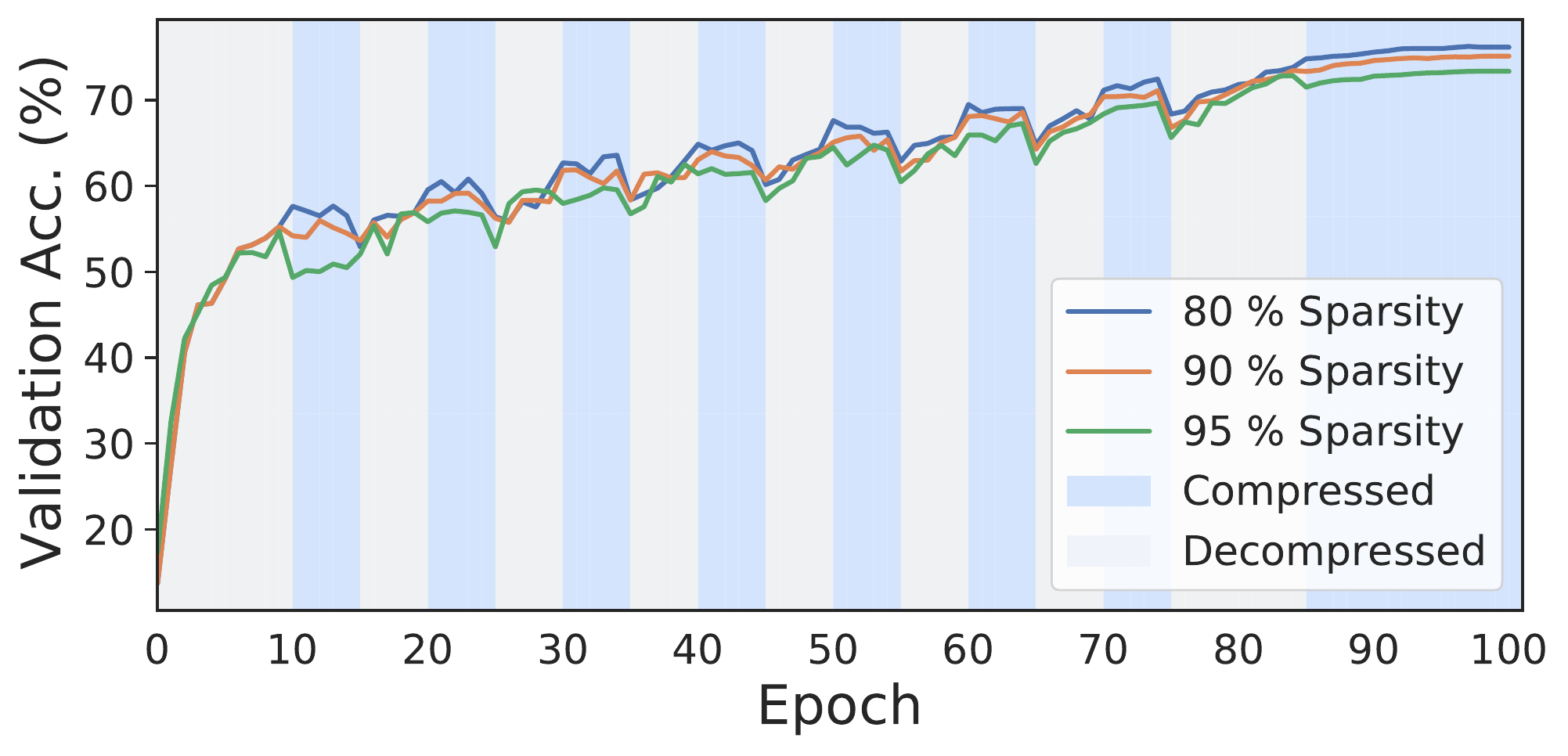}
        \caption{Sparsity pattern and validation accuracy vs. number of epochs (ResNet50/ImageNet).}
        \label{fig:valacc-rn50}
    \end{subfigure}%
    ~ 
    \begin{subfigure}[t]{0.5\textwidth}
        \centering
        \includegraphics[height=1.2in]{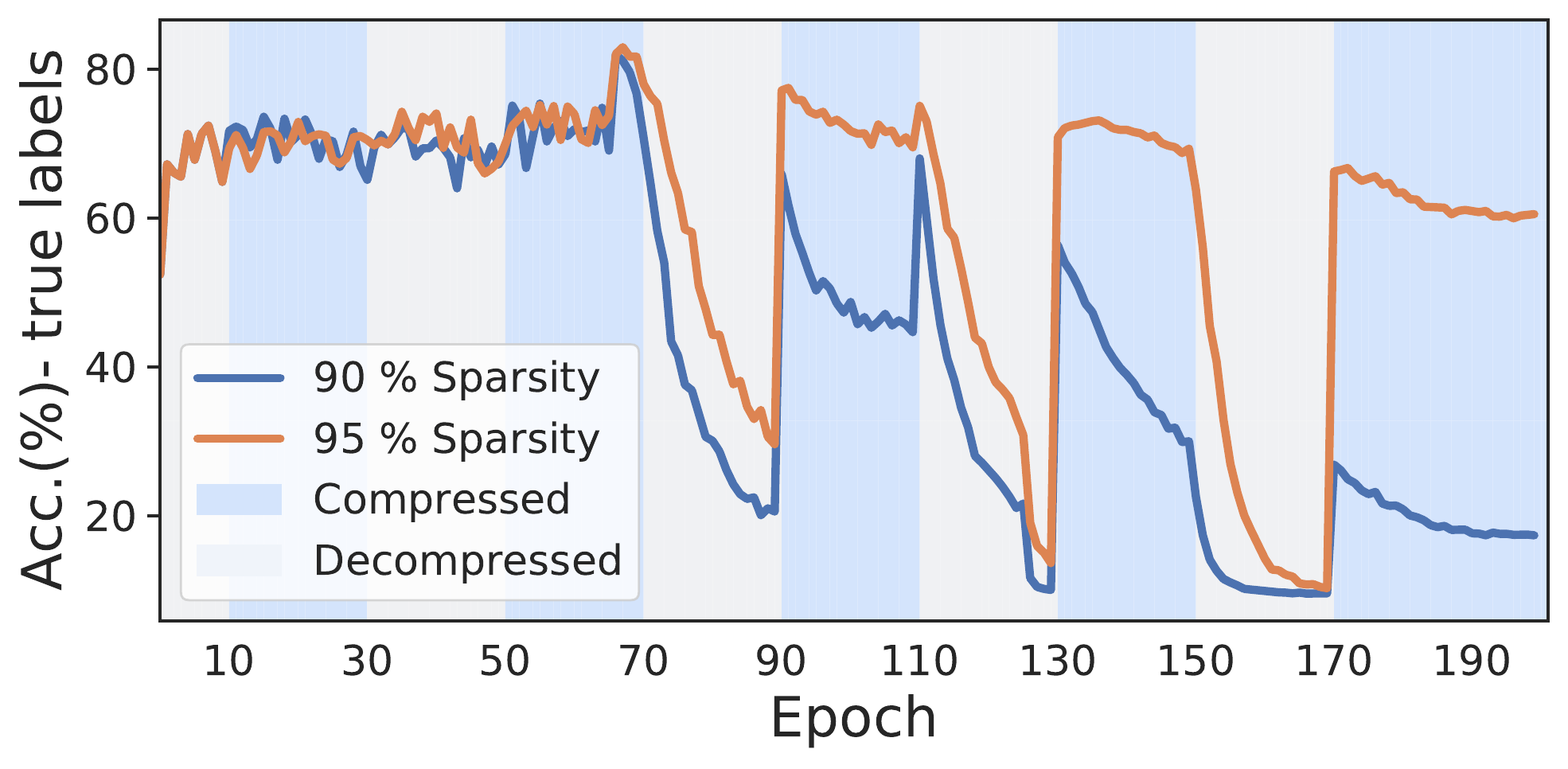}
        \caption{Percentage of samples with corrupted training labels classified to their \emph{true} class (ResNet20/CIFAR10).}
        \label{fig:random-correct-no-da-main}
    \end{subfigure}%
    \caption{Accuracy vs. sparsity during training, for the ResNet50/ImageNet experiment (left) and accuracy on the corrupted samples for ResNet20/CIFAR10, w.r.t. the \emph{true} class (right).  
    }
    \label{fig:acc-masks-rn50}
\end{figure*}

Of the existing sparse training methods, Top-KAST is closest in terms of validation accuracy to our sparse model, at 90\% sparsity. However, Top-KAST does not prune the first and last layers, whereas the results in the tables do not restrict the sparsity pattern. 
For fairness, we executed AC/DC using the same layer-wise sparsity distribution as Top-KAST, for both uniform and global magnitude pruning. For $90\%$ global pruning, results for AC/DC improved; the best sparse model reached 75.64\% validation accuracy (0.6\% increase over Table~\ref{table:medium-sparse-rn50}), while the best dense model had 76.85\% after fine-tuning. For uniform sparsity, our results were very similar: 75.04\% validation accuracy for the sparse model and 76.43\% - for the fine-tuned dense model. We also note that Top-KAST has better results  
at 98\% when increasing the number of training epochs 2 times, and considerably fewer training FLOPs (e.g. $15\%$ of the dense FLOPs).  
For fairness, we compared against all methods on a fixed number of 100 training epochs and we additionally trained AC/DC at high sparsity without pruning the first and last layers. Our results improved to $74.16\%$ accuracy for 95\% sparsity, and $71.27\%$ for 98\% sparsity, both surpassing Top-KAST with prolonged training. We provide a more detailed comparison in the Supplementary Material, which also contains results on CIFAR-100.

An advantage of AC/DC is that it provides \emph{both}  sparse and dense models at cost \emph{below} that of a single dense training run. 
For medium sparsity, the accuracy of the dense-finetuned model is very close to the dense baseline. 
Concretely, at 90\% sparsity, with 58\% of the total (theoretical) baseline training FLOPs, we obtain a \emph{sparse} model which is close to state of the art; 
in addition, by fine-tuning the best dense model, we obtain a dense model with  $76.56\%$ (average) validation accuracy. The whole process takes  at most 73\% of the baseline training FLOPs. In general, for 80\% and 90\% target sparsity, the dense models derived from AC/DC are able to recover the baseline accuracy, after finetuning, defined by replacing the final compression phase with regular dense training. The complete results are presented in the Supplementary Material, in Table~\ref{table:dense-rn50}.

The sparsity distribution over layers does not change dramatically during training; yet, the dynamic of the masks has an important impact on the performance of AC/DC. Specifically, we observed that masks update over time, although the change between consecutive sparse masks decreases. Furthermore, a small percentage of the weights remain fixed at $0$ even during dense training, which is explained by filters that are pruned away during the compressed phases.
Please see the Supplementary Material for additional results and analysis.

We additionally compare AC/DC with Top-KAST and RigL, in terms of the validation accuracy achieved depending on the number of training FLOPs. We report results at uniform sparsity, which ensures that the inference FLOPs will be the same for all methods considered. For AC/DC and Top-KAST, the first and last layers are kept dense, whereas for RigL, only the first layer is kept dense; however, this has a negligible impact on the number of FLOPs. Additionally, we experiment with extending the number of training iterations for AC/DC at 90\% and 95\% sparsity two times, similarly to Top-KAST and RigL which also provide experiments for extended training. The comparison between AC/DC, Top-KAST and RigL presented in Figure \ref{fig:flops-vs-acc-rn50} shows that AC/DC is similar or surpasses Top-KAST 2x at 90\% and 95\% sparsity, and RigL 5x at 95\% sparsity both in terms of training FLOPs and validation accuracy. Moreover, we highlight that extending the number of training iterations two times results in AC/DC models with uniform sparsity that surpass all existing methods at both 90\% and 95\% sparsity; namely, we obtain 76.1\% and 74.3\% validation accuracy with 90\% and 95\% uniform sparsity, respectively. 

\begin{figure*}[t]
    \centering
    \includegraphics[height=1.4in]{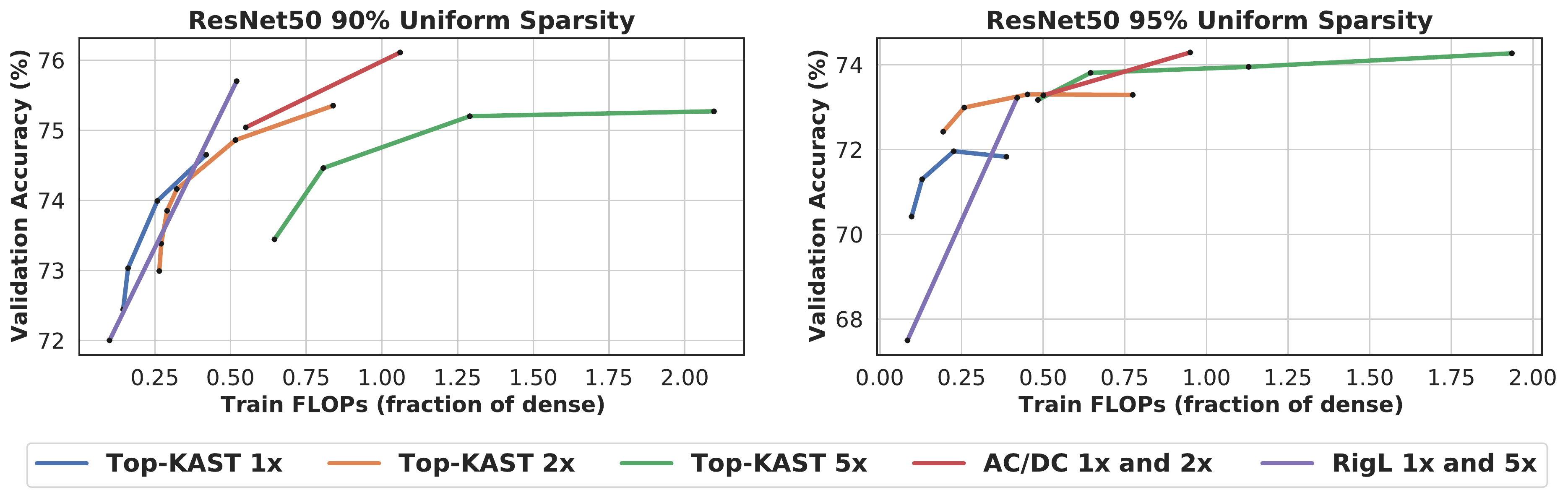}
    \caption{Training FLOPs vs validation accuracy for AC/DC, RigL and Top-KAST, with uniform sparsity, at 90\% and 95\% sparsity levels. (ResNet50/ImageNet).}
    \label{fig:flops-vs-acc-rn50}
\end{figure*}

Compared to purely sparse training methods, such as Top-KAST or RigL, AC/DC requires dense training phases. The length of the dense phases can be decreased, with a small impact on the accuracy of the sparse model. Specifically, we use dense phases of two instead of five epochs in length, and we no longer extend the final decompressed phase prior to the finetuning phase. For 90\% global sparsity, this resulted in 74.6\% validation accuracy for the sparse model, using 44\% of the baseline FLOPs. Similarly, for uniform sparsity, we obtain 74.7\% accuracy on the 90\% sparse model, with 40\% of the baseline FLOPs; this value can be further improved to 75.8\% validation accuracy when extending two times the number of training iterations. Furthermore, at 95\% uniform sparsity, we reach 72.8\% accuracy with 35\% of the baseline training FLOPs.

\parhead{MobileNet Results.} We perform the same experiment, using exactly the same setup, on the MobileNetV1 architecture~\citep{howard2017mobilenets}, which is compact and thus harder to compress. On a training budget of 100 epochs, our method finds sparse models with higher Top-1 validation accuracy than existing sparse- and post-training methods, on both 75\% and 90\% sparsity levels (Table~\ref{table:sparse-mobnet}). Importantly, AC/DC uses exactly the same hyper-parameters used for training the dense baseline~\citep{kusupati2020soft}. 
Similar to ResNet50, at  75\% sparsity, the \emph{dense-finetuned model} recovers the baseline performance, while for 90\% it is less than 1\% below the baseline.
The only method which obtains higher accuracy for the same sparsity is the version of RigL~\citep{evci2020rigging} which executes for 5x more training epochs than the dense baseline. 
However, this version also uses more computation than the dense model.
We limit ourselves to a fixed number of 100 epochs, the same used to train the dense baseline, which would allow for savings in training time. 
Moreover, RigL does not prune the first layer and the depth-wise convolutions, whereas for the results reported we do not impose any sparsity restrictions. Overall, we found that keeping these layers dense improved our results on 90\% sparsity by almost 0.5\%. Then, our results are quite close to RigL$_{2\times}$, with half the training epochs, and less training FLOPs. We provide a more detailed comparison in the Supplementary Material.

\begin{minipage}[c]{0.45\textwidth}
\centering
\captionof{table}{\small{MobileNetV1/ImageNet sparsity results}}
\vspace{-0.2cm}
\label{table:sparse-mobnet}
\scalebox{0.65}{%
\begin{tabular}{@{}ccccc@{}}
\toprule
Method   & \multicolumn{1}{c}{\begin{tabular}[c]{@{}c@{}}Sparsity \\ (\%)\end{tabular}}  & \multicolumn{1}{c}{\begin{tabular}[c]{@{}c@{}}Top-1 \\ Acc. (\%)\end{tabular}} &  \multicolumn{1}{c}{\begin{tabular}[c]{@{}c@{}}GFLOPs \\ Inference \end{tabular}} & \multicolumn{1}{c}{\begin{tabular}[c]{@{}c@{}}EFLOPs \\ Train \end{tabular}} \\ \midrule
Dense    & 0    & 71.78   & $1.1$  & $0.44$  \\ \midrule
\bf{AC/DC}    & $75$ & $\bf{70.3 \pm 0.07}$  & $0.34\times$ & $0.64\times$  \\
RigL$_{1\times}$ (ERK)   & $75$ & $68.39$  & $0.52\times$ & $0.53 \times$ \\
STR    & 75.28   & 68.35  & $0.18 \times$  & - \\
WoodFisher & 75.28   & 70.09  & $0.28 \times$ & -  \\ \midrule
\bf{AC/DC}   & 90 & $\bf{66.08} \pm 0.09$  & $0.18 \times$ & $0.56 \times$ \\ 
RigL$_{1\times}$ (ERK)   & 90   & $63.58$   & $0.27 \times$ & $0.29 \times$ \\
STR       & 89.01 & 62.1        & $0.07 \times$   & - \\
WoodFisher & 89 & 63.87 & -   & - \\
\bottomrule                
\end{tabular}}
\end{minipage}
\hspace{0.35cm}
\begin{minipage}[c]{0.5\textwidth}
\centering
    \captionof{table}{\small{Transformer-XL/WikiText sparsity results}}
    \label{table:transformer}
    \vspace{-0.2cm}
    \scalebox{0.66}{%
    \begin{tabular}{@{}ccccc@{}}
    \toprule
    Method   & Sparsity (\%) & \multicolumn{1}{c}{\begin{tabular}[c]{@{}c@{}}Perplexity \\ Sparse \end{tabular}} & \multicolumn{1}{c}{\begin{tabular}[c]{@{}c@{}}Perplexity \\ Dense\end{tabular}} & \multicolumn{1}{c}{\begin{tabular}[c]{@{}c@{}}Perplexity \\ Finetuned Dense \end{tabular}} \\ \midrule
    
    Dense    &  0   &  - & 18.95  & -  \\ \midrule
    AC/DC    & 80   & 20.65 & 20.24 & 19.54 \\
    AC/DC    & 80, 50 embed. & 20.83 & 20.25 & 19.68 \\
    \bf{Top-KAST}     & 80, 0 bwd & \bf{19.8}  & -  & - \\
    Top-KAST  & 80, 60 bwd & 21.3  & - & -  \\ \midrule
    \bf{AC/DC}    & 90 & \bf{22.32} & 21.0 & 20.28\\ 
    \bf{AC/DC}    & 90, 50 embed. & \bf{22.84} & 21.34 & 20.41 \\
    Top-KAST   & 90, 80 bwd  & 25.1   & - & - \\ \bottomrule            
    \end{tabular}}
\end{minipage}

\vspace{0.1cm}
\parhead{Semi-structured Sparsity.} 
We also experiment with the recent 2:4 sparsity pattern (2 weights out of each block of 4 are zero) proposed by NVIDIA, which ensures inference speedups on the Ampere architecture. Recently,~\cite{mishra2021accelerating} showed that accuracy can be preserved under this pattern, by re-doing the entire training flow. Also, \cite{zhou2021learning} proposed more general N:M structures, together with a method for training such sparse models from scratch.  We applied AC/DC to the 2:4 pattern, performing training from scratch and obtained sparse models with $76.64\%\pm 0.05$ validation accuracy, i.e. slightly below the baseline.
Furthermore, the dense-finetuned model fully recovers the baseline performance (76.85\% accuracy). We additionally experiment with using AC/DC with global pruning at 50\%; in this case we obtain sparse models that slightly improve the baseline accuracy to 77.05\%. This confirms our intuition that AC/DC can act as a regularizer, similarly to~\cite{han2016dsd}.

\parhead{Language Modeling.}
Next, we apply AC/DC to compressing NLP models. 
We use Transformer-XL~\citep{dai2019transformer}, on the  WikiText-103 dataset \citep{wikitext103}, with the standard model configuration with 18 layers and 285M parameters, trained using the Lamb optimizer~\citep{you2019large} and standard hyper-parameters, which we describe in the Supplementary Material. 
The same Transformer-XL model trained on WikiText-103 was used in Top-KAST \citep{jayakumar2020top}, which allows a direct comparison. Similar to Top-KAST, we did not prune the embedding layers, as this greatly affects the quality, without reducing computational cost. (For completeness, we do provide results when embeddings are pruned to 50\% sparsity.) 
Our sparse training configuration consists in starting with a dense warm-up phase of 5 epochs, followed by alternating between compression and decompression phases every 3 epochs; we follow with a longer decompression phase between epochs 33-39, 
and end with a compression phase between epochs 40-48. 
The results are shown in Table~\ref{table:transformer}. 
Relative to Top-KAST, our approach provides significantly improved test perplexity at 90\% sparsity, as well as better results at 80\% sparsity with sparse back-propagation.  
The results confirm that AC/DC is scalable and extensible.
We note that our hyper-parameter tuning for this experiment was minimal.  

\parhead{Output Analysis.} 
Finally, we probe the accuracy difference between the sparse and dense-finetuned models. 
We first examineed \emph{sample-level agreement} between sparse and dense-finetuned pairs produced by AC/DC, relative to model pairs produced by gradual magnitude pruning (GMP). 
Co-trained model pairs consistently agree on more samples relative to GMP: 
for example, on the 80\%-pruned ResNet50 model, the AC/DC model pair agrees on the Top-1 classification of 90\% of validation samples, whereas the GMP models agree on 86\% of the samples. The differences are better seen in terms of validation error (10\% versus 14\%), which indicate that the dense baseline and GMP model disagree on 40\% more samples compared to the AC/DC models. A similar trend holds for the \emph{cross-entropy} between model outputs. This is a potentially useful side-effect of the method; for example, in constrained environments where sparse models are needed, it is important to estimate their similarity to the dense ones.

Second, we analyze differences in ``memorization'' capacity~\citep{zhang2016understanding} between dense and sparse models. 
For this, we apply AC/DC to ResNet20 trained on a variant of CIFAR-10 where a subset of 1000 samples have randomly corrupted class labels, and examine the accuracy on these samples during training. We consider 90\% and 95\% sparsity AC/DC runs. 
Figure~\ref{fig:random-correct-no-da-main} shows the results, when the accuracy for each sample is measured with respect to the \emph{true, un-corrupted} label. 
During early training and during \emph{sparse phases}, 
the network tends to classify corrupted samples to their \emph{true class}, ``ignoring'' label corruption. 
However, as training progresses, due to dense training phases and lower learning rate, networks tend to ``memorize'' these samples, assigning them to their corrupted class. This phenomenon is even more prevalent at 95\% sparsity, where the network is less capable of memorization. 
We discuss this finding in more detail in the Supplementary Material.

\parhead{Practical Speedups.} 
One remaining question regards the potential of sparsity to provide real-world speedups. 
While this is an active research area, e.g.~\citep{elsen2020fast}, 
we partially address this concern in the Supplementary Material, by showing inference speedups for our models on a CPU inference platform supporting unstructured sparsity~\citep{NM}: for example, our 90\% sparse ResNet50 model provides  1.75x speedup for real-time inference (batch-size 1) on a resource-constrained processor with 4 cores, and 2.75x speedup on 16 cores at batch size 64, versus the dense model.  

\vspace{-0.5em}
\section{Conclusion, Limitations, and Future Work}
\label{sec:conclusion}
\vspace{-0.5em}

We introduced AC/DC---a method for co-training sparse and dense models, with theoretical guarantees. Experimental results show that AC/DC improves upon the accuracy of previous sparse training methods, and obtains state-of-the-art results at high sparsities. Importantly, we recover near-baseline performance for dense models and do not require extensive hyper-parameter tuning. We also show that AC/DC has potential for real-world speed-ups in  inference and training, with the appropriate software and hardware support. 
The method has the advantage of returning both an accurate standard model, and a compressed one. 
Our model output analysis confirms the intuition that sparse training phases act as a regularizer, preventing the (dense) model from memorizing corrupted samples. At the same time, they prevent the memorization of \emph{hard samples}, which can affect accuracy. 

The main limitations of AC/DC are its reliance on dense training phases, which limits the achievable training speedup, and the need for tuning the length and frequency of sparse/dense phases. We believe the latter issue can be addressed with more experimentation (we show some preliminary results in Section~\ref{sec:experiments} and Appendix \ref{sec:cifar100}); however, both the theoretical results and the output analysis suggest that dense phases may be \emph{necessary} for good accuracy. 
We plan to further investigate this in future work, together with applying AC/DC to other compression methods, such as quantization, as well as leveraging sparse training on hardware that could efficiently support it, such as Graphcore IPUs~\citep{graphcore}.

\vspace{-0.2em}
\acksection
\vspace{-0.2em}
This project has received funding from the European Research Council (ERC) under the European Union’s Horizon 2020 research and innovation programme (grant agreement No 805223 ScaleML), and a CNRS PEPS grant. This research was supported by the Scientific Service Units (SSU) of IST Austria through resources provided by Scientific Computing (SciComp). We would also like to thank Christoph Lampert for his feedback on an earlier version of this work, as well as for providing hardware for the Transformer-XL experiments.

\small{
\bibliographystyle{plain}
\bibliography{papers}
}

\newpage

\appendix
\appendixpage
\addappheadtotoc
\tableofcontents

\global\long\def\xx{\theta}%

\global\long\def\ss{k}%

\global\long\def\supp{\textnormal{supp}}%

\global\long\def\dom{\mathbb{R}^{N}}%

\global\long\def\PLlong{\text{Polyak-\L ojasiewicz}}%

\global\long\def\PLshort{\text{PL}}%

\global\long\def\cPLshort{\text{CPL}}%

\section{Convergence Proofs}
In this section we provide the convergence analysis for our algorithms. We prove Theorem~\ref{thm:iht-pl} and show as a corollary that under reasonable assumptions our implementation of AC/DC converges to a sparse minimizer.
\subsection{Overview}

We use the notation and assumptions defined in Section~\ref{subsec:theory}.
As all of our analyses revolve around bounding the progress made in a single iteration, to simplify notation we will generally use $\theta$ to denote the current iterate, and $\theta'$ to denote the iterate obtained after the IHT update:
\[
\theta' = T_k(\theta -  \eta g_{\theta})\,.
\]
Additionally, we let $S, S', S^*\subseteq [N]$ to denote the support of $\theta$, $\theta'$ and $\theta^*$ respectively, where $\theta^*$ is the promised $k^*$-sparse minimizer. 
Given an arbitrary vector $x$, we use $\supp(x)$ to denote its support, i.e. the set of coordinates where $x$ is nonzero.
We may also refer to the minimizing value of $f$ as $f^* = f(\theta^*)$.

Before providing more intuition for the analysis, we give the main theorem statements.

\subsubsection{Stochastic IHT for Functions with Concentrated PL Condition}
We restate the Theorem from Section~\ref{subsec:theory}.

\mainthm*

Additionally, we give a corollary that justifies our implementation of AC/DC. As opposed to the theoretical stochastic IHT algorithm, AC/DC performs a sequence of several dense SGD steps before applying a single pruning step. We show that even with this change we can provide theoretical convergence bounds, although these bounds can be weaker than the baseline IHT method under our assumptions.

\begin{cor}[Convergence of AC/DC]\label{cor:iht-acdc}
	Let $f:\dom \rightarrow \mathbb{R}$ be a function that decomposes as $f(\theta) = \frac{1}{m} \sum_{i=1}^m f_i(\theta)$, and has a $k^*$-sparse minimizer $\theta^*$. Let $\beta>\alpha>0$ be parameters, let $k = C\cdot k^* \cdot (\beta/\alpha)^2$ for some appropriately chosen constant $C$, suppose that each $f_i$ is $(N,\beta)$-smooth, and $L$-Lipschitz, and that $f$ is $(k^*,\alpha)$-CPL.
	
	Let $\Delta_c$ and $B$ be integers, and let $\{S_1, \dots, S_B\}$ be a partition of $[m]$ into $B$ subsets of cardinality $O(m/ B)$ each.
	Given $\theta$, let $g_{\theta}^{(i)} = \frac{1}{\vert S_i \vert} \sum_{j \in S_i} \nabla f_j(\theta)$.
	
	Suppose we replace the IHT iteration with a dense/sparse phase consisting of
	\begin{enumerate}
		\item $\Delta_c$ phases during each of which we perform a full pass over the data by performing the iteration $\theta' = \theta - \eta g_{\theta}^{(i)}$ for all $i \in [B]$, with an appropriate step size $\eta$;
		\item a pruning step implemented via an application the truncation operator $T_k$;
		\item an optional sparse training phase which fully optimizes $f$ over the sparse support.
	\end{enumerate}
	
	For initial parameters $\theta_0$ and precision $\epsilon >0$, this algorithm converges
	in $O\left(\frac{\beta}{\alpha}\cdot\ln\frac{f\left(\theta_{0}\right)-f\left(\theta^{*}\right)}{\epsilon}\right)$
	dense/sparse phases to a point $\theta$ with $\|\theta\|_0\leq k$,
	such that
	\vspace{-0.2cm}
	\begin{equation*}
	f\left(\theta\right)-f\left(\theta^{*}\right)\leq\epsilon+ O\left(\frac{L^2}{\alpha}\right).
	\end{equation*}
	
\end{cor}

To provide more intuition for these results, let us understand how various parameters affect convergence. As we will argue in more detail in Section~\ref{sec:proof-approach}, the main idea behind these algorithms is based on the vanilla IHT algorithm, which consists of alternating full gradient steps with pruning/truncation steps. Pruning to the largest $k$ coordinates in absolute value essentially represents projecting the current weights onto the non-convex set of $k$-sparse vectors, so a natural approach is to simply try to adapt the analysis of projected gradient methods to this setup. The major caveat of this idea is that projecting onto a non-convex set can potentially undo some of the progress made so far, unlike the standard case of projections over convex sets. However, we can argue that if we set the target sparsity $k$ sufficiently large compared to the promised sparsity $k^*$ of the optimum, the pruning step can only increase the function value by a fraction of how much it was decreased by the full gradient step. 

Notably, in order to guarantee this property, we require the target sparsity $k$ to be of order $\Omega(k^* \cdot (\beta/\alpha)^2)$, where $\beta/\alpha$ represents the "restricted condition number" of $f$. Hence "well conditioned" functions allow for better sparsity. The number of iterations is same as in vanilla smooth and strongly convex optimization, which only depends on the function $f$, and not on the target sparsity.

The next step is specializing this approach to the non-convex case, under the restricted smoothness and CPL conditions, when only stochastic gradients are available. To handle non-convexity, we show that in fact a weaker property than strong convexity is required, and we can achieve similar results only under these restricted assumptions. More importantly, we can also handle stochastic gradients, which occur naturally in deep learning. In fact, the stochastic variance ends up contributing an extra additive error of $O(\sigma^2/\alpha)$ to our final error bound (see Theorem~\ref{thm:iht-pl}). This can be reduced by decreasing variance, i.e. taking larger batches. Additionally, this term carries a dependence in the CPL parameter $\alpha$ -- the larger $\alpha$, the smaller the error. We leave as an open problem whether this dependence on variance can be eliminated for stochastic IHT methods.

Building on this, in Corollary~\ref{cor:iht-acdc} we provide theoretical guarantees for the algorithm we implemented. Notably, while stochastic IHT takes a single stochastic gradient step, followed by pruning, in practice we replace this with a dense training phase, followed by optimization in the sparse support.
We show that this does not significantly change the state of affairs. The dense training phase can be thought of as a replacement for a single stochastic gradient step. The difficulty in AC/DC comes from the fact that the total movement of the weights during the dense phase does not constitute an unbiased estimator for the true gradient, as it was previously the case. However, by damping down the learning rate we can see that under reasonable assumptions this total movement does not differ too much from the step we would have made using a single full gradient. The sparse training phase can only help, since our entire analysis is based on proving that the function value decreases in each step. Hence as long as sparse training reduces the function value, it can only improve convergence.

We provide the main proofs for these statements in Section~\ref{sec:stoch-iht-pl}.

\subsubsection{Stochastic IHT for Functions with Restricted Smoothness and Strong
	Convexity}
We also provide a streamlined analysis for stochastic IHT under standard assumptions.
Compared to \citep{jain2014iterative} which achieves similar guarantees (in particular
both suffer from a blow-up in sparsity that is quadratic in the restricted
condition number of $f$) we significantly simplify the analysis and
provide guarantees when only stochastic gradients are available. Notably, the quadratic dependence in condition number can be improved to linear, at the expense of a significantly more sophisticated algorithm which requires a longer running time~\citep{axiotis2020sparse}.
\begin{thm}\label{thm:vanilla-iht-stoch}
	Let $f : \dom \rightarrow \mathbb{R}$ be a function with a $\ss^{*}$-sparse minimizer $\xx^{*}$.
	Let $\beta > \alpha > 0$ be parameters, let $k = C\cdot k^* \cdot (\beta/\alpha)^2$ for some appropriately chosen constant $C$, and suppose that $f$ is $(2k+k^*,\beta)$-smooth and $(k+k^*,\alpha)$-strongly convex in the sense that
	\[
	f(\theta+\delta) \geq f(\theta) + \nabla f(\theta)^\top \delta + \frac{\alpha}{2} \|\delta\|^2, \textnormal{ for all } \theta, \delta \textnormal{ s. t. } \|\delta\|_0 \leq k+k^*.
	\]
	For initial parameters $\theta_0$ and precision $\epsilon > 0$, given access to stochastic gradients with variance $\sigma$, IHT~(\ref{eq:stochastic-iht}) converges in $O\left( \frac{\beta}{\alpha} \cdot \ln \frac{\|\theta_0 - \theta^*\|}{\epsilon/\beta + \sigma^2/(\alpha\beta)} \right)$ iterations to a point $\theta$ with $\|\theta\|_0 \leq k$, such that 
	\[
	\mathbb{E}\left[\left\Vert \xx-\xx^{*}\right\Vert ^{2}\right]\leq\frac{\epsilon}{\beta}+\frac{2\sigma^{2}}{\alpha\beta}
	\]
	and 
	\[
	\mathbb{E}\left[f\left(\xx\right)-f\left(\xx^{*}\right)\right]\leq\epsilon+\frac{2\sigma^{2}}{\alpha}\ .
	\]
\end{thm}
We prove this statement in Section~\ref{sec:iht-standard-proof}.

\subsubsection{Finding a Sparse Nearly-Stationary Point}
We can further relax the assumptions, and show that IHT can recover sparse iterates that are nearly-stationary (i.e. have
small gradient norm). Finding nearly-stationary points is a standard
objective in non-convex optimization. However, enforcing the sparsity
guarantee is not. Here we show that IHT can further
provide stationarity guarantees even with minimal assumptions. 

Intuitively, this results suggests that error feedback may not be necessary for converging to stationary points under parameter sparsity.

In this case, we alternate IHT steps with optimization steps over the sparse support, which reduce the norm of the gradient resteicted to the support to some target error $\epsilon$.
\begin{thm}
	\label{thm:iht-nonconv}Let $f:\dom\rightarrow\mathbb{R}$. Let $\beta,\epsilon> 0$ be parameters,  let $k$ be the target sparsity, and suppose that $f$ is $(2k,\beta)$-smooth. 
	Furthermore, suppose that after each IHT step with a step size $\eta = \beta^{-1}$ and target sparsity $k$, followed by optimizing over the support to gradient norm $\epsilon$,
	all the obtained iterates satisfy 
	\[
	\left\Vert \xx\right\Vert _{\infty}\leq R_{\infty}\ .
	\]
	For initial parameters $\theta_0$ and precision $\epsilon > 0$, given access to stochastic gradients with variance $\sigma$, in  $O\left(\beta\left(f\left(\xx_{0}\right)-f\left(\xx_{T}\right)\right)\cdot\min\left\{ \frac{1}{\epsilon^{2}},\frac{1}{\beta^{2}\ss R_{\infty}^{2}},\frac{1}{\beta^{2}\sigma^{2}}\right\} \right)$
	iterations we can obtain an $\ss$-sparse iterate $\xx$ such that
	\begin{align*}
	\left\Vert \nabla f\left(\xx\right)\right\Vert _{\infty} & =O\left(\beta R_{\infty}\sqrt{k}+\beta\sigma+\epsilon\right)\,.
	\end{align*}
\end{thm}
This theorem provides provable guarantees even in absence of properties that bound distance to the optimum via function value (such as restricted strong convexity or the $\cPLshort$ condition. Instead, it uses the reasonable assumption that the $\ell_\infty$ norm of all the sparse iterates witnessed during the optimization process is bounded by a parameter $R_\infty$. In fact, in practice we often notice that weights stay in a small range, so this assumption is well motivated.

Since it can not possibly offer guarantees on global optimality, this theorem instead provides a sparse near-stationary point, in the sense that the $\ell_\infty$ norm of the gradient at the returned point is small.
In fact this depends on several parameters, including the target sparsity $k$, smoothness $\beta$ and the promise on $R_\infty$. It is worth mentioning that the $\ell_\infty$ norm of the gradient depends on the norm of the output sparsity, which is independent on the sparsity of any promise, unlike in the previously analyzed cases. 

Several other works attempted to compute sparse near-stationary points~\cite{lin2019dynamic, mohtashami2021simultaneous}. As opposed to these we do not require feedback, nor do we sparsify the gradients. Instead, we simply prune the weights after updating them with a dense gradient, following which we optimize in the sparse support until we reach a point whose gradient has only small coordinates within that support. 

Since all iterates are in a small $\ell_\infty$  box, we can show that the progress guaranteed by the dense gradient step alone is "mildly" affected by pruning. Hence unless we are already at a near-stationary point in, which case the algorithm can terminate, we decrease the value of $f$ significantly, while only suffering a bit of penalty which is bounded by the parameters of the instance.

We provide full proofs in Section~\ref{subsec:sparse-stationary}.

\subsubsection{Proof Approach}
\label{sec:proof-approach}
Let us briefly explain the intuition behind our theoretical analyses.  We can view IHT as a version of projected gradient descent, where iterates are projected onto the \textit{non-convex} domain of $k$-sparse vectors. 

In general, when the domain is convex, projections do not hurt convergence. This is because under specific assumptions, gradient descent makes progress by provably decreasing the $\ell_2$ distance between the iterate and the optimal solution. As projecting the iterate back onto the convex domain can only improve the distance to the optimum, convergence is unaffected even in those constrained settings.

In our setting, as the set of $k$-sparse vectors is non-convex, the distance to the optimal solution can increase. However, we can show a trade-off between how much this distance increases and the ratio between the target and optimal sparsity $k/k^*$. 
\subparagraph{Intuition.}
Intuitively, consider a point $\widetilde{\theta}$ obtained by taking a gradient step
\begin{equation}
\widetilde{\theta} = \theta - \eta \nabla f(\theta)\,. \label{eq:vanilla-sgd-step}
\end{equation}
While this step provably decreases the distance to the optimum $\|\widetilde{\theta} - \theta^*\| \ll \|\theta - \theta^*\|$, after applying the projection 
by moving to the projection $T_k(\widetilde{\theta})$, the distance $\|T_k(\widetilde{\theta}) - \theta^*\|$ may increase again. The key is that this increase can be controlled. For example, the new distance can be bounded via triangle inequality by
\[
\|T_k(\widetilde{\theta}) - \theta^*\| \leq \|\widetilde{\theta} - T_k(\widetilde{\theta})\|  + \|\widetilde{\theta} - \theta^*\| \leq 2\|\widetilde{\theta} - \theta^*\| \,.
\]
The last inequality follows from the fact that by definition $T_k(\widetilde{\theta})$ is the closest $k$-sparse vector to $\widetilde{\theta}$ in $\ell_2$ norm, and thus the additional distance payed to move to this projected point is bounded by $\|\widetilde{\theta} - \theta^*\|$. Thus, if the gradient step (\ref{eq:vanilla-sgd-step}) made sufficient progress, for example 
$\|\widetilde{\theta} - \theta^*\| \leq \frac{1}{3} \|\theta - \theta^*\|$ we can conclude that the additional truncating step does not fully undo progress, as 
\[
\|T_k(\widetilde{\theta}) - \theta^*\| \leq 2\|\widetilde{\theta} - \theta^*\| \leq \frac{2}{3} \|\theta - \theta^*\|\,,
\]
so the iterate still converges to the optimum.

In reality, we can not always guarantee that a single gradient step reduces the distance by a large constant fraction -- as a matter of fact this is determined by how well the function $f$ is conditioned. However we can reduce the lost progress that is caused by the truncation step simply by increasing the number of nonzeros of the target solution, i.e. increasing the ratio $k / k^*$.

This is captured by the following crucial lemma which also appears in~\cite{jain2014iterative}. A short proof can be found in Section~\ref{sec:missing-proofs}.
\begin{lem}
	\label{lem:sparse-proj}Let $\xx^{*}$ be a $\ss^{*}$-sparse vector,
	and let $\xx$ be an $n$-sparse vector. Then
	\[
	\frac{\left\Vert T_{\ss}\left(\xx\right)-\xx\right\Vert ^{2}}{n-\ss}\leq\frac{\left\Vert \xx^{*}-\xx\right\Vert ^{2}}{n-\ss^{*}}\,.
	\]
\end{lem}

Its usefulness is made obvious in the proof of Theorem~\ref{thm:vanilla-iht-stoch} which explicitly tracks as a measure of progress the distance between the current iterate and the sparse optimum. This is shown in detail in Section~\ref{sec:iht-standard-proof}. 

The proofs of Theorems~\ref{thm:iht-pl} and~\ref{thm:iht-nonconv} are slightly more complicated, as they track progress in terms of function value rather than distance to the sparse optimum. However, similar arguments based on Lemma~\ref{lem:sparse-proj} still go through.  An added benefit is that these analyses show that alternating IHT steps with gradient steps over the sparse support can only help convergence, as these additional steps further decrease error in function value. This fact is important to theoretically justify the performance of the AC/DC algorithm. In the following sections we provide proofs for our theorems.

\subsection{Stochastic IHT for Non-Convex Functions with Concentrated PL Condition}
\label{sec:stoch-iht-pl}
In this section we prove Theorem~\ref{thm:iht-pl} by showing that the ideas developed before apply to non-convex settings. We analyze IHT for a class
of functions that satisfy a special version of the $\PLlong$ ($\PLshort$)
condition~\citep{karimi2016linear} which is standard in non-convex optimization, and certain versions of it were essential in several works analyzing the convergence of training methods for deep neural networks~\citep{liu2020toward, allen2019convergence}. Usually this
condition says that small gradient norm i.e. approximate stationarity
implies closeness to optimum in function value. Here we use the stronger $(r,\alpha)$-CPL condition (see Equation~\ref{eq:conc-pl}), which considers the norm of the gradient
contributed by its largest coordinates in absolute value.

We prove strong convergence bounds for functions that satisfy the
$\cPLshort$ condition. Compared to the classical $\PLlong$ condition,
this adds the additional assumption that most of the mass of the gradient
is concentrated on a small subset of coordinates. This phenomenon
has been witnessed in several instances, and is implicitly used in~\citep{lin2019dynamic}. %

Before proceeding with the proof we again provide a few useful lemmas.
\begin{lem}
	\label{lem:smooth-ub}If $f$ is $(\ell, \beta)$-smooth, then for any $\ell$-sparse $\delta$, one has that
	\[
	f\left(\xx+\delta\right)\leq f\left(\xx\right)+\frac{\beta}{2}\left\Vert \left(\frac{1}{\beta}\nabla f\left(\xx\right)+\delta\right)_{\supp\left(\delta\right)}\right\Vert ^{2}-\frac{1}{2\beta}\left\Vert \nabla f\left(\xx\right)_{\supp\left(\delta\right)}\right\Vert ^{2}\ .
	\]
\end{lem}

\begin{proof}
	Applying smoothness we bound
	\begin{align*}
	f\left(\xx+\delta\right) & \leq f\left(\xx\right)+\left\langle \nabla f\left(\xx\right),\delta\right\rangle +\frac{\beta}{2}\left\Vert \delta\right\Vert ^{2}\\
	& =f\left(\xx\right)+\frac{1}{2\beta}\left\Vert \nabla f\left(\xx\right)\right\Vert ^{2}+\left\langle \nabla f\left(\xx\right),\delta\right\rangle +\frac{\beta}{2}\left\Vert \delta\right\Vert ^{2}-\frac{1}{2\beta}\left\Vert \nabla f\left(\xx\right)\right\Vert ^{2}\\
	& =f\left(\xx\right)+\frac{1}{2}\left\Vert \frac{1}{\sqrt{\beta}}\nabla f\left(\xx\right)+\sqrt{\beta}\delta\right\Vert ^{2}-\frac{1}{2\beta}\left\Vert \nabla f\left(\xx\right)\right\Vert ^{2}\\
	& =f\left(\xx\right)+\frac{\beta}{2}\left\Vert \frac{1}{\beta}\nabla f\left(\xx\right)+\delta\right\Vert ^{2}-\frac{1}{2\beta}\left\Vert \nabla f\left(\xx\right)\right\Vert ^{2}\ .
	\end{align*}
	Next we notice that the contributions of the two terms $\frac{\beta}{2}\left\Vert \frac{1}{\beta}\nabla f\left(\xx\right)+\delta\right\Vert ^{2}$
	and $\frac{1}{2\beta}\left\Vert \nabla f\left(\xx\right)\right\Vert ^{2}$
	exactly match on the coordinates not touched by $\delta$. Hence everything
	outside the support of $\delta$ cancels out, which yields the desired
	conclusion.
\end{proof}
We require another lemma which will be very useful in the analysis. 
\begin{lem}
	\label{lem:trunc-and-opt}Let $\xx,g\in\dom$ such that $\supp\left(\xx\right)=S$,
	and let $S',S^{*}$ be some arbitrary subsets, with $\left|S'\right|=\left|S\right|>\left|S^{*}\right|$.
	Furthermore suppose that
	\[
	T_{\ss}\left(\xx+g\right)=\left(\xx+g\right)_{S'}\ .
	\]
	Then 
	\[
	\left\Vert \left(\xx+g\right)_{S\setminus S'}\right\Vert ^{2}-\left\Vert g_{S\cup S'}\right\Vert ^{2}\leq\left\Vert \left(\xx+g\right)_{Z\setminus S'}\right\Vert ^{2}-\left\Vert g_{S^{*}}\right\Vert ^{2}\ ,
	\]
	for some set $Z$ such that $\left|Z\setminus S'\right|\leq2\left|S^{*}\right|$.
\end{lem}

\begin{proof}
	We prove this as follows. We write
	\begin{align*}
	\left\Vert \left(\xx+g\right)_{S\setminus S'}\right\Vert ^{2}-\left\Vert g_{S\cup S'}\right\Vert ^{2} & =\left\Vert \left(\xx+g\right)_{\left(S^{*}\cup S\right)\setminus S'}\right\Vert ^{2}-\left\Vert \left(\xx+g\right)_{S^{*}\setminus\left(S\cup S'\right)}\right\Vert ^{2}-\left\Vert g_{S\cup S'}\right\Vert ^{2}\\
	& =\left\Vert \left(\xx+g\right)_{\left(S^{*}\cup S\right)\setminus S'}\right\Vert ^{2}-\left\Vert g_{S^{*}\setminus\left(S\cup S'\right)}\right\Vert ^{2}-\left\Vert g_{S\cup S'}\right\Vert ^{2}\\
	& =\left\Vert \left(\xx+g\right)_{\left(S^{*}\cup S\right)\setminus S'}\right\Vert ^{2}-\left\Vert g_{S^{*}\cup S\cup S'}\right\Vert ^{2}\\
	& =\left\Vert \left(\xx+g\right)_{\left(S^{*}\cup S\right)\setminus S'}\right\Vert ^{2}-\left\Vert g_{S'\setminus\left(S^{*}\cup S\right)}\right\Vert ^{2}-\left\Vert g_{S^{*}\cup S}\right\Vert ^{2}\ .
	\end{align*}
	Since 
	\[
	\left\Vert g_{S'\setminus\left(S^{*}\cup S\right)}\right\Vert ^{2}=\left\Vert \left(\xx+g\right)_{S'\setminus\left(S^{*}\cup S\right)}\right\Vert ^{2}\geq\left\Vert \left(\xx+g\right)_{R}\right\Vert ^{2}\ ,
	\]
	where $R$ is a subset $R\subseteq\left(S^{*}\cup S\right)\setminus S'$
	with $\left|R\right|=\left|S'\setminus\left(S^{*}\cup S\right)\right|$.
	Such a set definitely exists as 
	\begin{align*}
	\left|\left(S^{*}\cup S\right)\setminus S'\right| & \geq\left|S\setminus S'\right|=\left|S'\setminus S\right|\geq\left|S'\setminus\left(S^{*}\cup S\right)\right|=\left|R\right|\ .
	\end{align*}
	Hence we obtain that
	\[
	\left\Vert \left(\xx+g\right)_{S\setminus S'}\right\Vert ^{2}-\left\Vert g_{S\cup S'}\right\Vert ^{2}\leq\left\Vert \left(\xx+g\right)_{\left(\left(S^{*}\cup S\right)\setminus S'\right)\setminus R}\right\Vert ^{2}-\left\Vert g_{S^{*}\cup S}\right\Vert ^{2}\ .
	\]
	Note that 
	\begin{align*}
	\left|\left(\left(S^{*}\cup S\right)\setminus S'\right)\setminus R\right| & =\left|\left(S^{*}\cup S\right)\setminus S'\right|-\left|R\right|=\left|\left(S^{*}\cup S\right)\setminus S'\right|-\left|S'\setminus\left(S^{*}\cup S\right)\right|\\
	& \leq\left(\left|S^{*}\right|+\left|S\setminus S'\right|\right)-\left(\left|S'\setminus S\right|-\left|S^{*}\right|\right)\\
	& =2\left|S^{*}\right|\ .
	\end{align*}
	This concludes the proof.
\end{proof}
Using this we derive the following useful corollary.
\begin{cor}
	\label{cor:prog}Let $\xx,g\in\dom$ such that $\supp\left(\xx\right)=S$,
	and let $S',S^{*}$ be arbitrary subsets, with $\left|S'\right|=\left|S\right|>\left|S^{*}\right|$. Furthermore suppose that
	\[
	T_{\ss}\left(\xx+g\right)=\left(\xx+g\right)_{S'}\ .
	\]
	Then one has that 
	\[
	\left\Vert \left(\xx+g\right)_{S\setminus S'}\right\Vert ^{2}-\left\Vert g_{S\cup S'}\right\Vert ^{2}\leq\frac{2\left|S^{*}\right|+\left|\supp\left(\xx^{*}\right)\right|}{\left|S'\right|-\left|\supp\left(\xx^{*}\right)\right|}\cdot\left\Vert \left(\xx+g\right)_{T}-\xx^{*}\right\Vert ^{2}-\left\Vert g_{S^{*}}\right\Vert ^{2}\ ,
	\]
	for some $T$, such that $\left|T\right|\leq2\left|S^{*}\right|+\left|\supp\left(\xx^{*}\right)\right|+\left|S'\right|$
	and $\supp\left(\xx^{*}\right)\subseteq T$.
\end{cor}

\begin{proof}
	Using Lemma \ref{lem:trunc-and-opt} we can write 
	\begin{align*}
	\left\Vert \left(\xx+g\right)_{S\setminus S'}\right\Vert ^{2}-\left\Vert g_{S\cup S'}\right\Vert ^{2} & \leq\left\Vert \left(\xx+g\right)_{Z\setminus S'}\right\Vert ^{2}-\left\Vert g_{S^{*}}\right\Vert ^{2}\leq\left\Vert \left(\xx+g\right)_{\left(Z\cup\supp\left(\xx^{*}\right)\right)\setminus S'}\right\Vert ^{2}-\left\Vert g_{S^{*}}\right\Vert ^{2}\\
	& =\left\Vert \left(\xx+g\right)_{Z\cup\supp\left(\xx^{*}\right)\cup S'}-\left(\xx+g\right)_{S'}\right\Vert ^{2}-\left\Vert g_{S^{*}}\right\Vert ^{2}\ .
	\end{align*}
	where $\left|Z\cup\supp\left(\xx^{*}\right)\cup S'\right|\leq2\left|S^{*}\right|+\left|\supp\left(\xx^{*}\right)\right|+\left|S'\right|$.
	Applying Lemma \ref{lem:sparse-proj} we furthermore obtain that 
	\begin{align*}
	\left\Vert \left(\xx+g\right)_{Z\cup S^{*}\cup S'}-\left(\xx+g\right)_{S'}\right\Vert ^{2} & \leq\frac{\left|Z\cup\supp\left(\xx^{*}\right)\cup S'\right|-\left|S'\right|}{\left|Z\cup\supp\left(\xx^{*}\right)\cup S'\right|-\left|\supp\left(\xx^{*}\right)\right|}\cdot\left\Vert \left(\xx+g\right)_{Z\cup\supp\left(\xx^{*}\right)\cup S'}-\xx^{*}\right\Vert ^{2}\\
	& \leq\frac{2\left|S^{*}\right|+\left|\supp\left(\xx^{*}\right)\right|}{\left|S'\right|-\left|\supp\left(\xx^{*}\right)\right|}\cdot\left\Vert \left(\xx+g\right)_{Z\cup\supp\left(\xx^{*}\right)\cup S'}-\xx^{*}\right\Vert ^{2}\ .
	\end{align*}
\end{proof}
We can now proceed with the main proof.

\begin{proof}[Proof of Theorem~\ref{thm:iht-pl}]
	For simplicity we first provide the proof for the deterministic version, which roughly follows the ideas described in~\citep{jain2014iterative}. Afterwards, we extend it to the stochastic setting. To simplify notation, throughout this proof we will use $\eta=\frac{1}{\beta}$.
	
	Using Lemma \ref{lem:smooth-ub} we can write that for the update
	\[
	\delta=T_{\ss}\left(\xx-\eta\nabla f\left(\xx\right)\right)-\xx\,,
	\]
	we have
	\begin{align*}
	f\left(\xx'\right) & \leq f\left(\xx\right)+\frac{1}{2\eta}\left\Vert \left(\frac{1}{\beta}\nabla f\left(\xx\right)+T_{\ss}\left(\xx-\eta\nabla f\left(\xx\right)\right)-\xx\right)_{\text{supp}\left(\delta\right)}\right\Vert ^{2}-\frac{1}{2\eta}\left\Vert \eta\nabla f\left(\xx\right)_{\text{supp}\left(\delta\right)}\right\Vert ^{2}\\
	& =f\left(\xx\right)+\frac{1}{2\eta}\left\Vert \left(T_{\ss}\left(\xx-\eta\nabla f\left(\xx\right)\right)-\left(\xx-\eta\nabla f\left(\xx\right)\right)\right)_{\text{supp}\left(\delta\right)}\right\Vert ^{2}-\frac{1}{2\eta}\left\Vert \eta\nabla f\left(\xx\right)_{\text{supp}\left(\delta\right)}\right\Vert ^{2}\ .
	\end{align*}
	At this point we use the fact that by definition $\text{supp}\left(\delta\right)=S'\cup S$.
	Furthermore we see that 
	\begin{align*}
	\left\Vert \left(T_{\ss}\left(\xx-\eta\nabla f\left(\xx\right)\right)-\left(\xx-\eta\nabla f\left(\xx\right)\right)\right)_{S'\cup S}\right\Vert ^{2} & =\left\Vert \left(\xx-\eta\nabla f\left(\xx\right)\right)_{S\setminus S'}\right\Vert ^{2}
	\end{align*}
	since $T_{\ss}\left(\xx-\eta\nabla f\left(\xx\right)\right)$ exactly
	matches $\xx-\eta\nabla f\left(\xx\right)$ for the coordinates in
	$S'$, and is $0$ for all the others. Thus we have that
	\begin{align*}
	f\left(\xx'\right) & \leq f\left(\xx\right)+\frac{1}{2\eta}\left(\left\Vert \left(\xx-\eta\nabla f\left(\xx\right)\right)_{S\setminus S'}\right\Vert ^{2}-\left\Vert \eta\nabla f\left(\xx\right)_{S\cup S'}\right\Vert ^{2}\right)\ .
	\end{align*}
	In order to apply the $\cPLshort$ property, we need to relate this
	to the contribution to the gradient norm given by the heavy signal
	$\left\Vert T_{\ss^{*}}\left(\nabla f\left(\xx\right)\right)\right\Vert ^{2}$.
	Let $S^{*}$ be the support of $T_{\ss^{*}}\left(\nabla f\left(\xx\right)\right)$.
	We apply Corollary \ref{cor:prog} to further bound
	\begin{align*}
	f\left(\xx'\right) & \leq f\left(\xx\right)+\frac{1}{2\eta}\left(\frac{3\ss^{*}}{\ss-\ss^{*}}\left\Vert \left(\xx-\eta\nabla f\left(\xx\right)\right)_{T}-\xx^{*}\right\Vert ^{2}-\left\Vert \eta\nabla f\left(\xx\right)_{S^{*}}\right\Vert ^{2}\right)\ ,
	\end{align*}
	where $\left|T\right|\leq3\ss^{*}+\ss$. Now we apply the $\cPLshort$
	property as follows. From Lemma \ref{lem:conPLQG} we upper bound
	\begin{align*}
	\left\Vert \left(\xx-\eta\nabla f\left(\xx\right)\right)_{T}-\xx^{*}\right\Vert ^{2} & \leq\left\Vert \left(\xx-\eta\nabla f\left(\xx\right)\right)_{T\cup S}-\xx^{*}\right\Vert ^{2}\\
	& \leq\frac{8}{\alpha}\left(f\left(\xx-\eta\nabla f\left(\xx\right)_{T\cup S}\right)-f\left(\xx^{*}\right)\right)\\
	& \leq\frac{8}{\alpha}\left(f\left(\xx\right)-f\left(\xx^{*}\right)\right)\ ,
	\end{align*}
	In the first inequality we used the fact that $\supp\left(\xx^{*}\right)\subseteq T$.
	In the second one we applied Lemma \ref{lem:conPLQG}. In the third
	one we applied $(3\ss^{*}+2\ss,\beta)$-smoothness
	(since $\left|T\cup S\right|\leq3\ss^{*}+2\ss$) together with the
	fact that $\eta=1/\beta$, and so $f\left(\xx-\eta\nabla f\left(\xx\right)_{T\cup S}\right)\leq f\left(\xx\right)$.
	
	Similarly we apply the $\cPLshort$ inequality to conclude that
	\[
	f\left(\xx'\right)\leq f\left(\xx\right)+\frac{1}{2\eta}\left(\frac{3\ss^{*}}{\ss-\ss^{*}}\cdot\frac{8}{\alpha}\left(f\left(\xx\right)-f\left(\xx^{*}\right)\right)-\eta^{2}\cdot\frac{\alpha}{2}\left(f\left(\xx\right)-f\left(\xx^{*}\right)\right)\right)\ .
	\]
	Thus equivalently:
	\begin{align*}
	f\left(\xx'\right)-f\left(\xx^{*}\right) & \leq\left(f\left(\xx\right)-f\left(\xx^{*}\right)\right)\left(1+\frac{12\ss^{*}}{\ss-\ss^{*}}\cdot\frac{1}{\eta\alpha}-\frac{\eta\alpha}{4}\right)\\
	& =\left(f\left(\xx\right)-f\left(\xx^{*}\right)\right)\left(1+\frac{12\ss^{*}}{\ss-\ss^{*}}\cdot\kappa-\frac{1}{4\kappa}\right)\,,
	\end{align*}
	where $\kappa=\beta/\alpha=1/(\eta\alpha)$. Since $\ss\geq\ss^{*}\cdot\left(96\kappa^{2}+1\right)$
	we equivalently have that $\frac{\ss^{*}}{\ss-\ss^{*}}\leq\frac{1}{96k^{2}}$
	and thus $\frac{12\ss^{*}}{\ss-\ss^{*}}\cdot\kappa\leq\frac{1}{8\kappa}$.
	Hence 
	\[
	f\left(\xx'\right)-f\left(\xx^{*}\right)\leq\left(f\left(\xx\right)-f\left(\xx^{*}\right)\right)\cdot\left(1-\frac{1}{8\kappa}\right)\ .
	\]
	This shows that in $O\left(\kappa\ln\frac{f\left(\xx_{0}\right)-f\left(\xx^{*}\right)}{\epsilon}\right)$
	we reach a point $\xx$ such that $f\left(\xx\right)-f\left(\xx^{*}\right)\leq\epsilon$,
	which concludes the proof.
	
	We extend this proof to the stochastic version in Section \ref{sec:missing-proofs}.
\end{proof}

Now let us show how Corollary~\ref{cor:iht-acdc} follows from the same analysis.
\begin{proof}[Proof of Corollary~\ref{cor:iht-acdc}]
	The proof extends from that to Theorem~\ref{thm:iht-pl}. The difference we need to handle is the error introduced by performing $\Delta_c$ passes through the data instead of a single stochastic gradient step. Suppose that before starting the dense training phase, the current iterate is $\theta$. The key is to bound the error introduced by performing these $\Delta_c$ passes instead of simply changing the iterate by $-\eta \nabla f(\theta)$, prior to applying the pruning step. To do so we first upper bound the change in the iterate after $\Delta _c$ passes, which lead to a new iterate $\widetilde{\theta}$.
	Since we perform $\Delta_c \cdot B$ iterations instead of a single one, we damp down our step size by setting $\eta' = \eta / (\Delta_c B)$, and measure the movement we made compared to the one we would have made with a single deterministic gradient step.
	
	Using the Lipschitz property of the functions in the decomposition, we see that each step changes our iterate by at most $\eta' L$ in $\ell_2$ norm. Hence over $\Delta_c$ passes through the data, each of which involves $B$ mini-batches, the total change in the iterate is at most:
	\[
	\|\widetilde{\theta} - \theta\| \leq \Delta_c \cdot B  \cdot \eta' L = \eta L \,.
	\]
	Using the smoothness property of each $f_i$ this guarantees that for all the seen iterates $\widehat{\theta}$, the gradients of $f_i$'s never deviate significantly from their values at the original point $\theta$:
	\[
	\| \nabla f_i(\widehat{\theta}) - \nabla f_i(\theta) \| \leq \beta \eta  L\,,
	\]
	Thus if we interpret the scaled total movement in iterate $\frac{1}{\eta}(\theta - \widetilde{\theta})$ as a gradient mapping, it satisfies 
	\[
	\left\| \frac{1}{\eta}(\theta - \widetilde{\theta}) - \nabla f(\theta)\ \right\|  \leq \beta \eta L\,.
	\]
	Applying the analysis for the stochastic version of Theorem~\ref{thm:iht-pl} we can treat the error in the gradient mapping exactly as the stochastic noise. For the specific choice of step size used there $\eta = 1/\Theta(\beta)$, we thus get that
	\[
	\left\| \frac{1}{\eta}(\theta - \widetilde{\theta}) - \nabla f(\theta)\ \right\|^2 = O(L^2)\,,
	\]
	which enables us to conclude the analysis.
	Note that the steps performed during the sparse training phases do not affect convergence as they can only improve error in function value, which is the main quantity that our analysis tracks.

\end{proof}

\subsection{Stochastic IHT for Functions with Restricted Smoothness and Strong
	Convexity}\label{sec:iht-standard-proof}

Here we prove Theorem~\ref{thm:vanilla-iht-stoch}. Before proceeding with the proof, we provide a few useful statements related to the
fact that $f$ is well conditioned along sparse directions.
\begin{lem}
	If $f$ is $(2\ss+\ss^{*},\beta)$-smooth and $(\ss+\ss^*, \alpha)$-strongly convex
	then
	\begin{align*}
	f(\xx^{*}) & \geq f(\xx)+\left\langle \nabla f(\xx)_{S\cup S'\cup S^{*}},\xx^{*}-\xx\right\rangle +\frac{\alpha}{2}\left\Vert \xx-\xx^{*}\right\Vert ^{2}\ ,\\
	f\left(\xx-\frac{1}{\beta}\nabla f(\xx)_{S\cup S'\cup S^{*}}\right) & \leq f(\xx)-\frac{1}{2\beta}\left\Vert \nabla f(\xx)_{S\cup S'\cup S^{*}}\right\Vert ^{2}\ .
	\end{align*}
\end{lem}

\begin{proof}
	The former follows directly from the definition. For the latter we
	have
	\begin{align*}
	f\left(\xx-\frac{1}{\beta}\nabla f(\xx)_{S\cup S'\cup S^{*}}\right) & \leq f(\xx)-\frac{1}{\beta}\left\langle \nabla f(\xx),\nabla f(\xx)_{S\cup S'\cup S^{*}}\right\rangle +\frac{\beta}{2}\left\Vert \nabla f(\xx)_{S\cup S'\cup S^{*}}\right\Vert ^{2}\\
	& =f(\xx)-\frac{1}{2\beta}\left\Vert \nabla f(\xx)_{S\cup S'\cup S^{*}}\right\Vert ^{2}\ .
	\end{align*}
\end{proof}
Finally we can prove the main statement:
\begin{proof}
	We will track progress by measuring the distance$\left\Vert \xx-\xx^{*}\right\Vert ^{2}$.
	To do so we write
	\begin{align*}
	\left\Vert \xx'-\xx^{*}\right\Vert ^{2} & =\left\Vert T_{\ss}\left(\xx-\eta g_{\xx}\right)-\xx^{*}\right\Vert ^{2}\\
	& =\left\Vert \left(\xx-\eta g_{\xx}-\xx^{*}\right)+\left(T_{\ss}\left(\xx-\eta g_{\xx}\right)-\left(\xx-\eta g_{\xx}\right)\right)\right\Vert ^{2}\\
	& =\left\Vert \left(\xx-\eta g_{\xx}-\xx^{*}\right)_{S'\cup S^{*}}+\left(T_{\ss}\left(\left(\xx-\eta g_{\xx}\right)_{S'\cup S^{*}}\right)-\left(\xx-\eta g_{\xx}\right)_{S'\cup S^{*}}\right)\right\Vert ^{2}\ .
	\end{align*}
	The last identity follows from the fact that the term inside the norm
	is only supported at $S'\cup S^{*}$. Thus, after applying the triangle
	inequality, we obtain that:
	\[
	\left\Vert \xx'-\xx^{*}\right\Vert ^{2}\leq\left(\left\Vert \left(\xx-\eta g_{\xx}\right)_{S'\cup S^{*}}-\xx^{*}\right\Vert +\left\Vert T_{\ss}\left(\left(\xx-\eta g_{\xx}\right)_{S'\cup S^{*}}\right)-\left(\xx-\eta g_{\xx}\right)_{S'\cup S^{*}}\right\Vert \right)^{2}\ .
	\]
	The second term can now be bounded using Lemma \ref{lem:sparse-proj}.
	Since the sparsity of the projected point is $\ss+\ss^{*}$ we have
	that 
	\begin{align*}
	\left\Vert T_{\ss}\left(\left(\xx-\eta g_{\xx}\right)_{S'\cup S^{*}}\right)-\left(\xx-\eta g_{\xx}\right)_{S'\cup S^{*}}\right\Vert ^{2} & \leq\frac{\ss+\ss^{*}-\ss}{\ss+\ss^{*}-\ss^{*}}\left\Vert \left(\xx-\eta g_{\xx}\right)_{S'\cup S^{*}}-\xx^{*}\right\Vert ^{2}\\
	& =\frac{\ss^{*}}{\ss}\cdot\left\Vert \left(\xx-\eta g_{\xx}\right)_{S'\cup S^{*}}-\xx^{*}\right\Vert ^{2}\,.
	\end{align*}
	Therefore:
	\begin{align*}
	\left\Vert \xx'-\xx^{*}\right\Vert ^{2} & \leq\left(1+\sqrt{\frac{s^{*}}{s}}\right)^{2}\cdot\left\Vert \left(\xx-\eta g_{\xx}\right)_{S'\cup S^{*}}-\xx^{*}\right\Vert ^{2}\leq\left(1+\sqrt{\frac{s^{*}}{s}}\right)^{2}\cdot\left\Vert \left(\xx-\eta g_{\xx}\right)_{S\cup S'\cup S^{*}}-\xx^{*}\right\Vert ^{2}\ .
	\end{align*}
	The second inequality follows from the fact that as we increase the
	support of $\xx-\eta g_{\xx}$ to also include $S\setminus\left(S'\cup S^{*}\right)$,
	the norm inside can only increase. Finally by expanding, we write:
	\begin{align*}
	\left\Vert \left(\xx-\eta g_{\xx}\right)_{S\cup S'\cup S^{*}}-\xx^{*}\right\Vert ^{2} & =\left\Vert \xx-\eta\left(g_{\xx}\right)_{S\cup S'\cup S^{*}}-\xx^{*}\right\Vert ^{2}\\
	& =\left\Vert \xx-\xx^{*}\right\Vert ^{2}-2\eta\left\langle \left(g_{\xx}\right)_{S\cup S'\cup S^{*}},\xx-\xx^{*}\right\rangle +\eta^{2}\left\Vert \left(g_{\xx}\right)_{S\cup S'\cup S^{*}}\right\Vert ^{2}\\
	& \leq\left\Vert \xx-\xx^{*}\right\Vert ^{2}-2\eta\left\langle \nabla f\left(\xx\right)_{S\cup S'\cup S^{*}},\xx-\xx^{*}\right\rangle +2\eta^{2}\left\Vert \nabla f\left(\xx\right)_{S\cup S'\cup S^{*}}\right\Vert ^{2}\\
	& +\underbrace{2\eta\left\langle \left(\nabla f\left(\xx\right)-g_{\xx}\right)_{S\cup S'\cup S^{*}},\xx-\xx^{*}\right\rangle +2\eta^{2}\left\Vert \left(\nabla f\left(\xx\right)-g_{\xx}\right)_{S\cup S'\cup S^{*}}\right\Vert ^{2}}_{\zeta}\ ,
	\end{align*}
	where we use $\zeta$ to denote the error term introduced by using
	stochastic gradients. Next we bound the fist part of the term above.
	To do so we use lemma to write
	\begin{align*}
	& \left\Vert \xx-\xx^{*}\right\Vert ^{2}-2\eta\left\langle \nabla f\left(\xx\right)_{S\cup S'\cup S^{*}},\xx-\xx^{*}\right\rangle +2\eta^{2}\left\Vert \nabla f\left(\xx\right)_{S\cup S'\cup S^{*}}\right\Vert ^{2}\\
	\leq & \left\Vert \xx-\xx^{*}\right\Vert ^{2}-2\eta\left(f\left(\xx\right)-f\left(\xx^{*}\right)-\frac{\alpha}{2}\left\Vert \xx-\xx^{*}\right\Vert ^{2}\right)+2\eta^{2}\left\Vert \nabla f\left(\xx\right)_{S\cup S'\cup S^{*}}\right\Vert ^{2}\\
	= & \left\Vert \xx-\xx^{*}\right\Vert ^{2}\left(1-\eta\alpha\right)-2\eta\left(f\left(\xx\right)-f\left(\xx^{*}\right)-\eta\left\Vert \nabla f\left(\xx\right)_{S\cup S'\cup S^{*}}\right\Vert ^{2}\right)\\
	= & \left\Vert \xx-\xx^{*}\right\Vert ^{2}\left(1-\eta\alpha\right)-2\eta\left(f\left(\xx\right)-f\left(\xx^{*}\right)-\frac{1}{2\beta}\left\Vert \nabla f\left(\xx\right)_{S\cup S'\cup S^{*}}\right\Vert ^{2}\right)-2\eta\cdot\left(\frac{1}{2\beta}-\eta\right)\left\Vert \nabla f\left(\xx\right)_{S\cup S'\cup S^{*}}\right\Vert ^{2}\\
	\leq & \left\Vert \xx-\xx^{*}\right\Vert ^{2}\left(1-\eta\alpha\right)-2\eta\cdot\left(\frac{1}{2\beta}-\eta\right)\left\Vert \nabla f\left(\xx\right)_{S\cup S'\cup S^{*}}\right\Vert ^{2}\ .
	\end{align*}
	For the first inequality we used the restricted strong convexity property,
	while for the second we used the restricted smoothness property. Now
	for the error term we have that:
	\[
	\mathbb{E}\left[\zeta\vert\xx\right]\leq2\eta^{2}\sigma^{2}\ ,
	\]
	which enables us to conclude that 
	\[
	\mathbb{E}\left[\left\Vert \xx'-\xx^{*}\right\Vert ^{2}\bigg\vert\xx\right]\leq\left(1+\sqrt{\frac{\ss^{*}}{\ss}}\right)^{2}\cdot\left(\left\Vert \xx-\xx^{*}\right\Vert ^{2}\left(1-\eta\alpha\right)-2\eta\cdot\left(\frac{1}{2\beta}-\eta\right)\left\Vert \nabla f\left(\xx\right)_{S\cup S'\cup S^{*}}\right\Vert ^{2}+2\eta^{2}\sigma^{2}\right)\ .
	\]
	Setting $\eta=\frac{1}{2\beta}$ this gives us that
	\[
	\mathbb{E}\left[\left\Vert \xx'-\xx^{*}\right\Vert ^{2}\bigg\vert\xx\right]\leq\left(1+\sqrt{\frac{s^{*}}{s}}\right)^{2}\cdot\left(\left\Vert \xx-\xx^{*}\right\Vert ^{2}\left(1-\frac{1}{2}\cdot\frac{\alpha}{\beta}\right)+\frac{1}{2}\left(\frac{\sigma}{\beta}\right)^{2}\right)\ .
	\]
	Thus for as long as $\left\Vert \xx-\xx^{*}\right\Vert ^{2}\geq2\frac{\sigma^{2}}{\alpha\beta}$
	we have that
	\[
	\mathbb{E}\left[\left\Vert \xx'-\xx^{*}\right\Vert ^{2}\bigg\vert\xx\right]\leq\left(1+\sqrt{\frac{\ss^{*}}{\ss}}\right)^{2}\cdot\left(1-\frac{1}{4}\cdot\frac{\alpha}{\beta}\right)\cdot\left\Vert \xx-\xx^{*}\right\Vert ^{2}\ .
	\]
	We see that the expected squared distance contracts by setting the
	ratio $\ss/\ss^{*}$ sufficiently large. Indeed if $\ss\geq81\left(\beta/\alpha\right)^{2}\cdot\ss^{*}$
	we have that: 
	
	\begin{align*}
	\mathbb{E}\left[\left\Vert \xx'-\xx^{*}\right\Vert ^{2}\bigg\vert\xx\right] & \leq\left(1+\frac{1}{9}\cdot\frac{\alpha}{\beta}\right)^{2}\cdot\left(1-\frac{1}{4}\cdot\frac{\alpha}{\beta}\right)\cdot\left\Vert \xx-\xx^{*}\right\Vert ^{2}\leq\left(1-\frac{1}{36}\cdot\frac{\alpha}{\beta}\right)\cdot\left\Vert \xx-\xx^{*}\right\Vert ^{2}\ .
	\end{align*}
	Taking expectation over the entire history of iterates we can thus
	conclude that after $T=O\left(\frac{\beta}{\alpha}\cdot\ln\frac{\left\Vert \xx_{0}-\xx^{*}\right\Vert ^{2}}{\epsilon/\beta+\sigma^{2}/\alpha\beta}\right)$
	iterations we obtain
	\[
	\mathbb{E}\left[\left\Vert \xx_{T}-\xx^{*}\right\Vert ^{2}\right]\leq\frac{\epsilon}{\beta}+\frac{2\sigma^{2}}{\alpha\beta}\ .
	\]
	Applying the restricted smoothness property, this also gives us that:
	\[
	\mathbb{E}\left[f\left(\xx_{T}\right)-f\left(\xx^{*}\right)\right]\leq\epsilon+\frac{2\sigma^{2}}{\alpha}\ .
	\]
	which is what we wanted.
\end{proof}

\subsection{Finding a Sparse Nearly-Stationary Point}
\label{subsec:sparse-stationary}

Here we prove Theorem~\ref{thm:iht-nonconv}. For simplicity, we first prove the deterministic version of the theorem, where $\sigma = 0$. We show how to extend this proof to the stochastic version in Section~\ref{sec:missing-proofs}.
\begin{proof}
	The proof is similar to that of Theorem \ref{thm:iht-pl}, but in
	addition requires that the algorithm alternates standard IHT steps with optimizing inside
	the support of the iterate. More precisely given a current iterate
	supported at $S$, we additionally run an inner loop which optimizes
	only over the coordinate in $S$, seeking a near stationary point
	$\xx$ such that $\left\Vert \nabla f\left(\xx\right)_{S}\right\Vert \leq\epsilon$.
	Thus in our analysis we can assume that before performing a gradient,
	followed by a pruning step, our current iterate supported at $S$
	satisfies $\left\Vert \nabla f\left(\xx\right)_{S}\right\Vert \leq\epsilon$.
	Hence following the previous analyses and setting $\eta=1/\beta$,
	we have:
	\begin{align*}
	f\left(\xx'\right) & \leq f\left(\xx\right)+\frac{1}{2\eta}\left(\left\Vert \left(\xx-\eta\nabla f\left(\xx\right)\right)_{S\setminus S'}\right\Vert ^{2}-\left\Vert \eta\nabla f\left(\xx\right)_{S\cup S'}\right\Vert ^{2}\right)\\
	& \leq f\left(\xx\right)+\frac{1}{2\eta}\left(2\left\Vert \xx_{S\setminus S'}\right\Vert ^{2}+2\eta^{2}\epsilon^{2}-\left\Vert \eta\nabla f\left(\xx\right)\right\Vert _{\infty}^{2}\right)\ .
	\end{align*}
	For the second inequality we first applied triangle inequality together
	with the near stationarity condition for $\xx$ to bound 
	\begin{align*}
	\left\Vert \left(\xx-\eta\nabla f\left(\xx\right)\right)_{S\setminus S'}\right\Vert  & \leq\left\Vert \xx_{S\setminus S'}\right\Vert +\left\Vert \eta\nabla f\left(\xx\right)_{S\setminus S'}\right\Vert \leq\left\Vert \xx_{S\setminus S'}\right\Vert +\eta\epsilon\ ,
	\end{align*}
	then applied $\left(a+b\right)^{2}\leq2a^{2}+2b^{2}$. In addition
	we used the fact that 
	\[
	\left\Vert \nabla f\left(\xx\right)_{S\cup S'}\right\Vert _{\infty}=\left\Vert \nabla f\left(\xx\right)\right\Vert _{\infty}\ .
	\]
	This follows from a simple case analysis. If one of the coordinates
	of the gradient with the largest absolute value lies in $S\cup S'$,
	we are done. Otherwise, we have two possibilities. Either $S'$ is
	different from $S$, so $S\cup S'$ contains one coordinate outside
	of $S$. Since these coordinates are obtained by hard thresholding
	and $\xx$ is supported only at $S$, the absolute value of $\nabla f\left(\xx\right)$
	at the largest coordinate in $S'\setminus S$ must be at least as
	large as the largest in $S^{*}\setminus S$, which yields our claim.
	Otherwise we have that $S=S'$, which means that the pruning step
	did not change the support, and thus
	\[
	\left\Vert \eta\nabla f\left(\xx\right)_{\overline{S}}\right\Vert _{\infty}\leq\min_{i\in S}\left|\xx-\eta\nabla f\left(\xx\right)\right|_{i}\leq R_{\infty}+\eta\epsilon
	\]
	which guarantees that 
	\begin{align*}
	\left\Vert \nabla f\left(\xx\right)\right\Vert _{\infty} & =\max\left\{ \left\Vert \nabla f\left(\xx\right)_{S}\right\Vert _{\infty},\left\Vert \nabla f\left(\xx\right)_{\overline{S}}\right\Vert _{\infty}\right\} \\
	& \leq\max\left\{ \epsilon,\epsilon+\frac{R_{\infty}}{\eta}\right\} =\epsilon+\beta R_{\infty}\ ,
	\end{align*}
	and so we are done.
	
	In the former case we thus see that 
	\[
	\frac{\eta}{2}\left\Vert \nabla f\left(\xx\right)\right\Vert _{\infty}^{2}\leq f\left(\xx\right)-f\left(\xx'\right)+\left(\frac{\left\Vert \xx_{S\setminus S'}\right\Vert ^{2}}{\eta}+\eta\epsilon^{2}\right)\leq f\left(\xx\right)-f\left(\xx'\right)+\left(\frac{\ss R_{\infty}^{2}}{\eta}+\eta\epsilon^{2}\right)\ .
	\]
	Telescoping over $T$ iterations we see that 
	\[
	\frac{\eta}{2}\sum_{t=0}^{T-1}\left\Vert \nabla f\left(\xx_{t}\right)\right\Vert _{\infty}^{2}\leq f\left(\xx_{0}\right)-f\left(\xx_{T}\right)+T\cdot\left(\frac{\ss R_{\infty}^{2}}{\eta}+\eta\epsilon^{2}\right)
	\]
	and so returning a random point $\xx$ among those witnessed during
	the algorithm we have 
	\[
	\mathbb{E}\left[\left\Vert \nabla f\left(\xx\right)\right\Vert _{\infty}^{2}\right]\leq\frac{2\left(f\left(\xx_{0}\right)-f\left(\xx_{T}\right)\right)}{\eta T}+2\left(\frac{\ss R_{\infty}^{2}}{\eta^{2}}+\epsilon^{2}\right)
	\]
	By AM-QM,
	\[
	\mathbb{E}\left[\left\Vert \nabla f\left(\xx\right)\right\Vert _{\infty}\right]\leq\sqrt{\mathbb{E}\left[\left\Vert \nabla f\left(\xx\right)\right\Vert _{\infty}^{2}\right]}\,,
	\]
	which enables us to conclude that after sufficiently many iterations
	we are guaranteed to find a point such that 
	\[
	\left\Vert \nabla f\left(\xx\right)\right\Vert _{\infty}=O\left(\frac{R_{\infty}\sqrt{k}}{\eta}+\epsilon\right)=O\left(\beta R_{\infty}\sqrt{k}+\epsilon\right)\ .
	\]
\end{proof}

\subsection{Deferred Proofs}

\label{sec:missing-proofs}
\paragraph{Proof of Lemma~\ref{lem:sparse-proj}.}

Our proofs crucially rely on the following lemma. Intuitively it
shows that projecting a vector $\xx$ onto the non-convex set of sparse
vectors does not increase the distance to the optimum by too much.
While this is indeed always true for projections onto convex sets,
in this case we can provably show that the possible increase in distance
is small.

\begin{proof}
	We have that the function 
	\[
	h(k)=\frac{\left\Vert T_{\ss}\left(\xx\right)-\xx\right\Vert ^{2}}{n-\ss}
	\]
	is non-increasing. Indeed, using more nonzeros can only decrease
	the ratio. Thus 
	\[
	\frac{\left\Vert T_{\ss}\left(\xx\right)-\xx\right\Vert ^{2}}{n-\ss}\leq\frac{\left\Vert T_{\ss^{*}}\left(\xx\right)-\xx\right\Vert ^{2}}{n-\ss^{*}}\leq\frac{\left\Vert \xx^{*}-\xx\right\Vert ^{2}}{n-\ss^{*}}\ ,
	\]
	where the last inequality follows from the fact that among all $\ss$-sparse
	vectors, $T_{\ss}\left(\xx\right)$ minimizes the distance to $\xx$.
\end{proof}

\paragraph{Proof of Theorem~\ref{thm:iht-pl} (stochastic version).}
Next we extend the proof of Theorem \ref{thm:iht-pl} to the case when only stochastic gradients are available.
\begin{proof}
	Similarly to before we can write
	\begin{align*}
	f\left(\xx'\right) & \leq f\left(\xx\right)+\frac{\beta}{2}\left\Vert \frac{1}{\beta}\nabla f\left(\xx\right)+\left(T_{\ss}\left(\xx-\eta g_{\xx}\right)-\xx\right)\right\Vert ^{2}-\frac{1}{2\beta}\left\Vert \nabla f\left(\xx\right)\right\Vert ^{2}\\
	& =f\left(\xx\right)+\frac{\beta}{2}\left\Vert \left(T_{\ss}\left(\xx-\eta g_{\xx}\right)-\left(\xx-\frac{1}{\beta}\nabla f\left(\xx\right)\right)\right)_{S\cup S'}\right\Vert ^{2}-\frac{1}{2\beta}\left\Vert \nabla f\left(\xx\right)_{S\cup S'}\right\Vert ^{2}\\
	& \leq f\left(\xx\right)+\frac{\beta}{2}\left(\left\Vert \left(T_{\ss}\left(\xx-\eta g_{\xx}\right)-\left(\xx-\eta g_{\xx}\right)\right)_{S\cup S'}\right\Vert +\left\Vert \left(\eta g_{\xx}-\frac{1}{\beta}\nabla f\left(\xx\right)\right)_{S\cup S'}\right\Vert \right)^{2}-\frac{1}{2\beta}\left\Vert \nabla f\left(\xx\right)_{S\cup S'}\right\Vert ^{2}\\
	& \leq f\left(\xx\right)+\beta\left\Vert \left(T_{\ss}\left(\xx-\eta g_{\xx}\right)-\left(\xx-\eta g_{\xx}\right)\right)_{S\cup S'}\right\Vert ^{2}+\beta\left\Vert \left(\eta g_{\xx}-\frac{1}{\beta}\nabla f\left(\xx\right)\right)_{S\cup S'}\right\Vert ^{2}-\frac{1}{2\beta}\left\Vert \nabla f\left(\xx\right)_{S\cup S'}\right\Vert ^{2}\\
	& =f\left(\xx\right)+\beta\left\Vert \left(\xx-\eta g_{\xx}\right)_{S\setminus S'}\right\Vert ^{2}+\beta\left\Vert \left(\eta g_{\xx}-\frac{1}{\beta}\nabla f\left(\xx\right)\right)_{S\cup S'}\right\Vert ^{2}-\frac{1}{2\beta}\left\Vert \nabla f\left(\xx\right)_{S\cup S'}\right\Vert ^{2}\ .
	\end{align*}
	Next we apply Corollary \ref{cor:prog} to further bound 
	\begin{align*}
	\left\Vert \left(\xx-\eta g_{\xx}\right)_{S\setminus S'}\right\Vert ^{2} & \leq\frac{2\left|S^{*}\right|+\left|\supp\left(\xx^{*}\right)\right|}{\left|S'\right|-\left|\supp\left(\xx^{*}\right)\right|}\cdot\left\Vert \left(\xx-\eta g_{\xx}\right)_{T}-\xx^{*}\right\Vert ^{2}-\left\Vert \eta\left(g_{\xx}\right)_{S^{*}}\right\Vert ^{2}+\left\Vert \eta\left(g_{\xx}\right)_{S\cup S'}\right\Vert ^{2}\\
	& =\frac{3\ss^{*}}{\ss-\ss^{*}}\cdot\left\Vert \left(\xx-\eta g_{\xx}\right)_{T}-\xx^{*}\right\Vert ^{2}-\left\Vert \eta\left(g_{\xx}\right)_{S^{*}}\right\Vert ^{2}+\left\Vert \eta\left(g_{\xx}\right)_{S\cup S'}\right\Vert ^{2}\ ,
	\end{align*}
	where $\left|T\right|\leq3\ss^{*}+\ss$. Thus
	\begin{align*}
	f\left(\xx'\right) & \leq f\left(\xx\right)+\beta\left(\frac{3\ss^{*}}{\ss-\ss^{*}}\cdot\left\Vert \left(\xx-\eta g_{\xx}\right)_{T}-\xx^{*}\right\Vert ^{2}-\left\Vert \eta\left(g_{\xx}\right)_{S^{*}}\right\Vert ^{2}+\left\Vert \eta\left(g_{\xx}\right)_{S\cup S'}\right\Vert ^{2}\right)\\
	& +\beta\left\Vert \left(\eta g_{\xx}-\frac{1}{\beta}\nabla f\left(\xx\right)\right)_{S\cup S'}\right\Vert ^{2}-\frac{1}{2\beta}\left\Vert \nabla f\left(\xx\right)_{S\cup S'}\right\Vert ^{2}\ .
	\end{align*}
	Taking expectations, and applying Lemma \ref{lem:stdevexp} we see
	that 
	\begin{align*}
	\mathbb{E}\left[f\left(\xx'\right)-f^{*}\vert\xx\right] & \leq\mathbb{E}\left[f\left(\xx\right)-f^{*}\vert\xx\right]\\
	& +\beta\left(\frac{3\ss^{*}}{\ss-\ss^{*}}\cdot\left\Vert \left(\xx-\eta\nabla f\left(\xx\right)\right)_{T}-\xx^{*}\right\Vert ^{2}-\left\Vert \eta\nabla f\left(\xx\right)_{S^{*}}\right\Vert ^{2}+\left\Vert \eta\nabla f\left(\xx\right)_{S\cup S'}\right\Vert ^{2}\right)\\
	& +\beta\left\Vert \left(\eta\nabla f\left(\xx\right)-\frac{1}{\beta}\nabla f\left(\xx\right)\right)_{S\cup S'}\right\Vert ^{2}-\frac{1}{2\beta}\left\Vert \nabla f\left(\xx\right)_{S\cup S'}\right\Vert ^{2}+2\beta\eta^{2}\sigma^{2}\\
	& =\mathbb{E}\left[f\left(\xx\right)-f^{*}\vert\xx\right]+\beta\left(\frac{3\ss^{*}}{\ss-\ss^{*}}\cdot\left\Vert \left(\xx-\eta\nabla f\left(\xx\right)\right)_{T}-\xx^{*}\right\Vert ^{2}-\left\Vert \eta\nabla f\left(\xx\right)_{S^{*}}\right\Vert ^{2}\right)\\
	& +\left\Vert \nabla f\left(\xx\right)_{S\cup S'}\right\Vert ^{2}\left(\beta\eta^{2}+\beta\left(\eta-\frac{1}{\beta}\right)^{2}-\frac{1}{2\beta}\right)+2\beta\eta^{2}\sigma^{2}\\
	& \leq\mathbb{E}\left[f\left(\xx\right)-f^{*}\vert\xx\right]+\beta\left(\frac{3\ss^{*}}{\ss-\ss^{*}}\cdot\left\Vert \left(\xx-\eta\nabla f\left(\xx\right)\right)_{T}-\xx^{*}\right\Vert ^{2}-\left\Vert \eta\nabla f\left(\xx\right)_{S^{*}}\right\Vert ^{2}\right)+2\beta\eta^{2}\sigma^{2}\,,
	\end{align*}
	where the last inequality follows from setting $\eta=\frac{1}{2\beta}$,
	which makes $\beta\eta^{2}+\beta\left(\eta-\frac{1}{\beta}\right)^{2}-\frac{1}{2\beta}=0$.
	
	Finally, repeating the argument used for the deterministic proof in
	Section \ref{sec:stoch-iht-pl}, we further bound
	\begin{align*}
	\mathbb{E}\left[f\left(\xx'\right)-f^{*}\vert\xx\right] & \leq f\left(\xx\right)-f^{*}+\beta\left(\frac{3\ss^{*}}{\ss-\ss^{*}}\cdot\frac{8}{\alpha}\left(f\left(\xx\right)-f\left(\xx^{*}\right)\right)-\eta^{2}\cdot\frac{\alpha}{2}\left(f\left(\xx\right)-f\left(\xx^{*}\right)\right)\right)+2\beta\eta^{2}\sigma^{2}\\
	& =\left(f\left(\xx\right)-f^{*}\right)\left(1+\frac{24\ss^{*}}{\ss-\ss^{*}}\cdot\frac{\beta}{\alpha}-\frac{\beta\eta^{2}\alpha}{2}\right)+2\beta\eta^{2}\sigma^{2}\\
	& =\left(f\left(\xx\right)-f^{*}\right)\left(1+\frac{24\ss^{*}}{\ss-\ss^{*}}\cdot\frac{\beta}{\alpha}-\frac{\alpha}{8\beta}\right)+\frac{\sigma^{2}}{2\beta}\ .
	\end{align*}
	Setting $\ss=\ss^{*}\cdot\left(384\left(\frac{\beta}{\alpha}\right)^{2}+1\right)$
	we have $\frac{24\ss^{*}}{\ss-\ss^{*}}\cdot\frac{\beta}{\alpha}\leq24\cdot\frac{1}{384\cdot\left(\beta/\alpha\right)^{2}}\cdot\frac{\beta}{\alpha}=\frac{1}{16}\cdot\frac{\alpha}{\beta}$,
	so
	\[
	\mathbb{E}\left[f\left(\xx'\right)-f^{*}\vert\xx\right]\leq\left(f\left(\xx\right)-f^{*}\right)\left(1-\frac{\alpha}{16\beta}\right)+\frac{\sigma^{2}}{2\beta}\ .
	\]
	Thus for as long as 
	\[
	\frac{\sigma^{2}}{2\beta}\leq\left(f\left(\xx\right)-f^{*}\right)\cdot\frac{\alpha}{32\beta}\iff f\left(\xx\right)-f^{*}\geq16\cdot\frac{\sigma^{2}}{\alpha}
	\]
	one has that 
	\[
	\mathbb{E}\left[f\left(\xx'\right)-f^{*}\vert\xx\right]\leq\left(f\left(\xx\right)-f^{*}\right)\left(1-\frac{\alpha}{32\beta}\right)\ .
	\]
	Taking expectation over the entire history, this shows that after
	$T=O\left(\left(\frac{\beta}{\alpha}\right)\ln\frac{f\left(\xx_{0}\right)-f^{*}}{\epsilon}\right)$
	iterations we obtain an iterate $\xx_{T}$ such that
	\[
	\mathbb{E}\left[f\left(\xx_{T}\right)-f^{*}\right]\leq\epsilon+\frac{16\sigma^{2}}{\alpha}\ ,
	\]
	which concludes the proof.
\end{proof}

\paragraph{Proof of Theorem~\ref{thm:iht-nonconv} (stochastic version).}
We also provide the proof for the stochastic version of Theorem~\ref{thm:iht-nonconv}.

\begin{proof}
	We follow the steps of the proof provided in Section~\ref{subsec:sparse-stationary}.
	More precisely we write:
	\begin{align*}
	f\left(\xx'\right) & \leq f\left(\xx\right)+\frac{\beta}{2}\left\Vert \frac{1}{\beta}\nabla f\left(\xx\right)+\left(T_{\ss}\left(\xx-\eta g_{\xx}\right)-\xx\right)\right\Vert ^{2}-\frac{1}{2\beta}\left\Vert \nabla f\left(\xx\right)\right\Vert ^{2}\\
	& =f\left(\xx\right)+\frac{\beta}{2}\left\Vert \left(T_{\ss}\left(\xx-\eta g_{\xx}\right)-\left(\xx-\frac{1}{\beta}\nabla f\left(\xx\right)\right)\right)_{S\cup S'}\right\Vert ^{2}-\frac{1}{2\beta}\left\Vert \nabla f\left(\xx\right)_{S\cup S'}\right\Vert ^{2}\\
	& \leq f\left(\xx\right)+\frac{\beta}{2}\left(\left\Vert \left(T_{\ss}\left(\xx-\eta g_{\xx}\right)-\left(\xx-\eta g_{\xx}\right)\right)_{S\cup S'}\right\Vert +\left\Vert \left(\eta g_{\xx}-\frac{1}{\beta}\nabla f\left(\xx\right)\right)_{S\cup S'}\right\Vert \right)^{2}-\frac{1}{2\beta}\left\Vert \nabla f\left(\xx\right)_{S\cup S'}\right\Vert ^{2}\\
	& \leq f\left(\xx\right)+\beta\left\Vert \left(T_{\ss}\left(\xx-\eta g_{\xx}\right)-\left(\xx-\eta g_{\xx}\right)\right)_{S\cup S'}\right\Vert ^{2}+\beta\left\Vert \left(\eta g_{\xx}-\frac{1}{\beta}\nabla f\left(\xx\right)\right)_{S\cup S'}\right\Vert ^{2}-\frac{1}{2\beta}\left\Vert \nabla f\left(\xx\right)_{S\cup S'}\right\Vert ^{2}\\
	& =f\left(\xx\right)+\beta\left\Vert \left(\xx-\eta g_{\xx}\right)_{S\setminus S'}\right\Vert ^{2}+\beta\left\Vert \left(\eta g_{\xx}-\frac{1}{\beta}\nabla f\left(\xx\right)\right)_{S\cup S'}\right\Vert ^{2}-\frac{1}{2\beta}\left\Vert \nabla f\left(\xx\right)_{S\cup S'}\right\Vert ^{2}\ ,
	\end{align*}
	where we used the inequality $\left(a+b\right)^{2}\leq2a^{2}+2b^{2}$.
	Applying Lemma \ref{lem:stdevexp} and using the fact that $\left\Vert \nabla f\left(\xx\right)_{S}\right\Vert \leq\epsilon$,
	we see that setting $\eta=1/\beta$ we obtain:
	\begin{align*}
	\mathbb{E}\left[f\left(\theta'\right)\right] & \leq f\left(\theta\right)+\beta\left(\left\Vert \left(\xx-\eta\nabla f\left(\xx\right)\right)_{S\setminus S'}\right\Vert ^{2}+\sigma^{2}\right)\\
	& +\beta\left(\left\Vert \left(\left(\eta-\frac{1}{\beta}\right)\nabla f\left(\xx\right)\right)_{S\cup S'}\right\Vert ^{2}+\sigma^{2}\right)-\frac{1}{2\beta}\left\Vert \nabla f\left(\xx\right)_{S\cup S'}\right\Vert ^{2}\\
	& \leq f\left(\theta\right)+2\beta\left(\left\Vert \xx_{S\setminus S'}\right\Vert ^{2}+\frac{\epsilon^{2}}{\beta^{2}}+\sigma^{2}\right)-\frac{1}{2\beta}\left\Vert \nabla f\left(\xx\right)_{S\cup S'}\right\Vert ^{2}\\
	& \leq f\left(\theta\right)+2\beta\left(\ss R_{\infty}^{2}+\frac{\epsilon^{2}}{\beta^{2}}+\sigma^{2}\right)-\frac{1}{2\beta}\left\Vert \nabla f\left(\xx\right)\right\Vert _{\infty}^{2}\ ,
	\end{align*}
	where again we used the fact that $\left\Vert \nabla f\left(\xx\right)_{S\cup S'}\right\Vert _{\infty}=\left\Vert \nabla f\left(\xx\right)\right\Vert _{\infty}$,
	unless $\left\Vert \nabla f\left(\xx\right)\right\Vert _{\infty}\leq\frac{R_{\infty}}{\eta}=\beta R_{\infty}$.
	Thus, telescoping over $T$ iterations we see that 
	\[
	\frac{1}{2\beta}\sum_{t=0}^{T-1}\left\Vert \nabla f\left(\xx_{t}\right)\right\Vert _{\infty}^{2}\leq f\left(\xx_{0}\right)-f\left(\xx_{T}\right)+T\cdot2\beta\left(\ss R_{\infty}^{2}+\beta^{-2}\epsilon^{2}+\sigma^{2}\right)
	\]
	and so returning a random point $\xx$ among those witnessed during
	the algorithm we have 
	\[
	\mathbb{E}\left[\left\Vert \nabla f\left(\xx\right)\right\Vert _{\infty}^{2}\right]\leq\frac{2\beta\left(f\left(\xx_{0}\right)-f\left(\xx_{T}\right)\right)}{T}+2\left(\beta^{2}\ss R_{\infty}^{2}+\epsilon^{2}+\beta^{2}\sigma^{2}\right)
	\]
	By AM-QM,
	\[
	\mathbb{E}\left[\left\Vert \nabla f\left(\xx\right)\right\Vert _{\infty}\right]\leq\sqrt{\mathbb{E}\left[\left\Vert \nabla f\left(\xx\right)\right\Vert _{\infty}^{2}\right]}
	\]
	which enables us to conclude that after sufficiently many iterations
	we are guaranteed to find a point such that 
	\[
	\left\Vert \nabla f\left(\xx\right)\right\Vert _{\infty}=O\left(\beta R_{\infty}\sqrt{k}+\beta\sigma+\epsilon\right)\ .
	\]
	
\end{proof}

\paragraph{Miscellaneous Proofs.}
\begin{lem}
	\label{lem:stdevexp}Let $\sigma > 0$ and let $g_{\xx}$ be a stochastic gradient satisfying
	standard conditions: 
	\[
	\mathbb{E}[g_{\xx}| \xx] = \nabla f(\xx)], 
	\]
	and
	\[
	\mathbb{E}[\|g_{\xx}- \nabla f(\xx)\|^2] \leq \sigma^2\,.
	\]
	Then for any vector $a\in\mathbb{R}^{N}$ and any subset $S$ of coordinates:
	\[
	\left\Vert \left(\nabla f\left(\xx\right)+a\right)_{S}\right\Vert ^{2}\leq\mathbb{E}\left[\left\Vert \left(g_{\xx}+a\right)_{S}\right\Vert ^{2}\bigg\vert\xx\right]\leq\left\Vert \left(\nabla f\left(\xx\right)+a\right)_{S}\right\Vert ^{2}+\sigma^{2}\ .
	\]
\end{lem}

\begin{proof}
	We expand the norm under the expected value as: 
	\begin{align*}
	\mathbb{E}\left[\left\Vert \left(g_{\xx}+a\right)_{S}\right\Vert ^{2}\bigg\vert\xx\right] & =\mathbb{E}\left[\left\Vert \left(\nabla f\left(\xx\right)+a\right)_{S}+\left(g_{\xx}-\nabla f\left(\xx\right)\right)_{S}\right\Vert ^{2}\bigg\vert\xx\right]\\
	& =\mathbb{E}\left[\left\Vert \left(\nabla f\left(\xx\right)+a\right)_{S}\right\Vert ^{2}+\left\Vert \left(g_{\xx}-\nabla f\left(\xx\right)\right)_{S}\right\Vert ^{2}+2\left\langle \left(\nabla f\left(\xx\right)+a\right)_{S},\left(g_{\xx}-\nabla f\left(\xx\right)\right)_{S}\right\rangle \bigg\vert\xx\right]\\
	& =\left\Vert \left(\nabla f\left(\xx\right)+a\right)_{S}\right\Vert ^{2}+\mathbb{E}\left[\left\Vert \left(g_{\xx}-\nabla f\left(\xx\right)\right)_{S}\right\Vert ^{2}\bigg\vert\xx\right]\\
	&+\mathbb{E}\left[2\left\langle \left(\nabla f\left(\xx\right)+a\right)_{S},\left(g_{\xx}-\nabla f\left(\xx\right)\right)_{S}\right\rangle \bigg\vert\xx\right]\\
	& \leq\left\Vert \left(\nabla f\left(\xx\right)+a\right)_{S}\right\Vert ^{2}+\sigma^{2}\ .
	\end{align*}
	From the the chain of equalities above, we can trivially lower bound
	the expectation by $\left\Vert \left(\nabla f\left(\xx\right)+a\right)_{S}\right\Vert ^{2}$
	, which gives us what we needed.
\end{proof}

A crucial step in the proof of Theorem~\ref{thm:iht-pl} requires upper bounding the $\ell_{2}$
distance to the closest global optimizer by the difference in function
value. In general, it is known that this is implied by the $\PLlong$
condition, and so it automatically holds for the stronger concentrated $\PLlong$ condition.
We reproduce the proof from~\citep{karimi2016linear} for completeness.
\begin{lem}
	\label{lem:conPLQG}(From $\PLlong$ to quadratic growth)
	Let $f:\mathbb{R}^{N}\rightarrow\mathbb{R}$ be a function satisfying
	the  $\PLlong$ inequality 
	\[
	\left\Vert \nabla f\left(\xx\right)\right\Vert ^{2}\geq\frac{\alpha}{2}\left(f\left(\xx\right)-f^{*}\right)\ .
	\]
	Then there exists a global minimizer $\xx^{*}$ of $f$ such that
	\[
	f\left(\xx\right)-f^{*}\geq\frac{\alpha}{8}\left\Vert \xx-\xx^{*}\right\Vert ^{2}\ .
	\]
\end{lem}

\begin{proof}
	Let $g\left(\xx\right)=\sqrt{f\left(\xx\right)-f^{*}}$
	for which we have
	\[
	\nabla g\left(\xx\right)=\frac{1}{2\sqrt{f\left(\xx\right)-f^{*}}}\nabla f\left(\xx\right)\ .
	\]
	Using the $\PLshort$ condition we have
	\begin{align*}
	\left\Vert \nabla g\left(\xx\right)\right\Vert ^{2} & =\frac{1}{4\left(f\left(\xx\right)-f^{*}\right)}\cdot\left\Vert \nabla f\left(\xx\right)\right\Vert ^{2}\geq\frac{1}{4\left(f\left(\xx\right)-f^{*}\right)}\cdot\frac{\alpha}{2}\cdot\left(f\left(\xx\right)-f^{*}\right)=\frac{\alpha}{8}\ .
	\end{align*}
	Now starting at some $\xx_{0}$, we consider the dynamic $\dot{\xx}=-\nabla g\left(\xx\right)$.
	We see that this always decreases function value until it reaches
	some $\xx_{T}$ for which $\nabla g\left(\xx_{T}\right)=0$
	and hence by the $\PLshort$ inequality, $\xx_{T}$ is a minimizer
	i.e. $f\left(\xx_{T}\right)=f^{*}$. Now we can write 
	\begin{align*}
	g\left(\xx_{T}\right) & =g\left(\xx_{0}\right)+\int_{0}^{T}\left\langle \nabla g\left(\xx_{t}\right),\dot{\xx_{t}}\right\rangle dt=g\left(\xx_{0}\right)+\int_{0}^{T}\left\langle \nabla g\left(\xx_{t}\right),-\nabla g\left(\xx_{t}\right)\right\rangle dt\\
	& =g\left(\xx_{0}\right)-\int_{0}^{T}\left\Vert \nabla g\left(\xx_{t}\right)\right\Vert ^{2}dt\ .
	\end{align*}
	Thus 
	\begin{align*}
	g\left(\xx_{0}\right)-g\left(\xx_{T}\right) & =\int_{0}^{T}\left\Vert \nabla g\left(\xx_{t}\right)\right\Vert ^{2}dt\geq\sqrt{\frac{\alpha}{8}}\cdot\int_{0}^{T}\left\Vert \nabla g\left(\xx_{t}\right)\right\Vert dt=\sqrt{\frac{\alpha}{8}}\cdot\int_{0}^{T}\left\Vert \dot{\xx_{t}}\right\Vert dt\ ,
	\end{align*}
	where we used our lower bound on the norm of $\nabla g\left(\xx\right)$.
	Finally, we use the fact that the last integral lower bounds the total
	movement of $\xx$ as it moves from $\xx_{0}$ to $\xx_{T}$. Thus
	\[
	\int_{0}^{T}\left\Vert \dot{\xx_{t}}\right\Vert dt\geq\left\Vert \xx_{0}-\xx_{T}\right\Vert \ ,
	\]
	so 
	\[
	g\left(\xx_{0}\right)-g\left(\xx_{T}\right)\geq\sqrt{\frac{\alpha}{8}}\left\Vert \xx_{0}-\xx_{T}\right\Vert \ ,
	\]
	which enables us to conclude that
	\[
	f\left(\xx_{0}\right)-f^{*}\geq\frac{\alpha}{8}\left\Vert \xx_{0}-\xx_{T}\right\Vert ^{2}\ ,
	\]
	where $\xx_{T}$ is some global minimizer of $f$. This concludes
	the proof.
\end{proof}

\section{Additional Experiments}

\subsection{CIFAR-100 Experiments}
\label{sec:cifar100}

We tested the IHT pruning algorithm on the CIFAR-100 dataset \citep{cifar100}. We used the WideResNet-28-10 architecture \citep{zagoruyko2016wide}, which to our knowledge gives state-of-the-art performance on this dataset. Models were trained for a fixed 200 epochs, using Stochastic Gradient Descent (SGD) with momentum, and a stepwise decreasing learning rate scheduler. For our main experiment, after a warm-up period of ten epochs, we alternated sparse and dense phases of 20 epochs until epoch 170, at which point we allowed the sparse model to train to convergence for 30 more epochs.  In addition, we found a small benefit to resetting momentum to 0 at each transition from sparse to dense, and we have done so throughout the trials. These experiments were replicated starting from three different seeds, and we report average results and their standard deviations, and they are shown in table \ref{table:sparse-cifar100} in rows labeled "AC/DC-20".

We further explored the possibility of using even smaller dense phases to further reduce training FLOPs. These trials, presented in \ref{table:sparse-cifar100} show the results of  reducing the sparse phase from 20 to 14 or even 7 epochs while keeping the length of dense phases the same, but increasing the number of total epochs to roughly match overall FLOPs (thus increasing the number of sparse and dense phases); overall we train the AC/DC-14 runs for 225 epochs and AC/DC-7 runs for 240 epochs (vs 200 for AC/DC-20). Our experiments show that it is possible to obtain competitive accuracy results with shorter dense phases, and, at higher sparsities, a further FLOP reduction - however, reducing the dense phase too much may lead to some accuracy degradation at higher sparsities. These results suggest that a further reduction in FLOPs is possible by adjusting the length of the dense phase and the overall epochs. We emphasize that $(i)$ these gains are only theoretical until hardware is available that can take advantage of sparsity in training and $(ii)$ these results, although promising, are highly preliminary; additionally, each trial was only run once due to timing constraints. 

We compare our results with Gradual Magnitude Pruning \citep{zhu2017prune}. To our knowledge, we are the first to release CIFAR-100 pruning results for this network architecture, and GMP was chosen as a baseline due to its generally strong performance against a range of other approaches \citep{gale2019state}. We obtain the GMP baseline by training on WideResNet architecture at full density for 50 epochs, then gradually increase the sparsity over the next 100 before allowing the model to converge for the final 50, matching the 200 training epochs of AC/DC. We further validated this baseline by training dense WideResNet-28-10 models for 200 epochs and then gradually pruning over 50 and finetuning over 50 more, for a total of 300 epochs, which gave similar performance at the cost of greater FLOPs and training time.

\begin{table}
\centering
\caption{CIFAR-100 Sparsity results on WideResNet}
\label{table:sparse-cifar100}
\begin{tabular}{@{}ccccc@{}}
\toprule
Method   & \begin{tabular}[c]{@{}c@{}}Top-1\\ Acc. ($\%$)\end{tabular} & Sparsity & \begin{tabular}[c]{@{}c@{}}GFLOPs\\ Inference \end{tabular} & \begin{tabular}[c]{@{}c@{}}EFLOPs\\ Train\end{tabular} \\ \midrule
Dense      & $79.0 \pm 0.25$ & $0\%$      & $11.9$ & $0.36$ \\ \midrule

AC/DC-20       & $79.6 \pm 0.17$  & $49.98\%$  & $0.50 \times$ & $0.72 \times$ \\
AC/DC-14       & $\bf{79.99}$  & $49.98\%$  & $0.50 \times$ & $0.75 \times$ \\
AC/DC-7       & $79.92$  & $49.98\%$  & $0.49 \times$ & $0.73 \times$ \\
GMP & $79.2 \pm 0.17$ & $49.98\%$    & $0.46 \times$ & $1.64 \times$  \\
\midrule  
AC/DC-20       & $\bf{80.0} \pm 0.17 $  & $74.96 \%$  & $0.29 \times$ & $0.6 \times$ \\
AC/DC-14       & ${79.5}  $  & $74.96 \%$  & $0.29 \times$ & $0.59 \times$ \\
AC/DC-7       & ${79.7}  $  & $74.96 \%$  & $0.29 \times$ & $0.54 \times$ \\
GMP      & $78.9 \pm 0.14$ & $74.96 \%$  & $0.26 \times$  & $1.52 \times$  \\
\midrule
AC/DC-20      & $\bf{79.1} \pm 0.07$  & $89.96\%$  & $0.14 \times$ & $0.51 \times$ \\
AC/DC-14      & $79.0 $  & $89.96\%$  & $0.14 \times$ & $0.47 \times$ \\
AC/DC-7      & $78.4 $  & $89.96\%$  & $0.14 \times$ & $0.39 \times$ \\
GMP & $77.7 \pm 0.23$ & $89.96\%$    & $0.08 \times$ & $1.44 \times$  \\
\midrule
AC/DC-20       & $78.2 \pm 0.12$  & $94.95\%$  & $0.08 \times$ & $0.47 \times$ \\
AC/DC-14       & $\bf{78.5}$  & $94.95\%$  & $0.08 \times$ & $0.41 \times$ \\
AC/DC-7       & $77.8 $  & $94.95\%$  & $0.08 \times$ & $0.33 \times$ \\
GMP & $76.6 \pm 0.07$ & $94.95\%$    & $0.07 \times$ & $1.41 \times$  \\
\bottomrule
\end{tabular}
\end{table}

\begin{figure*}[t]
    \centering
    \begin{subfigure}[t]{0.5\textwidth}
        \centering
        \includegraphics[height=1.3in]{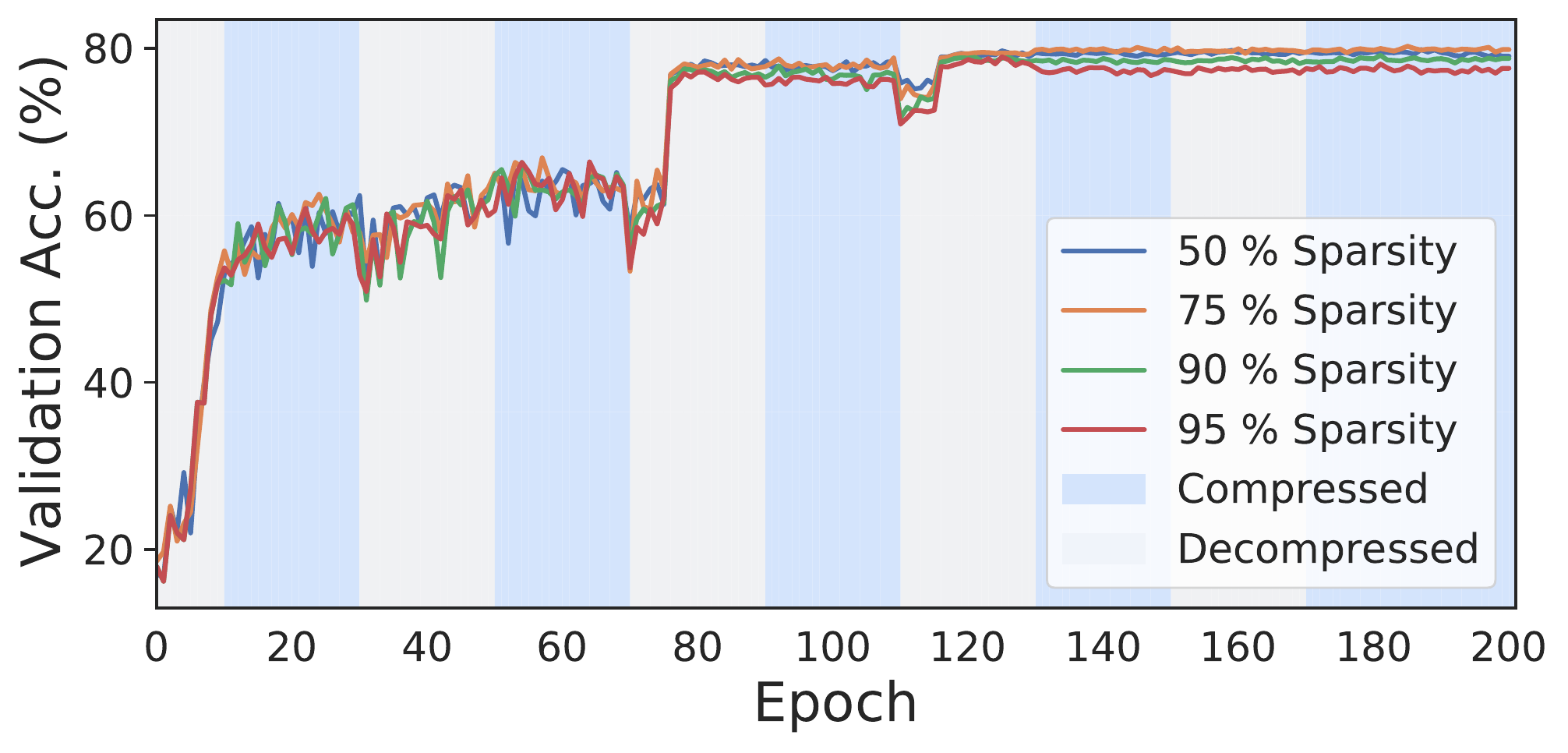}
        \caption{Sparsity pattern and test accuracy}
        \label{fig:valacc-rn50-all}
    \end{subfigure}%
    ~ 
    \begin{subfigure}[t]{0.5\textwidth}
        \centering
        \includegraphics[height=1.3in]{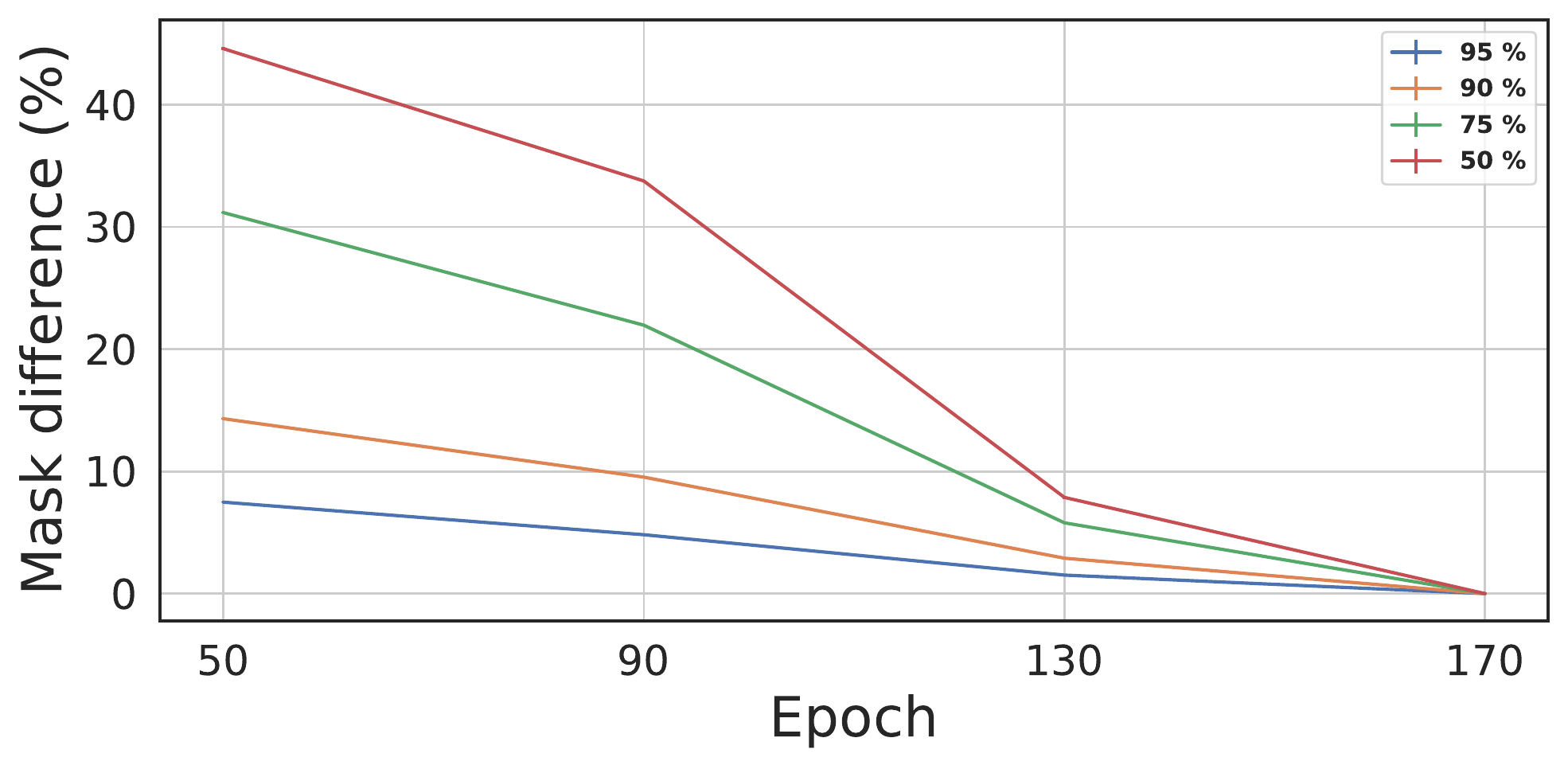}
        \caption{Relative change in consecutive masks}
        \label{fig:masks-rn50-all}
    \end{subfigure}%
    ~
    \caption{Validation accuracy and sparsity during training, together with differences in consecutive masks for WideResNet on CIFAR-100 using AC/DC.}
    \label{fig:acc-masks-rn50}
\end{figure*}

The results are shown in Table~\ref{table:sparse-cifar100}. We see that AC/DC pruning significantly outperforms Gradual Magnitude Pruning at all sparsity levels tested, and further that AC/DC models pruned at $50\%$ and $75\%$ even outperform the dense baseline, while the models pruned at $90\%$ at least match it.

\subsection{ResNet50 on ImageNet}

\parhead{Performance of dense models.} The AC/DC method has an advantage over other pruning methods of obtaining \emph{both} sparse and dense models. The performance of the dense baseline can be recovered after fine-tuning the resulting AC/DC dense model, for a small number of epochs. Namely, we start from the best dense baseline, which is usually obtained after 85 epochs, and replace the final compression phase of 15 epochs with regular dense training; we use the same learning rate scheduler and keep all other training hyper-parameters the same. For 80\% sparsity we recover the dense baseline accuracy completely, while for 90\% we are slightly below the baseline by $0.3\%$. We note that for 90\% sparsity, when the first and last layers are dense, our fine-tuned dense model recovers the baseline accuracy fully. The results for the dense models, together with the baseline accuracy, are presented in Table \ref{table:dense-rn50}, where ($\star$) denotes that the first and last layers of the network are dense. 

\begin{minipage}[c]{0.5\textwidth}
\centering
\captionof{table}{AC/DC Dense ResNet50}
\label{table:dense-rn50}
\begin{tabular}{@{}ccc@{}}
\toprule
\begin{tabular}[c]{@{}c@{}} Target \\ Sparsity \end{tabular} & \begin{tabular}[c]{@{}c@{}}Accuracy \\Dense (\%)\end{tabular} &  \begin{tabular}[c]{@{}c@{}}Accuracy \\ Finetuned (\%) \end{tabular}  \\ \midrule
$0\%$     & $76.84$    & - \\
$80\%$     & $73.82 \pm 0.02$    & $ 76.83 \pm 0.07$ \\
$90\%$     & $73.25 \pm 0.16$    & $ 76.56 \pm 0.1$  \\
$90\%^\star$ & 73.66 & 76.85 \\ \bottomrule
\end{tabular}
\end{minipage}
\begin{minipage}[c]{0.5\textwidth}
\centering
\captionof{table}{AC/DC Dense MobileNetV1}
\label{table:dense-mobnet}
\begin{tabular}{@{}ccc@{}}
\toprule
Sparsity & \begin{tabular}[c]{@{}c@{}}Accuracy \\ Dense (\%)\end{tabular} & \begin{tabular}[c]{@{}c@{}}Accuracy\\ Finetuned (\%)\end{tabular} \\ \midrule
0\%      & 71.78 & -  \\    
75\%     & $68.55  \pm 0.2$  & $71.63 \pm 0.1$    \\
90\%     & $67.47 \pm 0.13$    & $70.86 \pm 0.08$ \\ 
90\%$^\star$ & $67.65$ & $70.97$ \\ \bottomrule           
\end{tabular}
\end{minipage}

Moreover, an interesting property of AC/DC is that the resulting dense networks have a small percentage of zero-valued weights, as shown in Table~\ref{table:acdc-dense-zeros}. This is most likely caused by ``dead'' neurons or convolutional filters resulted after each compression phase; the corresponding weights do not get re-activated during the dense stages, as they can no longer receive gradients. This can be easily seen particularly for high sparsity (95\% and 98\%) where a non-trivial percentage of the weights remain inactive.

\parhead{Dynamics of masks and FLOPs during training.} The mask dynamics, measured by the relative change between two consecutive compression masks, have an important influence on the AC/DC training process. Namely, more changes between consecutive compression masks typically imply more exploration of the weights' space, and faster recovery from sub-optimal pruning decision, which in turn results in more accurate sparse models. As can be seen in Figure~\ref{fig:masks-rn50-all}, the relative mask difference between consecutive compression phases decreases during training, but it is critical to be maintained at a non-trivial level. For completeness, we also included the evolution of the validation accuracy during AC/DC training, for all sparsity levels (please see Figure~\ref{fig:valacc-rn50-all}); at 98\% sparsity in particular, it is easiest to see that dense phases enable the exploration of better pruning masks, which ensure that the sparse model improves continuously during training.

Despite the dynamics of the compression masks, we noticed that the sparsity distribution does not change significantly. This can be observed from the number of inference FLOPs per sample, at the end of each compression phase, in Figure \ref{fig:sparse-flops-rn50}. Interestingly, as training progresses, AC/DC also induces structured sparsity, as more neurons and convolutional filters get pruned. This was previously discussed in more detail (see Table~\ref{table:acdc-dense-zeros}), but can also be deduced from the decreasing inference FLOPs at the end of each dense phase, as shown in Figure~\ref{fig:dense-flops-rn50}.

\begin{figure*}[t]
    \centering
    \begin{subfigure}[t]{0.5\textwidth}
        \centering
        \includegraphics[height=1.3in]{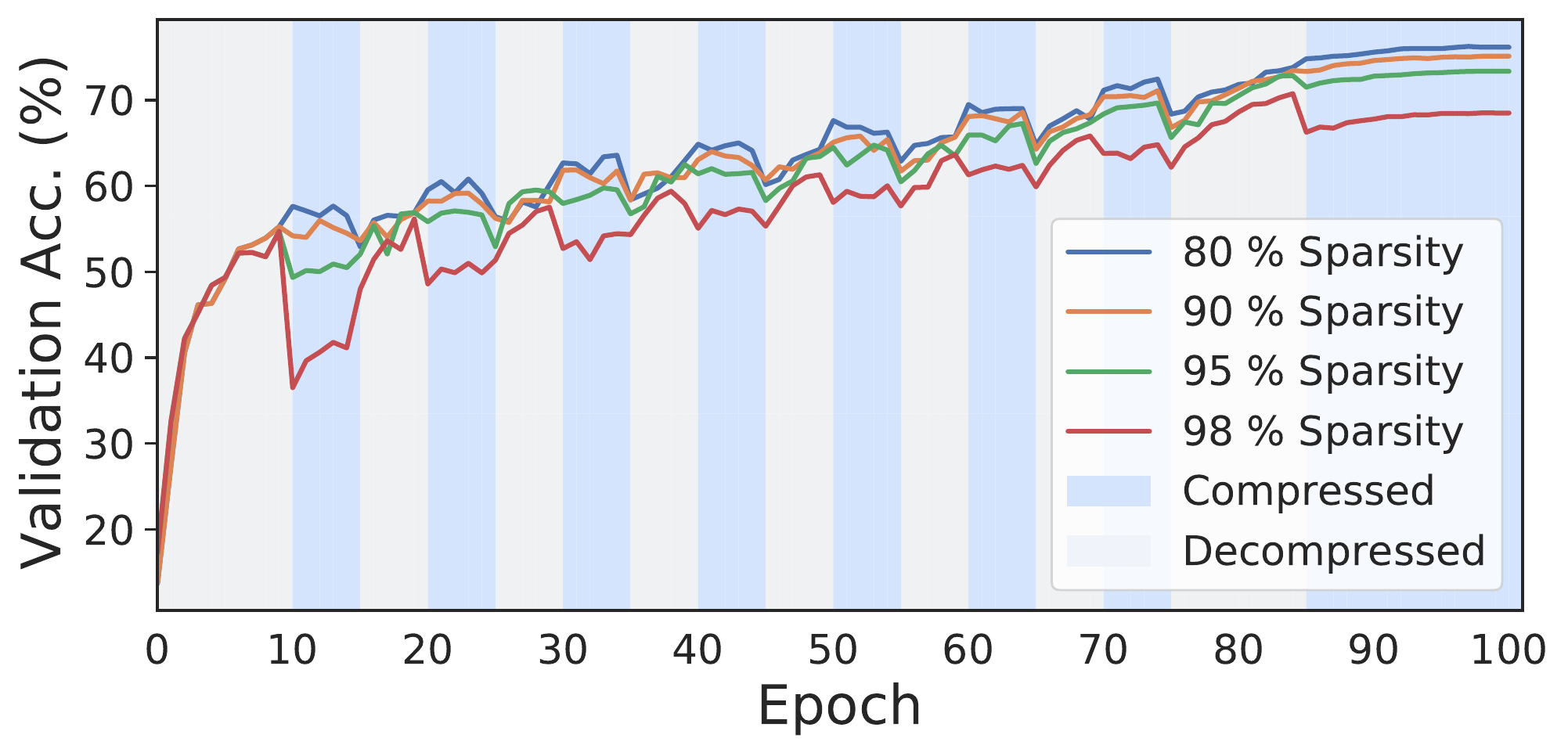}
        \caption{Sparsity pattern and test accuracy}
        \label{fig:valacc-rn50-all}
    \end{subfigure}%
    ~ 
    \begin{subfigure}[t]{0.5\textwidth}
        \centering
        \includegraphics[height=1.3in]{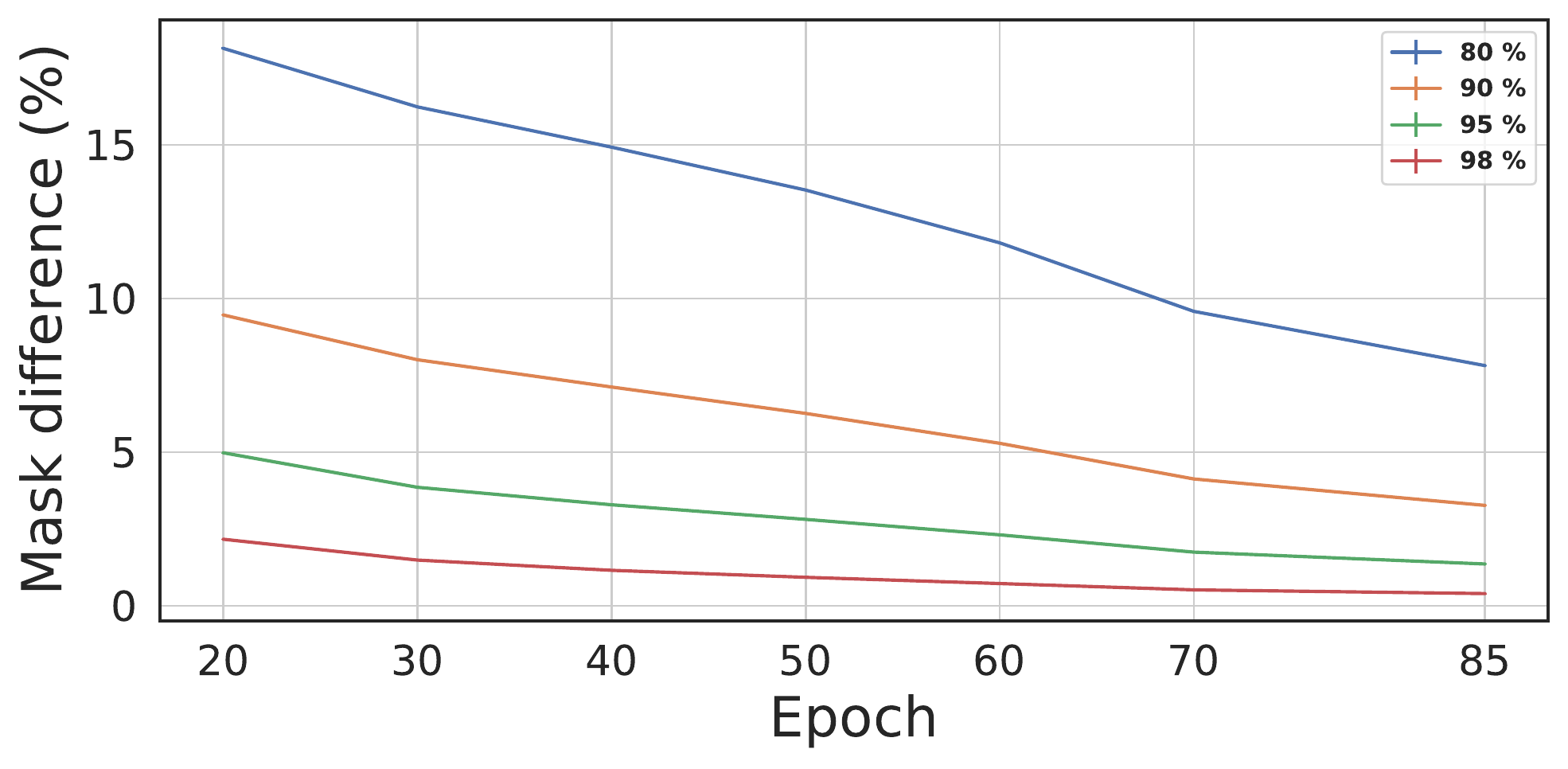}
        \caption{Relative change in consecutive masks}
        \label{fig:masks-rn50-all}
    \end{subfigure}%
    ~
    \caption{Validation accuracy and sparsity during training, together with differences in consecutive masks for ResNet50 on ImageNet using AC/DC.}
    \label{fig:acc-masks-rn50}
\end{figure*}

\begin{figure*}[t]
    \centering
    \begin{subfigure}[t]{0.45\textwidth}
        \centering
        \includegraphics[height=1.15in]{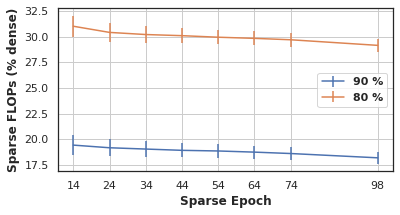}
        \caption{Test FLOPs after each sparse phase}
        \label{fig:sparse-flops-rn50}
    \end{subfigure}%
    ~ 
    \begin{subfigure}[t]{0.45\textwidth}
        \centering
        \includegraphics[height=1.15in]{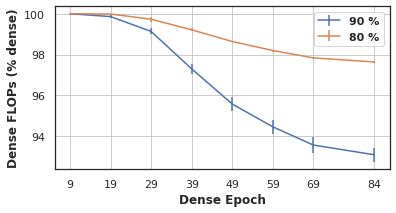}
        \caption{Test FLOPs after each dense phase}
        \label{fig:dense-flops-rn50}
    \end{subfigure}%
  \caption{Dynamics of sparse and dense inference FLOPs for ImageNet on ResNet50, as a percentage of the dense baseline FLOPs}
    \label{fig:flops-resnet50}
\end{figure*}

\begin{table}
\centering
\caption{Accuracy, sparsity, inference FLOPs and percentage of inactive weights for the resulting AC/DC dense models on ResNet50 (before fine-tuning, one seed).}
\label{table:acdc-dense-zeros}
\begin{tabular}{@{}cccc@{}}
\toprule
\begin{tabular}[c]{@{}c@{}}Target\\ Sparsity \end{tabular} & \begin{tabular}[c]{@{}c@{}}Top-1\\ Accuracy (\%) \end{tabular} &
\begin{tabular}[c]{@{}c@{}}Inference\\ FLOPs \end{tabular} &
\begin{tabular}[c]{@{}c@{}}Inactive\\ Weights (\%) \end{tabular} \\ \midrule
80      & $73.8$ & $0.98\times$ & 3.2 \\
90      & $73.1$ & $0.93\times$ & 10.5 \\
95      & $72.9$ & $0.85\times$ & 22.0 \\
98      & $70.8$ & $0.67\times$ & 49.8\\
\bottomrule
\end{tabular}
\end{table}

\parhead{AC/DC with uniform pruning.} As discussed, for example, in \cite{singh2020woodfisher}, global magnitude pruning usually performs better than its uniform counterpart. Interestingly, with global magnitude pruning later layers (which also tend to be the largest) are pruned the most. Moreover, we did not encounter convergence issues caused by entire layers being pruned, as hypothesized in some previous work \citep{evci2020rigging, jayakumar2020top}. However, one concern related to global magnitude pruning is a potential FLOP inefficiency of the resulting models; in theory, this would be a consequence of the earlier layers being pruned the least. For this reason, we performed additional experiments with AC/DC at uniform sparsity, with the first and last layers dense (as commonly used in the literature \cite{evci2020rigging, jayakumar2020top}). Our results show that there are no significant differences compared to AC/DC with global magnitude pruning. However, keeping the first and last layers dense significantly improves the results with global magnitude pruning. These observations emphasize that AC/DC is an easy-to-use method which works reliably well with different pruning criteria. For complete results, please see Table~\ref{table:unif-acdc}.

\begin{table}
\centering
\caption{AC/DC with uniform vs global magnitude pruning on ResNet50 (one seed), where ($\star$) denotes that the first and last layers are dense.}
\label{table:unif-acdc}
\begin{tabular}{@{}ccccc@{}}
\toprule
\begin{tabular}[c]{@{}c@{}}Sparsity\\ Distribution \end{tabular} & \begin{tabular}[c]{@{}c@{}}Target \\ Sparsity(\%)\end{tabular} &
\begin{tabular}[c]{@{}c@{}}Global \\ Sparsity(\%)\end{tabular}&
\begin{tabular}[c]{@{}c@{}}Top-1\\ Accuracy (\%) \end{tabular} & \begin{tabular}[c]{@{}c@{}}FLOPs\\ Inference \end{tabular} \\ \midrule
global     & $90$ & $89.8$ & $75.14$ & $0.18\times$ \\ 
global$^\star$ & $90$ & $82.6$ & $75.64$ & $0.21\times$ \\
uniform$^\star$ & $90$ & $82.6$ & $75.04$ & $0.13\times$ \\
\midrule
global  & $95$     & $94.8$ & $73.15$ & $0.11 \times$ \\ 
global$^\star$ & $95$ & $87.2$ & $74.16$ & $0.13\times$ \\ 
uniform$^\star$ & $95$ & $87.2$ & $73.28$ & $0.08\times$ \\
\bottomrule
\end{tabular}
\end{table}

\parhead{Direct comparison with Top-KAST.} As previously highlighted, Top-KAST is the closest to us, in terms of validation accuracy, out of existing sparse training methods. However, for the results reported, the authors kept the first convolutional and final fully-connected layers dense. To obtain a fair comparison, we used AC/DC on the same sparse distribution, and for 90\% sparsity over the pruned layers (82.57\% overall network sparsity), our results improved significantly. Namely, the best sparse model reached 75.64\% validation accuracy (0.6\% increase from the results in Table~\ref{table:medium-sparse-rn50}), while the accuracy of the best dense model was 76.85\% after fine-tuning. For more details, we also provide in Table \ref{table:us_vs_topkast} the results for Top-KAST when all layers are pruned, as they were provided to us by the authors. Notice that AC/DC surpasses even Top-KAST with dense back-propagation. 

It is important to note, however, that because of its flexibility in choosing the gradients density, Top-KAST can theoretically obtain significantly better training speed-ups than AC/DC, the latter being constrained by its dense training phases. This allows Top-KAST to improve the accuracy of the models by increasing the number of training epochs, while still enabling (theoretical) training speed-up. We present in Table~\ref{table:us_vs_topkast_nx} another comparison between AC/DC and Top-KAST, when the training time for the latter is increased 2 or 5 times; for all results (which were provided to us by the authors), the first and last layers for Top-KAST are dense. When comparing with AC/DC with all layers pruned, Top-KAST obtains better results at 98\% and 95\% sparsity, with increased training epochs. However, when using the same sparse distribution as Top-KAST (not pruning the first and last layers), the results for AC/DC at 95\% and 98\% sparsity are significantly better than Top-KAST with increased steps. For all the results reported on AC/DC the number of training steps was fixed at 100 epochs.

We note that the results obtained with AC/DC can be improved as well with increased number of training epochs. As an example, when using the same sparsity schedule extended over 150 epochs, the best sparse model obtained with AC/DC on 90\% sparsity reached $75.99\%$ accuracy, using fewer training FLOPs compared to the original dense baseline trained on 100 epochs (namely 87\%). Furthermore, when we fine-tune the dense model by replacing the final 15 epochs compression phase with dense training, we obtain a dense model with $76.95\%$ accuracy, higher than the original dense baseline.

\begin{table}
\centering
\caption{Comparison with Top-KAST when pruning all layers (ResNet50)}
\label{table:us_vs_topkast}

\begin{tabular}{@{}cccc@{}}
\toprule
Method   & Sparsity (\%) & \multicolumn{1}{c}{\begin{tabular}[c]{@{}c@{}}Backward \\ Sparsity (\%) \end{tabular}}  & \multicolumn{1}{c}{\begin{tabular}[c]{@{}c@{}}Sparse Top-1 \\ Accuracy (\%)\end{tabular}} \\
\midrule
\bf{AC/DC}    & 80 & 80  / 0 & $\bf{76.3} \pm 0.1$   \\
Top-KAST & 80 & 0 & 75.64  \\
Top-KAST & 80 & 50 & 74.78 \\
Top-KAST & 80 & 80 & 72.19 \\
\midrule
\bf{AC/DC}    & 90 & 90 / 0 & $\bf{75.03} \pm 0.1$ \\
Top-KAST & 90 & 0 & 74.42  \\
Top-KAST & 90 & 50 & 74.09 \\
Top-KAST & 90 & 80 & 73.07 \\

\bottomrule               
\end{tabular}
\end{table}

\begin{table}
\centering
\caption{Comparison with Top-KAST with increased training steps (ResNet50). ($\star$) indicates that the first and last layers are dense for AC/DC, while this is the case for all Top-KAST results.}
\label{table:us_vs_topkast_nx}

\begin{tabular}{@{}cccccc@{}}
\toprule
Method   & \multicolumn{1}{c}{\begin{tabular}[c]{@{}c@{}} Sparsity \\ (\%) \end{tabular}} & \multicolumn{1}{c}{\begin{tabular}[c]{@{}c@{}}Backward \\ Sparsity (\%) \end{tabular}}  & \multicolumn{1}{c}{\begin{tabular}[c]{@{}c@{}}Sparse Top-1 \\ Accuracy (\%)\end{tabular}} & \multicolumn{1}{c}{\begin{tabular}[c]{@{}c@{}}Train \\ FLOPs (\%)\end{tabular}} & \multicolumn{1}{c}{\begin{tabular}[c]{@{}c@{}}Inference \\ FLOPs (\%)\end{tabular}}  \\
\midrule
\bf{AC/DC}    & 80 & 80  / 0 & $\bf{76.3} \pm 0.1$  & $0.65\times$ & $0.29\times$ \\
Top-KAST$_{1\times}$ & 80 & 0 & 75.59 & $0.48\times$ & $0.23\times$  \\
 Top-KAST$_{1\times}$ & 80 & 60 & 74.59 & $0.29\times$ & $0.23\times$ \\
Top-KAST$_{2\times}$ & 80 & 0 & 76.11 & $0.97\times$ & $0.23\times$ \\
Top-KAST$_{2\times}$ & 80 & 60 & 75.29 & $0.58\times$ & $0.23\times$ \\
\midrule
AC/DC    & 90 & 90 / 0 & $75.03 \pm 0.1$ & $0.58\times$ & $0.18 \times$ \\ 
\bf{AC/DC$^\star$}    & 90 & 90 / 0 & $\bf{75.64}$ & $0.6\times$ & $0.21\times$\\
AC/DC$^\star$ unif. & $90$ & $90/0$ & $75.04$ & $0.55\times$ & $0.13\times$ \\
Top-KAST$_{1\times}$ & 90 & 0 & 74.65 & $0.42\times$ & $0.13\times$\\
Top-KAST$_{1\times}$ & 90 & 80 & 73.03 & $0.16\times$ & $0.13\times$ \\
Top-KAST$_{2\times}$ & 90 & 0 & 75.35 & $0.84\times$ & $0.13\times$ \\
Top-KAST$_{2\times}$ & 90 & 80 & 74.16 & $0.32\times$ & $0.13\times$ \\
\midrule
AC/DC    & 95 & 95 / 0 & $73.14\pm 0.2$ & $0.53\times$ & $0.11\times$ \\ 
\bf{AC/DC}$^\star$   & 95 & 95 / 0 & \bf{74.16} & $0.54\times$ & $0.13\times$ \\ 
AC/DC$^\star$ (unif)   & 95 & 95 / 0 & 73.28 & $0.5\times$ & $0.08\times$\\
Top-KAST$_{1\times}$ & 95 & 0 & 71.83 & $0.39\times$ & $0.08\times$\\
Top-KAST$_{1\times}$ & 95 & 90 & 70.42 & $0.1\times$ & $0.08\times$ \\
Top-KAST$_{2\times}$ & 95 & 0 & 73.29 & $0.77\times$ & $0.08\times$ \\
Top-KAST$_{2\times}$ & 95 & 90 & 72.42 & $0.19\times$ & $0.08\times$ \\
\bf{Top-KAST$_{5\times}$} & 95 & 0 & \bf{74.27} & $1.94\times$ & $0.08\times$ \\
Top-KAST$_{5\times}$ & 95 & 90 & 73.17 & $0.48\times$ & $0.08\times$ \\
\midrule
AC/DC    & 98 & 98 / 0 & $68.44\pm 0.09$ & $0.46\times$ & $0.06\times$\\
\bf{AC/DC$^\star$}    & 98 & 98 / 0 & $\bf{71.27}$ & $0.47\times$ & $0.08\times$ \\
Top-KAST$_{1\times}$ & 98 & 90 & 67.06 & $0.08\times$ & $0.05\times$\\
Top-KAST$_{1\times}$ & 98 & 95 & 66.46 & $0.06\times$ & $0.05\times$ \\
Top-KAST$_{2\times}$ & 98 & 90 & 68.99 & $0.15\times$ & $0.05\times$ \\
Top-KAST$_{2\times}$ & 98 & 85 & 68.87 & $0.12\times$ & $0.05\times$ \\

\bottomrule               
\end{tabular}
\end{table}

\subsection{MobileNet on ImageNet}

\parhead{Performance of dense models.} Similar to ResNet50, we observed that dense models obtained with AC/DC are able to recover the baseline accuracy after additional fine-tuning. We performed fine-tuning identically to the ResNet50 experiments and observe that AC/DC models obtained with a 75\% target sparsity recovered the baseline accuracy, while for 90\% the gap is just below 1\%. We present results for the (fine-tuned) dense models in Table \ref{table:dense-mobnet}, where ($\star$) indicates that the first layer and depth-wise convolutions were never pruned.

\parhead{Masks dynamics.} Similar to ResNet50, the change between consecutive AC/DC compression masks plays an important role in obtaining accurate sparse models on MobileNet. As shown in Figure~\ref{fig:masks-mobnet}, the compression masks stabilize as training progresses. For completeness, we also illustrate the evolution of the validation accuracy during AC/DC training on MobileNet, at 75\% and 90\% sparsity, in Figure~\ref{fig:valacc-mobnet}.

\parhead{Comparison with RigL.} We note that the results obtained by RigL \citep{evci2020rigging} improve significantly when increasing the number of training steps 2 or 5 times. Moreover, for all results reported with RigL on MobileNet the first convolutional layer and all depth-wise convolutions are dense, whereas we do not impose such restrictions on our sparse model. Our results can further be improved by using the same sparsity distribution; namely, for 90\% sparsity over the pruned parameters (88.57\% overall sparsity), the best sparse model obtained with AC/DC achieved 66.56\% accuracy (0.5\% improvement), while the best dense improved from 67.64\% to 70.97\% after fine-tuning. 
In Table \ref{table:us_vs_rigl} we present results for AC/DC and RigL at 75\% and 90\% sparsity, when the latter is trained over the same number of epochs, or with 2x or 5x the number of passes through the training data. We conclude that AC/DC has very similar validation accuracy to RigL$_{2\times}$. For $75\%$ sparsity, AC/DC achieves similar performance with significantly fewer training and inference FLOPs than RigL. At $90\%$ sparsity, AC/DC and RigL$_{2\times}$ are close in terms of both validation accuracy and training FLOPs; however, the validation accuracy of AC/DC can be improved by almost $0.5\%$ when the first and depth-wise convolutional layers are kept dense. We note that RigL$_{5\times}$ has significantly higher validation accuracy, and for $75\%$ sparsity it even matches the baseline; however, this variant of RigL also uses $2.6\times$ and $1.5\times$ the dense baseline training FLOPs for $75\%$ and $90\%$ sparsities, respectively, which makes it impractical due to its high computational training cost. 

\begin{figure*}[t]
    \centering
    \begin{subfigure}[t]{0.5\textwidth}
        \centering
        \includegraphics[height=1.3in]{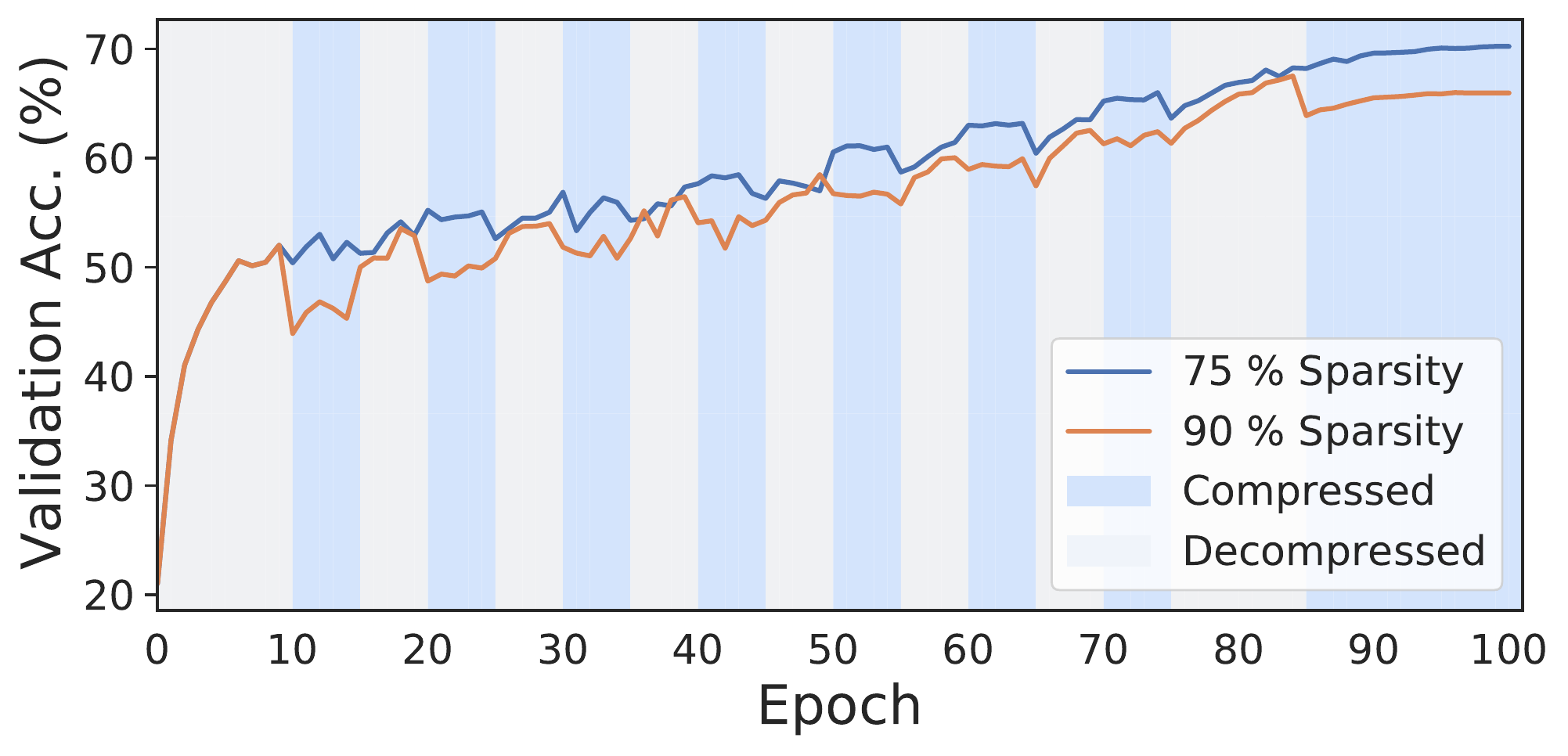}
        \caption{Sparsity pattern and test accuracy}
        \label{fig:valacc-mobnet}
    \end{subfigure}%
    ~ 
    \begin{subfigure}[t]{0.5\textwidth}
        \centering
        \includegraphics[height=1.3in]{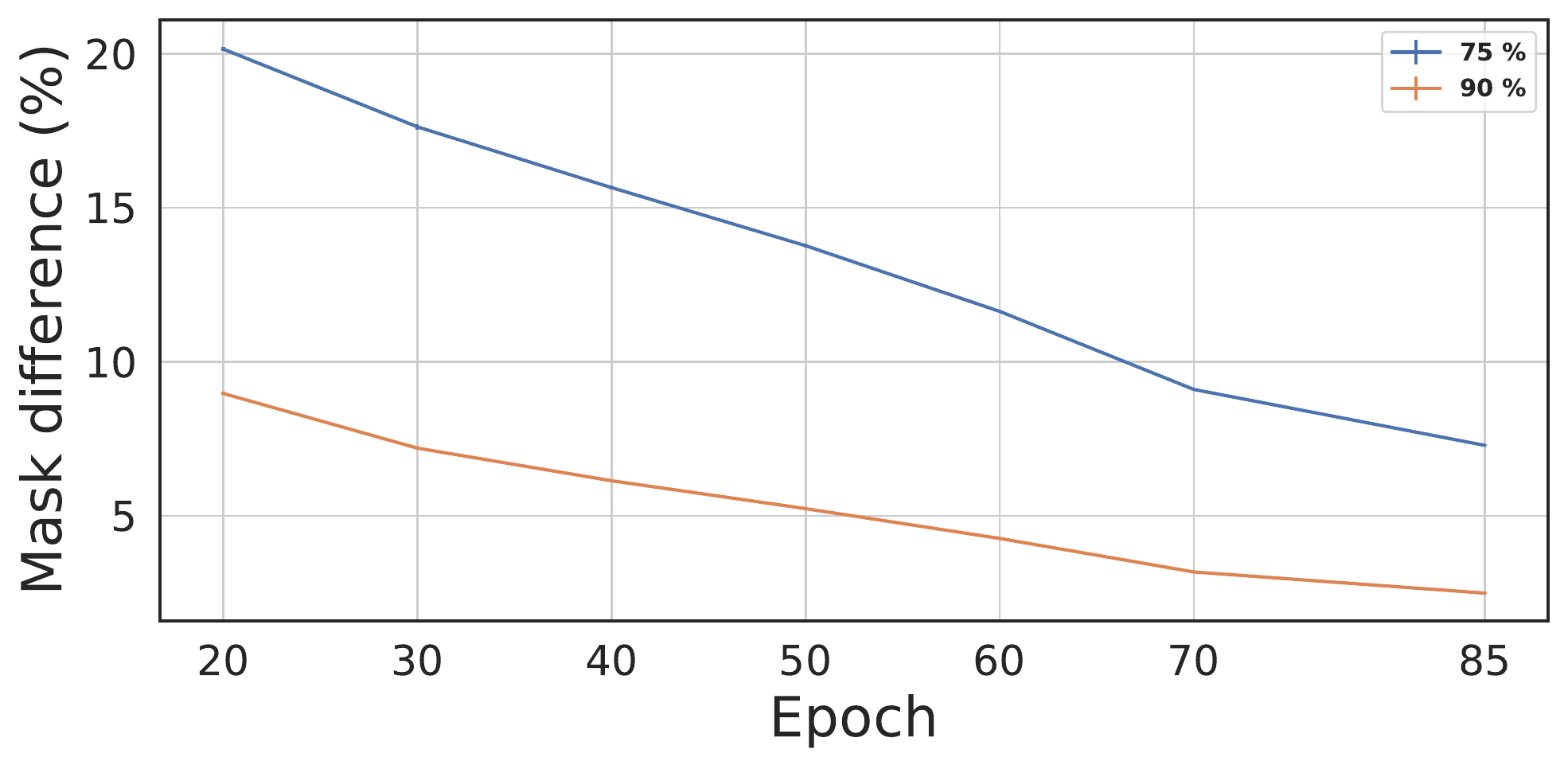}
        \caption{Relative change in consecutive masks}
        \label{fig:masks-mobnet}
    \end{subfigure}%
    ~
    \caption{Validation accuracy and sparsity during training, together with differences in consecutive masks for ImageNet with MobileNetV1 using AC/DC.}
    \label{fig:acc-masks-mobnet}
\end{figure*}

\begin{table}
\centering
\caption{Comparison between AC/DC and RigL on MobileNet, where ($\star$) denotes that the first and depth-wise convolutions were kept dense.}
\label{table:us_vs_rigl}

\begin{tabular}{@{}ccccc@{}}
\toprule
Method   & Sparsity (\%) & \multicolumn{1}{c}{\begin{tabular}[c]{@{}c@{}}Top-1 \\ Accuracy (\%)\end{tabular}} & \multicolumn{1}{c}{\begin{tabular}[c]{@{}c@{}} Inference \\ FLOPs \end{tabular}} & \multicolumn{1}{c}{\begin{tabular}[c]{@{}c@{}} Train \\ FLOPs \end{tabular}} \\
\midrule
AC/DC  & 75 &  70.3  & $0.34\times$ & $0.64\times$ \\ 
AC/DC$^\star$  & 75 &  70.41  & $0.36\times $ & $0.66\times$ \\ 
RigL$^\star$ (ERK) & 75 & 68.39 & $0.52\times$ & $0.52\times$  \\
RigL$^\star_{2\times}$(ERK) & 75 &  70.49 & $0.52\times$ & $1.05\times$ \\
RigL$^\star_{5\times}$(ERK) & 75 &  71.9 & $0.52\times$ & $2.63\times$ \\
\midrule

AC/DC    & 90 &  66.08  & $0.18\times$ & $0.56\times$ \\
AC/DC$^\star$  & 90 &  66.56  & $0.21\times$ & $0.58\times$ \\
RigL$^\star$(ERK) & 90 &  63.58 & $0.27\times$ & $0.29 \times$  \\
RigL$_{2\times}^\star$(ERK) & 90 & 65.92 & $0.27\times$ &  $0.59\times$ \\
RigL$_{5\times}^\star$(ERK) & 90 &  68.1 & $0.27 \times$ & $1.47 \times$ \\

\bottomrule               
\end{tabular}
\end{table}

\subsection{Inference Speedups}

We now examine the potential for real-world speedup of models produced through our framework. 
For this, we use the CPU-based inference framework of~\citep{NM}, which supports efficient inference over unstructured sparse models, and is free to use for non-commercial purposes.   
Specifically, we export our Pytorch-trained models to the ONNX intermediate format, preserving weight sparsity, and then execute inference on a subset of samples, at various batch sizes, measuring time per batch. We execute on an Intel i9-7980XE CPU with 16 cores and 2.60GHz core frequency. 
We simulate two scenarios: the first is \emph{real-time inference}, i.e. samples are processed one at a time, in a resource-constrained environment, using only 4 cores.
The second is \emph{batch inference}, for which we pick batch size 64, in a cloud environment, for which we use all 16 cores. 
We measure average time per batch for the sparse models against dense baselines, for which we use both the Deepsparse engine, and the ONNX runtime (ONNXRT). 
We present the average over 10 runs. The variance is extremely low, so we omit it for readability.

\begin{table}
\centering
\caption{Time per batch (milliseconds) using a sparse inference engine~\citep{NM}.}
\label{table:inference}

\begin{tabular}{@{}ccccc@{}}
\toprule
Model/Setup   & Real-Time Inference, 4 cores & Batch 64 Inference, 16 cores \\
\midrule
ResNet50 ONNXRT v1.6 &	14.773	& 329.734  \\
ResNet50 Dense	& 15.081	& 285.958 \\
ResNet50 90\% Pruned &	9.46 &	124.193 \\
ResNet50 90\% Unif. Pruned &	8.495 &	116.897 \\
\midrule
MobileNetV1 ONNXRT v1.6	& 2.552	& 80.748 \\		
MobileNetV1 Dense	& 2.513	& 55.845 \\
MobileNetV1 Pruned 75\%	& 1.96	& 40.976 \\
MobileNetV1 Pruned 90\%	& 1.468	& 34.909 \\
\bottomrule               
\end{tabular}
\end{table}

We now briefly discuss the results. 
First, notice that the dense baselines offer similar performance for real-time inference, but that the Deepsparse engine has a slight edge at batch 64. We will therefore compare against its timings below. 
The results show a speedup of 1.6x for the 90\% global-pruned ResNet50 model, and 1.8x for the uniformly pruned one: the uniformly-pruned model is slightly faster, which correlates with its lower FLOP count due to the uniform pruning pattern. 
This pattern is preserved in MobileNetV1 experiments, although the speedups are relatively lower, since the architecture is more compact.  
We note that the speedups are more significant for batched inference, where the engine has more potential for parallelization, and our setup uses more cores. 

\subsection{Sparse-Dense Output Comparison}

\begin{table}
\centering
\caption{Sample agreement between ResNet50 sparse and dense models}
\label{table:similarity-vd}

\begin{tabular}{@{}cccccc@{}}
\toprule
Method   & Sparsity &  \multicolumn{1}{c}{\begin{tabular}[c]{@{}c@{}}Sparse Top-1 \\ Accuracy (\%)\end{tabular}} & \multicolumn{1}{c}{\begin{tabular}[c]{@{}c@{}}Dense Top-1 \\ Accuracy (\%)\end{tabular}} & \multicolumn{1}{c}{\begin{tabular}[c]{@{}c@{}}Sparse-Dense \\ Agreement (\%) \end{tabular}} &
\multicolumn{1}{c}{\begin{tabular}[c]{@{}c@{}}Sparse-Dense \\ Cross-entropy \end{tabular}}
\\
\midrule

AC/DC    & 80\% & $76.3 \pm 0.1$      & $76.8 \pm 0.07$ & $89.8 \pm 0.3$ & $0.85 \pm 0.005$   \\
SparseVD & 80\% & 75.3 & 75.2 & 98.6 & -\\
GMP & 80\% & 76.4 & 76.9 & 86.0 & 1.03\\
\midrule
AC/DC    & 90\% & $75.0 \pm 0.1$      & $76.6 \pm 0.09$  & $86.8 \pm 1.5$ & $1.02 \pm 0.004$    \\
SparseVD & 90\% & 73.8 & 73.6 & 98.3  & -\\
GMP & 90\% & 74.7 & 76.9 & 83.5  & 1.29 \\
\bottomrule               
\end{tabular}
\end{table}

To the best of our knowledge, our results are the first ones to show that \emph{both} a dense and a sparse model can be trained jointly. Although other sparse training techniques such as RigL or Top-KAST can train sparse models faster, none of them offer the additional benefit of an accurate dense model. 

One method that does generate sparse-dense model couples is Sparse Variational Dropout (SparseVD) \citep{molchanov2017variational}; there, after training a dense model with variational inference, a large proportion of the weights can be pruned in a single step, without affecting the accuracy of the dense model. (We note however that Sparse Variational Dropout doubles the FLOP cost of training, due to the variational parameters.)
However, our investigation of the SparseVD models trained by \citep{gale2019state} shows that the sparse and dense models 
agree in over 98\% of their predictions as measured on the ImageNet validation set, and are of no better quality than the sparse model - if anything, they are slightly worse. Please see Table~\ref{table:similarity-vd} for complete results.

In comparison, AC/DC with finetuning produces dense models of validation accuracy that is comparable to that of a dense model trained without any compression, and that therefore do differ from their sparse co-trained counterparts. To understand the relative sizes of these differences, we used GMP pruning as the baseline. In particular, we compared the similarity of a fully trained dense model with GMP trained  
over 100 epochs. We note AC/DC and GMP show comparable accuracy for both their sparse and dense models in this scenario; however, the total training epochs are substantially lower for producing these models with AC/DC.

We use two metrics to investigate the difference in sparse-dense model pairs: the proportion of validation examples on which the top 
prediction agreed between the two models, and the average cross-entropy of the predictions across all validation examples. In both metrics, Table~\ref{table:similarity-vd} shows that model similarity is higher for $80\%$ sparsity than $90\%$, and higher for AC/DC than GMP training: in particular, sparse/dense cross-entropy is about $20\%$ lower for AC/DC, and the number of top-prediction disagreements is about $25\%$ lower for AC/DC at $90\%$ sparsity, and $37\%$ lower for AC/DC at $80\%$ sparsity.

\subsection{Memorization Experiments on CIFAR10}

In what follows, we study the similarities between the sparse and dense models learned with AC/DC, on the particular setup of memorizing random labels. Specifically, we select 1000 i.i.d. training samples from the CIFAR10 dataset and randomly change their labels. We train a ResNet20 model using AC/DC, at various target sparsity levels, ranging from 50\% to 95\%. We use SGD with momentum, weight decay, and initial learning rate $0.1$ which is decayed by a factor of 10 every $60$ epochs, starting with epoch 65. 

Using data augmentation dramatically affects the memorization of randomly-labelled training samples, and thus we differentiate between the two possible cases. Namely, the regular baseline can easily memorize (in the sense of reaching perfect accuracy) the randomly-labelled samples, when \emph{no} data augmentation is used; in comparison, with data augmentation memorization is more difficult, and the accuracy on randomly-labelled samples for the baseline is just above 60\%. In addition to the accuracy on the perturbed samples with respect to their new random labels, we also track the accuracy with respect to the ``true'' or correct labels. This differentiation offers a better understanding regarding where memorization fails and a glimpse into the robustness properties of neural networks in general, and of AC/DC, in particular.

\parhead{No data augmentation.} As previously mentioned, in this case the baseline model can perfectly memorize the perturbed data, with respect to their random labels. Interestingly, prior to the initial learning rate decay, most ($\geq 70\%$) perturbed samples are still correctly classified with respect to their ``true'' labels, and memorization happens very quickly after the learning rate is decreased. In the case of AC/DC with low target sparsity (50\% and 75\%), memorization has a very similar behavior to the dense baseline. However, for higher sparsity levels (90\% and 95\%) we can see a clear difference between the sparse and dense models. Namely, during each compression phase most perturbed samples are correctly classified with respect to their true labels, whereas in decompression phases their random labels are memorized. This phenomenon is illustrated in Figure~\ref{fig:memo-no-da}.

\begin{figure*}[t]
     \centering
    \begin{subfigure}[t]{0.5\textwidth}
        \centering
        \includegraphics[height=1.3in]{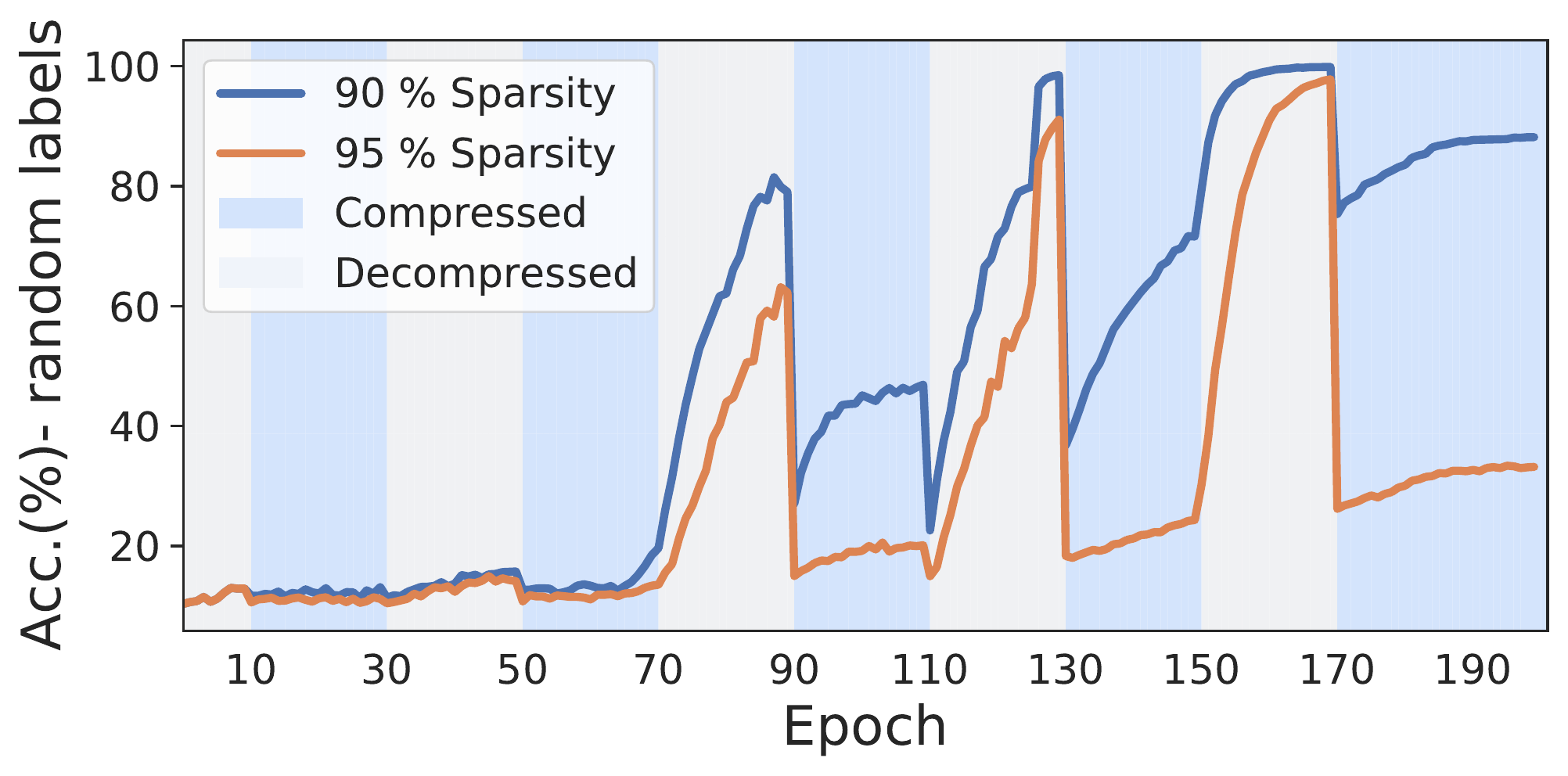}
        \caption{Accuracy on the mis-labelled data}
        \label{fig:random-no-da}
    \end{subfigure}%
    ~ 
    \hfill %
    \begin{subfigure}[t]{0.5\textwidth}
        \centering
        \includegraphics[height=1.3in]{figures/cifar10_random_correct_no_da_all.pdf}
        \caption{Accuracy on the mis-labelled data (w.r.t. the true labels)}
        \label{fig:random-correct-no-da}
    \end{subfigure}%
    
    \caption{Accuracy during training with AC/DC at 90\% and 95\% target sparsity, for 1000 randomly labelled CIFAR10 images. No data augmentation was applied to the training samples.}
    \label{fig:memo-no-da}
\end{figure*}

\parhead{Data augmentation.} In this case, memorization of the perturbed samples is more difficult, and it happens later on during training, usually after the second learning rate decrease for the baseline model. Interestingly, in the case of AC/DC we can see (Figure \ref{fig:memo-da}) a clear inverse relationship between the amount of memorization and the target sparsity. Although low sparsity enables more memorization, most perturbed samples are still correctly classified with respect to their true labels. For higher sparsity levels (90\% and 95\%), most perturbed samples are correctly classified with respect to their true labels (almost 90\%) and very few are memorized. Furthermore, the dense model resulted from AC/DC training is more robust than the original baseline, as it still learns the correct labels of the perturbed samples, despite being presented with random ones.

\begin{figure*}[t]
     \centering
    \begin{subfigure}[t]{0.5\textwidth}
        \centering
        \includegraphics[height=1.4in]{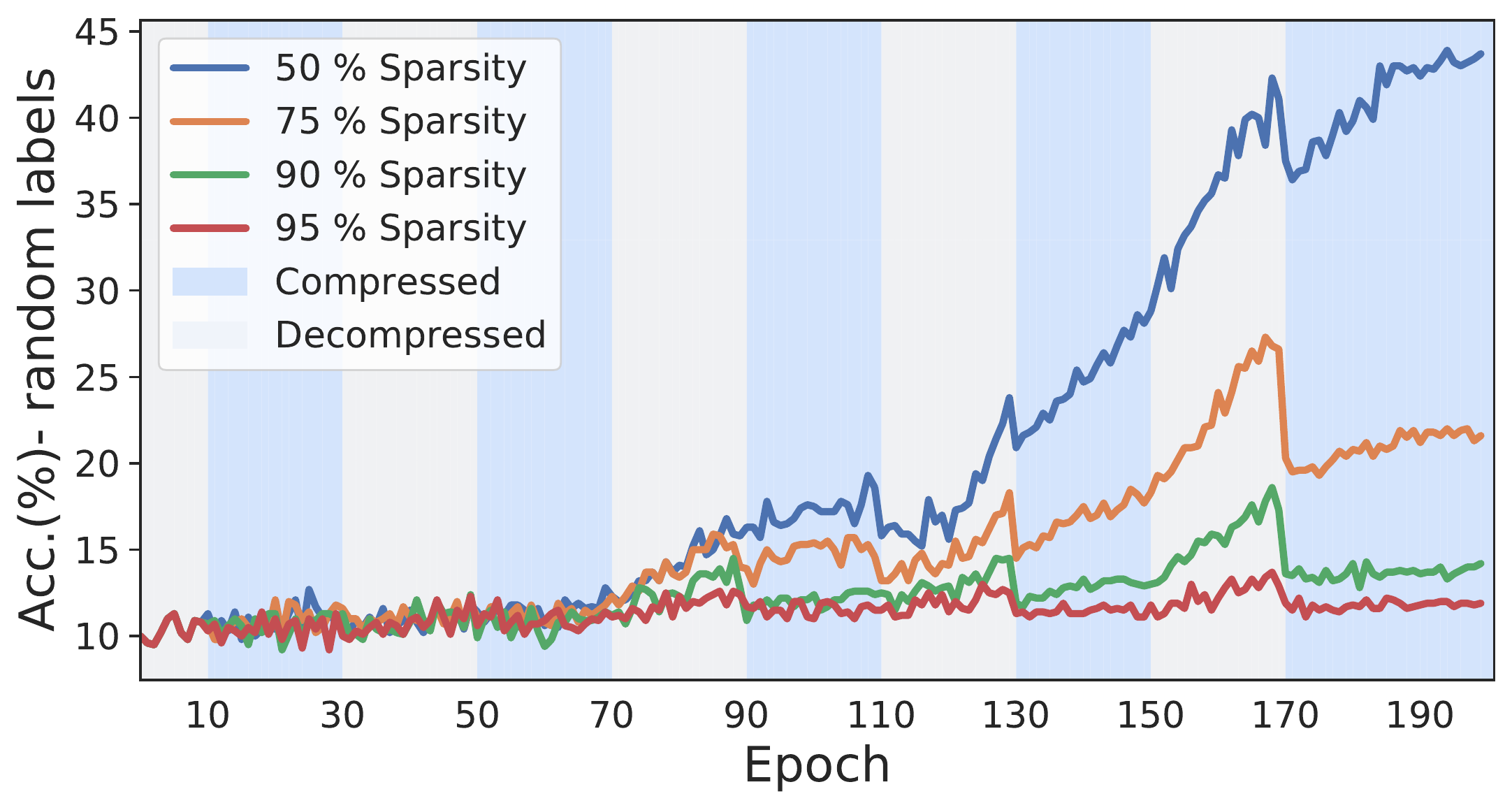}
        \caption{Accuracy on the mis-labelled data}
        \label{fig:random-da}
    \end{subfigure}%
    ~ 
    \begin{subfigure}[t]{0.5\textwidth}
        \centering
        \includegraphics[height=1.4in]{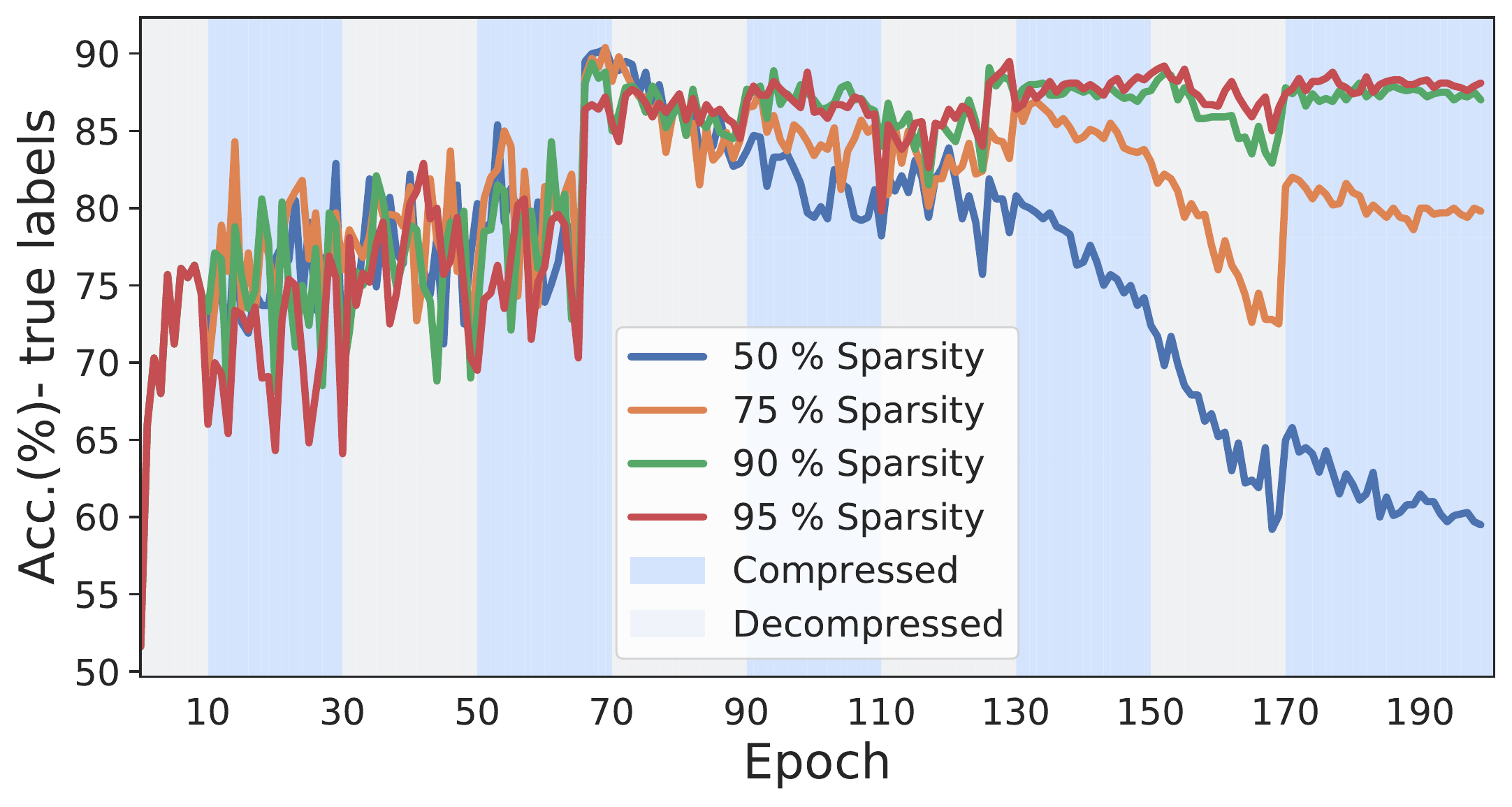}
        \caption{Accuracy on the mis-labelled data (w.r.t. the true labels)}
        \label{fig:random-correct-da}
    \end{subfigure}%
    
    \caption{Accuracy during training with AC/DC at 50\%, 75\%, 90\% and 95\% target sparsity, for 1000 randomly labelled CIFAR10 images. Here, all samples were trained using data augmentation.}
    \label{fig:memo-da}
\end{figure*}

\section{Computational Details}

\subsection{Hardware Details}
\label{subsec:hardware}

Experiments were run on NVIDIA RTX 2080 GPUs for image classification tasks, and NVIDIA RTX 3090 GPUs for language modelling. Each ImageNet run took approximately 2 days for ResNet50 and one day for MobileNet, while each Transformer-XL experiment took approximately 2 days. 

\subsection{FLOPs Computation}

When computing FLOPs, we take into account the number of zero-valued weights for linear and convolutional layers. 
To compute the FLOPs required for a backward pass over a sample, we use the same convention as RigL \citep{evci2020rigging}; namely, if $F$ denotes the inference FLOPs per sample, the number of backward FLOPs is estimated as $B = 2 \cdot F$, as we need $F$ FLOPs to backpropagate the error, and additional $F$ to compute the gradients w.r.t. the weights. For ImageNet experiments, we ignore the FLOPs required for Batch Normalization, pooling, ReLU or Cross Entropy, similarly to other methods \cite{evci2020rigging, singh2020woodfisher, kusupati2020soft}; however, these layers have a negligible impact on the total FLOPs number (at most $0.01\times$ the dense number).

For compression and decompression phases $C$ and $D$, we consider $F_C$ and $F_D$ the compression and decompression inference FLOPs per sample, respectively. We use $F$ to denote the inference FLOPs per sample for the baseline network. During each compression phase, the training FLOPs per sample can be estimated as $3\cdot F_C$. For decompression phases, we noticed that a small fraction of weights remain zero, and therefore $F_D < F$. When doing a backward pass we have additional $F_D$ from back-propagating the error, and $F$ extra FLOPs for the gradients with respect to all parameters. Therefore, we estimate the training FLOPs per sample during a decompression phase as $2\cdot F_D + F$. We measure the number of FLOPs on a random input sample, at the end of each training epoch and use this value to estimate the total training FLOPs for that particular epoch. To obtain the final number of FLOPs, we compute the inference FLOPs on a random input sample, estimate the backward FLOPs, compute the estimated training FLOPs over all training epochs as described above, and scale by the number of training samples.    

\subsection{Choice of Hyper-parameters}

\parhead{Length of compression/decompression phases.} AC/DC alternates between compression and decompression phases to co-train sparse and dense models. It is important to note, however, that the length of these phases, together with the warm-up and fine-tuning phases, could have a significant impact on the quality of the resulting models. Before settling on the sparsity pattern we used for all our ImageNet experiments (see Figure~\ref{fig:valacc-rn50-all} and Figure~\ref{fig:valacc-mobnet}), we experimented with different lengths for the sparse/dense phases, but found that ultimately the pattern used in the paper had the best trade-off between training FLOPs and validation accuracy. 

Notably, we experimented on ResNet50, 90\% sparsity, with increasing the training epochs to 130 and with different lengths for the compression/decompression phases. For example, we found that alternating between sparse/dense training every 10 epochs yielded slightly better results after 130 epochs: $75.34\%$ for the sparse model, $76.87\%$ for the fine-tuned dense model; however, this also had higher training FLOPs requirements ($0.7\times$ for the sparse model and $0.9\times$ including the fine-tuned dense). We additionally experimented with longer dense phases (10 epochs), compared to sparse phases (5 epochs); this also resulted in more accurate models: $75.45\%$ accuracy for the sparse model and $76.78\%$ -- for the fine-tuned dense model. However, the training FLOPs were substantially higher: $0.85\times$ for the sparse model and $1.15\times$ for the fine-tuned dense. 

Due to computational limitations, and to ensure a fair comparison with the dense baseline and other pruning methods, we decided on using a fixed number of 100 training epochs (the same used for the dense baseline). In this setup, we experimented mainly with the lengths for the compression/decompression phases used in Figure~\ref{fig:valacc-rn50-all} and Figure~\ref{fig:valacc-mobnet}, but noticed that having a longer final decompression phase had a positive impact on the fine-tuned dense model. For instance, when following a sparsity schedule as in Figure~\ref{fig:ac-dc}, the sparse model at 90\% sparsity had a very similar performance to the reported results (75.18\% accuracy, from one seed), while the fine-tuned dense model was significantly below the dense baseline (76.05\% validation accuracy). 
We believe having a short warm-up period and a longer fine-tuning phase are both beneficial for the sparse model; in our experiments, we only used warm-up phases of 10 epochs, but believe that shorter phases are worth exploring as well. Furthermore, the mask difference between consecutive compression phases is an important guide for choosing the sparsity schedule: as it was previously discussed, having a non-trivial difference between the masks typically results in better sparse models. Illustrations of the pruning masks during training on ImageNet are presented in Figure~\ref{fig:masks-rn50-all} and Figure~\ref{fig:masks-mobnet}.

When choosing the sparsity schedule for the language models experiments on Transformers-XL, we followed the same principles as for ImageNet. In fact, the sparsity schedule is very similar to the one used for ImageNet, scaled by the number of training epochs (~48 epochs or 100,000 steps for Transformers).

In the case of CIFAR100, we used for AC/DC the same number of 200 training epochs as for the dense baseline. We experimented with sparse/dense phases of lengths 10 or 20, and found that generally switching every 20 epochs between sparse and dense training yielded the best results. 

\parhead{Training Hyper-parameters for ImageNet.} We used the same hyper-parameters for all our ImageNet experiments, on both ResNet50 and MobileNetV1. Namely, we trained using SGD with momentum and batch size 256. We used a cosine learning rate scheduler, after an initial warm-up phase of 5 epochs, when the learning rate was linearly increased to 0.256. The momentum value was 0.875 and weight decay was 0.00003051757813. These hyper-parameters have the standard values used in the implementation of STR \cite{kusupati2020soft}. Furthermore, to improve efficiency, we train and evaluate the models using mixed precision (FP16). For models trained with mixed precision, the difference in accuracy between evaluating them with FP32 versus FP16 is negligible (<0.05\%). However, we noticed larger differences (around 0.2-0.3\%) in accuracy when training AC/DC with FP16 versus FP32. 

\parhead{Training Hyper-parameters for Transformer-XL.} For our Transformer-XL experiments, we integrated into our code-base the implementation provided by NVIDIA \footnote{\small{\url{https://github.com/NVIDIA/DeepLearningExamples/tree/master/PyTorch/LanguageModeling/Transformer-XL}}}, which also follows closely the original implementation in \cite{dai2019transformer}. We used the same hyper-parameters for training the large Transformer-XL model with 18 layers on WikiText-103, including the Lamb optimizer \cite{you2019large} with cosine learning rate scheduler.

\end{document}